\documentclass{article}

 \usepackage[nonatbib,preprint]{neurips_2021}
\usepackage{fullpage}
\usepackage{layout}
\usepackage{multirow}

\usepackage{amsfonts}
\usepackage{amsmath}
\usepackage{amsthm}
\usepackage{amssymb} 

\usepackage{nicefrac}
\usepackage{mathtools}

\usepackage{mdframed} 
\usepackage{thmtools}

\definecolor{shadecolor}{gray}{0.90}
\declaretheoremstyle[
headfont=\normalfont\bfseries,
notefont=\mdseries, notebraces={(}{)},
bodyfont=\normalfont,
postheadspace=0.5em,
spaceabove=1pt,
mdframed={
  skipabove=8pt,
  skipbelow=8pt,
  hidealllines=true,
  backgroundcolor={shadecolor},
  innerleftmargin=4pt,
  innerrightmargin=4pt}
]{shaded}

\declaretheorem[style=shaded,within=section]{definition}
\declaretheorem[style=shaded,sibling=definition]{theorem}

\declaretheorem[style=shaded,sibling=definition]{assumption}
\declaretheorem[style=shaded,sibling=definition]{corollary}

\declaretheorem[style=shaded,sibling=definition]{lemma}
\declaretheorem[style=shaded,sibling=definition]{remark}
\declaretheorem[style=shaded,sibling=definition]{example}

\usepackage{hyperref}       

\usepackage{algorithm}
\usepackage[noend]{algpseudocode}

 \usepackage{caption}
 \usepackage{subcaption}

 \usepackage{xcolor}
\usepackage{color}
\usepackage{graphicx}
\graphicspath{{./figures/}}  

\usepackage{hyperref}

\newcommand{\R}{\mathbb{R}} 

\newcommand{\cL}{{\cal L}}

\newcommand{\cO}{{\cal O}}

\newcommand{\cW}{{\cal W}}

\newcommand{\mD}{{\bf D}}

\newcommand{\mH}{{\bf H}}
\newcommand{\mI}{{\bf I}}

\newcommand{\mM}{{\bf M}}

\newcommand{\ones}{{\bf 1}}

\usepackage[colorinlistoftodos,bordercolor=orange,backgroundcolor=orange!20,linecolor=orange,textsize=scriptsize]{todonotes}

\newcommand{\eqdef}{\coloneqq} 

\newcommand{\dotprod}[1]{\left< #1\right>} 
\newcommand{\norm}[1]{ \left\| #1 \right\|}      

\DeclareMathOperator{\Range}{Range}     

\DeclareMathOperator{\argmin}{argmin}        

\newcommand{\diag}[1]{\mathbf{diag}\left( #1\right)}

\newcommand{\E}[1]{\mathbb{E}\left[#1\right] } 
\newcommand{\EE}[2]{\mathbb{E}_{#1}\left[#2\right] }

 \usepackage[firstinits=true,backend=bibtex,style=alphabetic,citestyle=alphabetic, url =false, arxiv =false, isbn = false, doi = false]{biblatex}
\usepackage{xspace}
\newcommand{\MOTAPS}{{\tt MOTAPS}\xspace}
\newcommand{\TAPS}{{\tt TAPS}\xspace}
\newcommand{\SP}{{\tt SP}\xspace}
\bibliography{../references}

\usepackage{fullpage}
\usepackage[utf8]{inputenc} 
\usepackage[T1]{fontenc}    
\usepackage{url}            
\usepackage{booktabs}       
\usepackage{amsfonts}       
\usepackage{nicefrac}       
\usepackage{microtype}      
\usepackage{xcolor}         

\newcommand{\disteq}{ }
\title{Stochastic Polyak Stepsize with a Moving Target}

\author{%
Robert M. Gower\thanks{Contact: gowerrobert@gmail.com} \\ LTCI, T\'el\'ecom Paris \\ Institut Polytechnique de Paris 
\And
Aaron Defazio\\ Facebook AI Research 
\And
  Michael Rabbat \\
  Facebook AI Research
}

\begin{document}

\maketitle

\begin{abstract}
We propose a new stochastic gradient method called \MOTAPS (Moving Targetted Polyak Stepsize) that uses recorded past loss values to compute adaptive stepsizes. \MOTAPS can be seen as a variant of the Stochastic Polyak~(\SP) which is also a method that also uses loss values to adjust the stepsize. The downside to the \SP method is that it only converges when the interpolation condition holds. \MOTAPS is an extension of \SP that does not rely on the interpolation condition.  The \MOTAPS method uses $n$ auxiliary variables, one for each data point, that track the loss value for each data point.    We provide a global convergence theory for \SP, an intermediary method \TAPS, and \MOTAPS  by showing that they all can be interpreted as a special variant of online SGD. We also perform several numerical experiments on convex learning problems, and deep learning models for image classification and language translation. In all of our tasks we show that  \MOTAPS is competitive with the relevant baseline method.
  
%
%
%
\end{abstract}



\section{Introduction}
\label{sec:intro}
Consider the problem \disteq
\begin{equation} \label{eq:main}
w^* \in \underset{{w \in \R^d}}{\argmin}  \; f(w) \eqdef \frac{1}{n}\sum_{i=1}^n f_i(w),
\end{equation}
where each $f_i(w)$ represents the loss of a model parametrized by $w\in \R^d$ over a given $i$th \emph{data point}.
We assume that there exists a  solution $w^* \in \R^d.$ Let $\cW^*$ denote the set of minimizers of~\eqref{eq:main}.

 An ideal method for solving~\eqref{eq:main} is one that exploits the sum of terms structure, has easy-to-tune hyper-parameters, and is guaranteed to converge. 
\emph{Stochastic gradient descent} (SGD) exploits this sum of terms structure by only using a single stochastic gradient (or a batch) $\nabla f_i(w)$  per iteration. 
Because of this, SGD is efficient when the \emph{number of data points} $n$ is large, and can even be applied when $n$ is infinite and~\eqref{eq:main} is an expectation over a continuous random variable.
   
  The issue with SGD is that it is difficult to use because it requires tuning a sequence of step sizes, otherwise known as a learning rate schedule.  Indeed,  SGD only converges when using a sequence of step sizes that converges to zero at just the right rate.  Here we develop methods with adaptive step sizes that use the loss values to set the stepsize. 
  
 We derive our new adaptive methods by first exploiting  the \emph{interpolation equations} given by
\begin{equation}\label{eq:interzero}
f_i(w) \; = \; 0, \quad \mbox{for }i=1,\ldots, n.
\end{equation}
We say that the \emph{interpolation assumption}  holds if there exists $w^* \in \cW^*$ that solves~\eqref{eq:interzero}. 
Two well-known settings where the interpolation assumption holds are 1) for binary classification with a linear model where the data can be separated by a hyperplane~\cite{Crammer06} or 2) when we know that each $f_i(w)$ is non-negative, and we have enough parameters in our model so there exists a solution that fits all data points. This second setting is often referred to as the overparametrized regime~\cite{vaswani2018fast}, and it is becoming a common occurrence in several sufficiently overparametrized deep neural networks~\cite{ZhangBHRV17,belkin2019reconciling}.

Our starting point is to observe that the Stochastic Polyak (\SP) method~\cite{ALI-G,SPS} directly exploits and solves the interpolation equations~\eqref{eq:interzero}.
 Indeed, the \SP method is a subsampled Newton Raphson method~\cite{SNR} as we show next.

The subsampled Newton Raphson method at each iteration samples a single index $i\in \{1,\ldots, n\}$ and focuses on solving the single equation $f_i(w) =0.$ This single equation can still be difficult to solve since $f_i(w)$ can be highly nonlinear. So instead, we linearize $f_i$ around a given  $w^t \in \R^d$, and set the linearization of $f_i(w)$ to zero, that is 
\[ f_i(w^t) + \dotprod{\nabla f_i(w^t), w-w^t} =0.\]
This is now a linear equation in $w\in\R^d$ that has $d$ unknowns and thus has infinite solutions. To pick just one solution we use a projection step as follows
\begin{eqnarray}
w^{t+1} = \argmin_{w\in\R^d} \norm{w-w^t}^2
\quad \mbox{subject to }f_i(w^t) + \dotprod{\nabla f_i(w^t), w-w^t} =0. \label{eq:newtraphintro}
\end{eqnarray}
 The solution  to this projection step (See Lemma~\ref{lem:proj1eq} for details) is given by
\begin{equation}\label{eq:SPSintro}
w^{t+1} = w^t - \frac{f_i(w^t)}{\norm{\nabla f_i(w^t)}^2} \nabla f_i(w^t).
\end{equation}

This method~\eqref{eq:SPSintro} is known as the  Stochastic Polyak method~\cite{SPS}\footnote{In~\cite{ALI-G} the authors also observed that the full batch Polyak stepsize in 1D is a Newton Raphson method.}. The \SP has many desirable properties: It is incremental, it adapts it step size according to the current loss function, and it enjoy several invariance properties as we point out later in Remark~\ref{rem:invariant}. Thus in many senses the \SP is an ideal stochastic  method. The downside to \SP is that to arrive at~\eqref{eq:SPSintro} we have to assume that the interpolation assumption holds. 
The main objective of this paper is design methods akin to the \SP method that do not rely on the interpolation assumption.

Next we highlight some of our contributions.
%

%
%

\subsection{Contributions}

\paragraph{New perspectives and analysis of Stochastic Polyak.}
We provide three viewpoints of the \SP method: 1) as a subsampled Newton method in Section~\ref{sec:SPSNR},  2) as a type of online SGD method in Section~\ref{sec:SPSonlineSGD} and finally 3) motivated through star-convexity in Section~\ref{sec:SPSalternativegamma1theory} which is closely related to Polyak's original motivation~\cite{Polyak87}.


\paragraph{Moving Targeted Stochastic Polyak.}

 By leveraging the subsampled Newton viewpoint, we develop a new variant of the \SP method that does not rely on the interpolation assumption. 
 As an initial step in this direction, we first assume that $f(w^*)$ is known, and we introduce the \emph{TArgeted stochastic Polyak Stepsize} (\TAPS) method.
\TAPS uses $n$ auxiliary scalar variables that track the evolution of the individual function values $f_i(w)$. Of course, $f(w^*)$ is not known in general. Using the SGD viewpoint of \SP, we propose the \emph{Moving} Targeted Stochastic Polyak (\MOTAPS), that does not require knowledge of $f(w^*)$. Rather, \MOTAPS has the same $n$ auxiliary scalars as \TAPS plus one additional variable that tracks the global loss $f(w)$.

\paragraph{Unifying Convergence Theory.}
We prove that all three of the methods \SP, \TAPS and \MOTAPS can be interpreted as variants of online SGD, and we use this to establish a unifying convergence theorem for all three of these methods.
Furthermore, we show how theses variants of online SGD enjoy a remarkable growth  property that greatly facilitates a proof of convergence. Indeed, we present Theorems~\ref{theo:onlinesgdstar} and~\ref{theo:onlinesgdstrongstar} that hold for these three methods by using this online SGD viewpoint. Theorem~\ref{theo:onlinesgdstar} uses a star-convexity assumption~\cite{LeeV16,hinder2019near}, which are a class of non-convex functions that includes convex functions,  loss functions of some neural networks along the path of SGD~\cite{zhou2018sgd,kleinberg2018alternative}, several non-convex  generalized linear models~\cite{LeeV16}, and learning linear dynamical systems~\cite{hardt2018dynam}.\footnote{To be precise, the proof in~\cite{hardt2018dynam} relies on a quasi-convex assumption that is $$h_t(z^*) \; \geq \; h_t(z^t) +\gamma \dotprod{\nabla h_t(z^t), z^* -z^t}, $$ 
where $\gamma >0$ is relaxation parameter. For $\gamma =1$  quasi-convex are star-convex functions. 
}
Using the SGD viewpoint of \SP, we derive an explicit convergence theory of \SP for smooth and star convex loss functions in Sections~\ref{sec:SPStheorystar} and~\ref{theo:onlinesgdsmoothSPS}.

\subsection{Related Work}


Developing methods that adapt the stepsize using  information collected during the iterative process is now a very active area of research. 
%
Adaptive methods such as AdaGrad~\cite{ADAGRAD} and Adam~\cite{ADAM} have a step size that adapts to the scaling of the gradient, and thus are generally easier to tune than SGD, and have now become staples in training deep neural networks (DNNs). 
While the practical success of Adam is undeniable, a fundamental understanding of why these methods work so well remains elusive, particularly on models that interpolate data such as DNNs.

 Recently a new family of adaptive methods based on the Polyak step size~\cite{Polyak87} has emerged, including the \emph{stochastic Polyak step size} (\SP) method~\cite{oberman2019stochastic,SPS} and ALI-G~\cite{ALI-G}.
 \SP is also an adaptive method, since it adjusts its step size depending on both the current loss value and magnitude of the stochastic gradient.
Under the interpolation condition, the \SP method converges sublinearly under convexity~\cite{SPS} and star-convexity~\cite{SGDstruct}, and linearly under strong convexity and the PL condition~\cite{SPS,SGDstruct}.  Recently in~\cite{SPS} the authors proposed the SPS$_{\max}$ method, which is a variant of \SP that caps large stepsizes which greatly helps to stablize the numeric convergence of \SP. 
 Prior to this, the ALI-G method~\cite{ALI-G} can be interpreted as dampened version of  SPS$_{\max}$ method  with follow-up work highlighting the importance of momentum in accelerating these methods in practice~\cite{ALIG2}.

 Our derivation of \SP as a projection in~\eqref{eq:newtraphintro} shows that \SP can be interpreted as a extension of the passive-aggressive methods to nonlinear models~\cite{Crammer06}. Indeed, the passive aggressive methods apply the same projection in~\eqref{eq:newtraphintro} but with the constraint  $f_i(w) =0$. This projection has a closed form solution when $f_i$ is a hinge loss over a linear model, which was the setting where passive-aggressive models were first developed and most applied.

 Another related set of methods are the model based methods in~\cite{AsiD19}, where each new iteration is the result of minimizing the sum of a model of $f_i(w)$ and the norm squared distance to a prior point, that is
\begin{equation}\label{eq:model}
w^{t+1} \; = \argmin_{w\in \R^d}\left\{ f_{t,i}(w) + \frac{1}{2} \norm{w-w^t}^2\right\},
\end{equation} 
 where $f_{t,i}(w)$ is some \emph{local model} of $f_i(w)$ such that $f_{t,i}(w^t) = f_i(w^t).$ 
  This model includes linearizations of $f_i(w)$ as a special case. The SPS$_{\max}$ method~\cite{SPS} is in fact a 
special case of the model based methods~\eqref{eq:model}, where-in  the model is given by the positive part of a local linearization, that is

  \begin{equation}
  f_{t,i}(w) \; = \; \max\{f_{i}(w^t) + \dotprod{\nabla f_i(w^t), w-w^t}, \; 0\},
  \end{equation}
  as observed in ~\cite{ALI-G}.
  Using the positive part is justified for non-negative loss functions.

 
Line search methods that work with stochastic gradients~\cite{vaswani2019painless,zhang2020statistical} are another promising and related direction for automatically adapting step-sizes.

\section{The Stochastic Polyak Method}

We start by
presenting the \SP (Stochastic Polyak method) through two  different viewpoints. First, we show that  \SP is  a special case of the subsampled Newton-Raphson method~\cite{SNR}. Using this first viewpoint, and leveraging results from~\cite{SNR}, we then go on to show that \SP can also be viewed as a type of \emph{online SGD method}, which greatly facilitates the analysis of \SP.

\subsection{The Newton-Raphson viewpoint}
\label{sec:SPSNR}

As observed in the introduction in Section~\ref{sec:intro}, the \SP method is designed for solving interpolation equations. Here we formalize and extend this observation before moving on to our new methods. 

We can derive an extended form of the \SP method that does not rely on the interpolation assumption. Instead of the interpolation assumptions, let us assume for now that we have access to the loss values $f_i(w^*)$ for each $i=1,\ldots, n$, where $w^* \in \cW$. If we knew the $f_i(w^*)$'s then we can solve the optimization problem~\eqref{eq:main} by solving instead the nonlinear equations\vspace{-0.1cm}
\begin{equation}\label{eq:fizero}
f_i(w) \;= \;f_i(w^*), \quad \mbox{for }i=1,\ldots, n.
\end{equation}

Using the same reasoning in Section~\ref{sec:intro}, we can design an adaptive method for solving~\eqref{eq:fizero} by sampling the $i$th loss, linearizing and projecting as follows
\begin{eqnarray}
w^{t+1} = \argmin_{w\in\R^d} \norm{w-w^t}^2
\quad \mbox{subject to }f_i(w^t) + \dotprod{\nabla f_i(w^t), w-w^t} =f_i(w^*). \label{eq:newtraph}
\end{eqnarray}
The solution  to this projection step (see Lemma~\ref{lem:proj1eq} for details) is given by
\begin{equation}\label{eq:SPS}
w^{t+1} = w^t - \frac{f_i(w^t)-f_i(w^*)}{\norm{\nabla f_i(w^t)}^2} \nabla f_i(w^t).
\end{equation}
This method is a minor extension of~\eqref{eq:SPSintro} wherein we now allow $f_i(w^*) \neq 0.$ Despite this minor change, we will also refer to~\eqref{eq:SPS} as the Stochastic Polyak method.\footnote{Using $f_i(w^*)$ in the numerator is apparently new, and what we call the Stochastic Polyak method here is not the same as the Stochastic Polyak method proposed in ~\cite{SPS}. 
 In Section~\ref{sec:SPtoSPS} we detail these differences. In the more common case where $f_i(w^*) =0$, there is consensus that~\eqref{eq:SPS} is called the Stochastic Polyak method .}

The issue with the \SP method is that we often do not know $f_i(w^*)$, except in the case of interpolation where $f_i(w^*) =0.$  Outside of this setting, it is unlikely that we would have access to each $f_i(w^*).$ Thus, we relax this requirement in Sections~\ref{sec:TAPS} and~\ref{sec:MOTAPS}. But first, we present yet another viewpoint of \SP as a type of online SGD method.

\subsection{The SGD viewpoint}
\label{sec:SPSonlineSGD}

 Fix a given $w^t\in \R^d$ and consider the following \emph{auxiliary} objective function 
\begin{equation} \label{eq:proxyobj}
 \underset{{w \in \R^d}}{\min} h_t(w)   \;  \eqdef  \; \frac{1}{n}\sum_{i=1}^n \frac{1}{2}\frac{(f_i(w) -f_i(w^*))^2}{\norm{\nabla f_i(w^t)}^2}.
\end{equation}
Here we use the pseudoinverse convention
that if $\norm{\nabla f_i(w^t)}=0\in \R$ then $\norm{\nabla f_i(w^t)}^{-1}=0$.
Clearly $w =w^*$ is a minimizer of~\eqref{eq:proxyobj}. 
This suggests that we could try to minimize~\eqref{eq:proxyobj} as a proxy for solving the equations~\eqref{eq:main}. 
Since~\eqref{eq:proxyobj} is a sum of terms that depends on $t$, we can use online SGD to minimize~\eqref{eq:proxyobj}. To describe this online SGD method let

\begin{equation}\label{eq:proxyloss}
h_{i,t}(w) \; \eqdef \; \frac{1}{2} \frac{(f_i(w)-f_i(w^*))^2}{\norm{\nabla f_i(w^t)}^2} \quad \mbox{and thus} \quad \nabla h_{i,t}(w) \; = \; \frac{f_i(w)-f_i(w^*)}{\norm{\nabla f_i(w^t)}^2} \nabla f_i(w).
\end{equation}
The online SGD method is given by sampling $i_t \in \{1,\ldots, n\}$ and then iterating
\begin{align}
w^{t+1} & =\; w^t -\gamma \nabla h_{i_t,t}(w^t)\; \overset{\eqref{eq:proxyloss}}{=}\; w^t -\gamma\frac{f_{i_t}(w^t)-f_{i_t}(w^*)}{\norm{\nabla f_{i_t}(w^t)}^2} \nabla f_{i_t}(w^t), \label{eq:SPS-SGDview}
\end{align}
which is equivalent to the \SP method~\eqref{eq:SPS} but with the addition of a stepsize $\gamma>0.$
This online SGD viewpoint of \SP is very useful for proving convergence of \SP. Indeed, there exist many convergence results in the literature on online SGD for convex, non-convex, smooth and non-smooth functions that we can now import  to analyzing \SP. Furthermore, it turns out that~\eqref{eq:SPS-SGDview} enjoys a remarkable growth property that facilitates many SGD proof techniques, as we show in the next lemma.

\begin{lemma}[Growth] \label{lem:weakgrowthsps}
The functions $h_{i,t}(w)$ defined in \eqref{eq:proxyloss} satisfy \disteq
\begin{equation} \label{eq:subsmooth1}
\norm{\nabla h_{i,t}(w^t) }^2 \; = \; 2 h_{i,t}(w^t).
\end{equation}
Consequently due to~\eqref{eq:proxyobj} we have that
\begin{equation} \label{eq:subsumsmooth1}
\frac{1}{n}\sum_{i=1}^n\norm{\nabla h_{i,t}(w^t) }^2 \; = \; 2 h_{t}(w^t).
\end{equation}
\end{lemma}
 \noindent \emph{Proof.} The proof follows immediately from~\eqref{eq:proxyloss} and~\eqref{eq:proxyobj} since \disteq
\begin{eqnarray*}
\norm{\nabla h_{i,t}(w^t) }^2  & \overset{ \eqref{eq:proxyloss}}{ =}&  \frac{(f_i(w)-f_i(w^t))^2}{\norm{\nabla f_i(w^t)}^4}  \norm{\nabla f_i(w^t)}^2 
 =  \frac{(f_i(w^t)-f_i(w^*))^2}{\norm{\nabla f_i(w^t)}^2}  \; \overset{ \eqref{eq:proxyloss}}{ =}\; 2 h_{i,t}(w^t). \qed
\end{eqnarray*} 

In Section~\ref{sec:theory} we will exploit this SGD viewpoint and the growth property in Lemma~\ref{lem:weakgrowthsps} to prove the convergence of \SP. But first we develop new variants of \SP that do not require knowing the $f_i(w^*)$'s.

\begin{remark}[Invariances] \label{rem:invariant}
This SGD viewpoint also hints as to why \SP is \emph{invariant} to several transformations of the problem~\eqref{eq:main}.
Indeed, the reformulation given in~\eqref{eq:proxyobj} is itself invariant to any re-scaling or translations of the loss functions. That is,  replacing each $f_i$ by $c_i f_i$ or $f_i +c_i$ where $c_i \neq 0$  has no effect on~\eqref{eq:proxyobj}. Furthermore, the \SP method is invariant to taking powers of the loss functions.\footnote{Observation thanks to Konstantin Mischenko.} We can again see this through the reformulation~\eqref{eq:proxyobj}.
 Indeed, if we assume interpolation with  $f_i(w^*) =0$ and replace each $f_i(w)$ by $f_i(w)^{p}$ where $p>0$, then 
$$\norm{\nabla (f_i(w)^{p})}^2 =  \norm{p f_i(w)^{p-1} \nabla f_i(w)}^2 = p^2 f_i(w)^{2p-2} \norm{ \nabla f_i(w)}^2 .$$
Thus $f_i(w)^{2p}/\norm{\nabla (f_i(w)^{p})}^2  = \tfrac{1}{p^2} f_i(w)^2/ \norm{ \nabla (f_i(w)}^2$ and raising each $f_i(w)$ to the power $p$ results in the same optimization problem, albeit scaled by $p^2$.
\end{remark}

\section{Targeted Stochastic Polyak Steps}
\label{sec:TAPS}

Now suppose that we do not know $f_i(w^*)$ for $i=1,\ldots, n$. Instead, we only have a target value for which we would like the \emph{total loss} to reach.

\begin{assumption}[Target] \label{ass:target}
There exists a \emph{target value} $\tau \ge 0$ such that every $w^* \in \cW^*$ 
 is a solution to the nonlinear equation
\begin{equation} \label{eq:tau}
f(w) \;= \;  \frac{1}{n}\sum_{i=1}^n f_i(w) \; = \; \tau.
\end{equation}
\end{assumption}
Using Assumption~\ref{ass:target} we develop new variants of the \SP method as follows.
First we re-write~\eqref{eq:tau} by  introducing auxiliary variables $\alpha_i \in \R$ for $i=1,\ldots, n$ such that \disteq
\begin{align}
 \frac{1}{n}\sum_{i=1}^n \alpha_i & \; = \; \tau, \label{eq:averagealphastau}\\
 f_i(w) & \;= \; \alpha_i, \quad \mbox{for }i=1,\ldots, n.\label{eq:fizeroalphas}
\end{align}
This reformulation \emph{exposes} the $i$th loss function (and thus the $i$th data point) as a separate equation. 
Because each loss (and associated data point) is on a separate row, applying a subsampled Newton-Raphson method results in an incremental method, as we show next.
   
Let $w^t \in \R^d$ and $\alpha^t = (\alpha^t_1, \ldots, \alpha^t_n) \in \R^n$ be the current iterates.
At each iteration we can either sample~\eqref{eq:averagealphastau} or one of the equations~\eqref{eq:fizeroalphas}. We then apply a Newton-Raphson step using just this sampled equation. For instance, if we sample one of the equations in~\eqref{eq:fizeroalphas},  we first linearize in $w$ and $\alpha_i$ around the current iterate and set this linearization to zero, which gives \vspace{-0.1cm}
\begin{equation}
f_i(w^t) + \dotprod{\nabla f_i(w^t), w-w^t} = \alpha_i. \end{equation}
Projecting the previous iterates onto this linear equation gives
\begin{align}
w^{t+1}, \alpha_i^{t+1} & =  \underset{w\in\R^d, \alpha_i \in \R}{\argmin} \norm{w-w^t}^2+\norm{\alpha_i-\alpha_i^t}^2 
\; \mbox{subject to }f_i(w^t) + \dotprod{\nabla f_i(w^t), w-w^t} =\alpha_i. \nonumber 
\end{align}
The solution\footnote{Proven in Lemma~\ref{lem:updateTasps}.} to the above
is given by the updates in lines~\ref{ln:alphaistep} and~\ref{ln:wstep} in Algorithm~\ref{alg:TAPS} when $\gamma=1$.

Alternatively, if we sample~\eqref{eq:averagealphastau}, projecting the current iterates onto this constraint gives \vspace{-0.1cm}
\begin{eqnarray}
\alpha^{t+1} & = & \argmin_{\alpha\in\R^n} \norm{\alpha-\alpha^t}^2 \quad \mbox{subject to } \quad \frac{1}{n}\sum_{i=1}^n \alpha_i  \; = \; \tau. \label{eq:newtraphalphaonly}
\end{eqnarray}
The closed form soution to~\eqref{eq:newtraphalphaonly} is given in line~\ref{ln:alphaggregate} in Algorithm~\ref{alg:TAPS} when $\gamma=1$. In Algorithm~\ref{alg:TAPS} we give the complete pseudocode of the subsampled Newton-Raphson method applied to~\eqref{eq:fizeroalphas}. We refer to this algorithm as the 
the \emph{Target Stochastic Polyak method}, or \TAPS for short.

\begin{algorithm}[H]
\begin{algorithmic}[1]
\State {\bf Inputs:}  target $\tau \geq 0$ and  stepsize $\gamma >0$
\State {\bf Initialize:} $w^0=0, \alpha_i^0 = \overline{\alpha}^0=0$ for $i=1,\ldots, n$ 
\For{$t =0,\ldots, T-1$}
\State  Sample $i \in \{1,\ldots, n+1\}$ according to some law
\If{$i = n+1$}
\State $\alpha^{t+1}_j  =  \alpha^t_j +\gamma (\tau -\overline{\alpha}^t)$, for $j =1,\ldots, n$ where $\overline{\alpha}^t = \frac{1}{n}\sum_{i=1}^n \alpha_i^t$.
\label{ln:alphaggregate}
\Else
\State $\displaystyle \alpha^{t+1}_i  = \alpha^t_i +\gamma \frac{f_i(w^t) -\alpha^t_i}{\norm{\nabla f_i(w^t)}^2+1}  $ \label{ln:alphaistep}
\State $ \displaystyle
w^{t+1}  = w^t -\gamma \frac{f_i(w^t) -\alpha^t_i}{\norm{\nabla f_i(w^t)}^2+1} \nabla f_i(w^t).$\label{ln:wstep}
\EndIf
\EndFor  
\State {\bf Output:} $w^{T}$
\end{algorithmic}
\caption{\TAPS: {\tt TA}rgetted {\tt S}tochastic {\tt P}olyak {\tt S}tep}
\label{alg:TAPS}
\end{algorithm}

%
%

\begin{remark}[\TAPS stops at the solution] Algorithm~\ref{alg:MOTAPS} stops when it reaches the solution. That is, if $w^t = w^*$ and $\alpha_i^t = f_i(w^*)$ for all $i$, then both lines~\ref{ln:alphaistep} and~\ref{ln:wstep} have no affect on $w$ or the $\alpha_i$'s.
 Furthermore $\tau = \overline{\alpha}^t \eqdef \frac{1}{n} \sum_{i=1}^n \alpha_i^t$ and consequently the $\alpha_i$'s are no longer updated in line~\ref{ln:alphaggregate}.   This natural stopping is a sanity check that  SGD does not satisfy.
\end{remark}

\subsection{The SGD viewpoint}
\label{sec:SGDviewpoint}

The \texttt{TAPS} method in Algorithm~\ref{alg:TAPS} can also be cast as an online SGD method.
To see this, first we re-write~\eqref{eq:proxyobj} as  the minimization of an auxiliary function
\begin{equation} \label{eq:proxyobjalpha}
 \underset{{w \in \R^d, \alpha \in \R^n}}{\min}  h_t(w,\alpha) \; \eqdef \; \frac{1}{n+1} \left(\sum_{i=1}^n \frac{1}{2}\frac{(f_i(w) -\alpha_i)^2}{\norm{\nabla f_i(w^t)}^2+1}  + \frac{n}{2}(\overline{\alpha} - \tau)^2 \right).
\end{equation}
In the following lemma we show that minimizing~\eqref{eq:proxyobjalpha} is equivalent to minimizing~\eqref{eq:main}.
\begin{lemma}\label{lem:reformsgdtasps}
Let Assumption~\ref{ass:target} hold.
Every stationary point of~\eqref{eq:proxyobjalpha} is a stationary point of~\eqref{eq:main}. Furthermore, every minimizer $(w^*,\alpha^*) \in \R^{d+n}$ of~\eqref{eq:proxyobjalpha} is such that $w^*$ is a minima of~\eqref{eq:main} and
\begin{equation}
\alpha^*_i \; = \; f_i(w^*).
\end{equation}
Consequently we have that $h_t(w^*,\alpha^*) = 0.$
\end{lemma}
The proof of this lemma, and all subsequent lemmas are in the appendix in Section~\ref{sec:missingproof}.
Due to Lemma~\ref{lem:reformsgdtasps} we can focus on minimizing~\eqref{eq:proxyobjalpha}.
Furthermore, note from Lemma~\ref{lem:reformsgdtasps} we have that the minimizer of~\eqref{eq:proxyobjalpha} does not depend on $t$, despite the dependence of the objective $h_t(w,\alpha) $ on $t.$

Since~\eqref{eq:proxyobjalpha} is an average of $(n+1)$ terms we can apply an online SGD method.
To simplify notation, for $i=1,\dots,n$ let 
\begin{eqnarray}
h_{i,t}(w,\alpha) \; \eqdef \; \frac{1}{2}\frac{(f_i(w) -\alpha_i)^2}{\norm{\nabla f_i(w^t)}^2+1} \quad \mbox{and} \quad
h_{n+1,t}(w,\alpha) \; \eqdef \; \frac{n}{2}(\overline{\alpha} - \tau)^2 . \label{eq:fi}
\end{eqnarray}
Note that despite our notation $h_{n+1,t}(w,\alpha)$ does not in fact depend on $t$ or $w$. We do this for notational consistency.

Let $\gamma >0$ be the learning rate. Starting from any $w^0$ and with $\alpha_i^0 = 0$ for all $i$, at each iteration $t$ we sample an index $i_t \in \{1,\ldots, n+1\}$. If $i_t \neq n+1$ then we sample $h_{i_t,t}$ and update
\begin{eqnarray}
w^{t+1} & =& w^t -\gamma \nabla_w h_{i_t,t}(w^t,\alpha^t) 
\;\overset{\eqref{eq:fi}}{=}\; w^t -\gamma \frac{f_{i_t}(w^t) -\alpha_i^t}{\norm{\nabla f_i(w^t)}^2+1} \nabla f_{i_t}(w^t) \label{eq:updatesgd1}\\
\alpha^{t+1}_i & =& \alpha^{t}_i - \gamma\nabla_{\alpha_i} h_{i_t,t}(w^t,\alpha^t)
\; \overset{\eqref{eq:fi}}{=} \; \alpha^{t}_i + \gamma \frac{f_{i_t}(w^t) -\alpha_i^t}{\norm{\nabla f_{i_t}(w^t)}^2+1}. \label{eq:updatesgd}
\end{eqnarray}
Thus  
we have that~\eqref{eq:updatesgd1} and ~\eqref{eq:updatesgd} are equal to lines~\ref{ln:wstep} and~\ref{ln:alphaistep} in Algorithm~\ref{alg:TAPS}, respectively.
Alternatively if $i_t = n+1$ then we sample $h_{n+1,t}$ and our SGD step is given by
\begin{equation}
\alpha^{t+1} \; = \; \alpha^t - \gamma \nabla_{\alpha}h_{n+1,t} (w^t, \alpha^t) \; \overset{\eqref{eq:fi}}{=}  \; \alpha^t - \gamma(\overline{\alpha}^t - \tau)\label{eq:updatessgdavalpha}
\end{equation}
which is equal to line \ref{ln:alphaggregate} in Algorithm~\ref{alg:TAPS}.

We rely on this SGD interpretation of the \TAPS~method to provide a convergence analysis in Section~\ref{sec:theory} (specialized to \TAPS in Section~\ref{sec:convTAPS}). Key to this forthcoming analysis, is the following property. 
%
\begin{lemma}[Growth]\label{lem:TAPSweakgrowth}
The functions $h_{i,t}(w)$ defined in \eqref{eq:fi} satisfy
\begin{eqnarray} 
\norm{\nabla h_{i,t}(w^t,\alpha) }^2 & = & 2 h_{i,t}(w^t,\alpha), \quad \mbox{for }i=1,\ldots, n+1.  \label{eq:smooth1}
\end{eqnarray}
Consequently the function $h_t(w,\alpha)$ in~\eqref{eq:proxyobjalpha} satisfies
\begin{equation} \label{eq:hzsmooth}
\frac{1}{n+1}\sum_{i=1}^n\norm{\nabla h_{i,t}(w^t,\alpha) }^2 \; = \; 2h_t(w^t,\alpha) .
\end{equation}
\end{lemma}

%


In the next section we completely remove Assumption~\ref{ass:target} to develop a  stochastic method that records only function values and needs no prior information on $f_i(w^*)$ or $f(w^*).$

\section{Moving Targeted Stochastic Polyak Steps}
\label{sec:MOTAPS}

Here we dispense of Assumption~\ref{ass:target} and instead introduce $\tau$ as a variable. Our objective is to design a \emph{moving target} variant of the \TAPS ~method that updates the target $\tau$ in a such a way that guarantees convergence. To design this moving target variant, we rely on the SGD online viewpoint.
Consider the auxiliary objective function \vspace{-0.2cm}
\begin{equation} \label{eq:proxyobjalphatau}
 \underset{{w \in \R^d, \alpha \in \R^n, \tau \in \R}}{\min}  h_t(w,\alpha,\tau) \; \eqdef\;  \; \frac{1}{n+1} \left(\sum_{i=1}^n \frac{1-\lambda}{2}\frac{(f_i(w) -\alpha_i)^2}{\norm{\nabla f_i(w^t)}^2+1}  + \frac{1-\lambda}{2}n(\overline{\alpha} - \tau)^2+\frac{\lambda}{2} \tau^2 \right),
\end{equation}
where $\lambda >0$ is a dampening parameter. Note that for $\lambda=0$ we recover the same auxiliary function of the \TAPS method in~\eqref{eq:proxyobjalpha}. 
\begin{lemma} \label{lem:movtargequiv}
Let
\begin{equation} \label{eq:alphastarandtaustar}
\alpha_i^* \eqdef f_i(w^*) \quad \mbox{and} \quad \tau^* = f(w^*), \quad \mbox{for }i=1,\ldots, n.
\end{equation}
It follows that
\begin{align}
h_t(w^*,\alpha^*,\tau^*) 
& =\; \frac{\lambda f(w^*)^2}{2(n+1)}. \label{eq:htstarmain}
\end{align}
Furthermore, 
every stationary point of~\eqref{eq:proxyobjalphatau} is a stationary point of~\eqref{eq:main}. Finally if $f(w) \geq 0$ and $(w^*,\hat{\alpha},\hat{\tau})$ is a minima  of~\eqref{eq:proxyobjalphatau} then $w^*$ is a minima of ~\eqref{eq:main}.
\end{lemma}

Since minimizing~\eqref{eq:proxyobjalphatau} is equivalent to minimizing~\eqref{eq:main}, we can focus on solving~\eqref{eq:proxyobjalphatau}. Following the same pattern from the previous sections, 
we will minimize the sum of $(n+1)$ terms in~\eqref{eq:proxyobjalphatau} using SGD. 
In applying SGD, we partition the function~\eqref{eq:proxyobjalphatau} into $n+1$ terms, where the first $n$ terms are given by 
\[h_{i,t}(w,\alpha,\tau) = \frac{1-\lambda}{2}\frac{(f_i(w) -\alpha_i)^2}{\norm{\nabla f_i(w^t)}^2+1}\quad \mbox{for }i=1,\ldots, n.\]
The $(n+1)$th term is given by \vspace{-0.1cm}
\begin{align}
h_{n+1,t}(w,\alpha,\tau) & \eqdef  \frac{(1-\lambda)n}{2}(\overline{\alpha} - \tau)^2 + \frac{\lambda}{2} \tau^2 \quad \mbox{thus} \nonumber \\
  \nabla _{\tau}h_{n+1,t}(w,\alpha,\tau) & = \;(1-\lambda)n( \tau-\overline{\alpha} ) + \lambda \tau.   \label{eq:fn1tau} 
\end{align} 
Sampling the  $(n+1)$th term and taking a gradient step to update $\tau^t\in\R$ gives the following update
\begin{align}
 \tau^{t+1} &= \tau^{t} - \gamma\nabla_\tau h_{n+1,t}(w,\alpha,\tau)|_{(w,\alpha,\tau) =(w^t,\alpha^t,\tau^t)  } \nonumber\\
 & =\big(1-\underbrace{ \gamma \left( \lambda+ (1-\lambda)n\right)}_{\eqdef \gamma_\tau}\big)\tau^{t}+\gamma (1-\lambda)n \overline{\alpha}^{t} \label{eq:tausgdupdate1}\\
 & = (1-\gamma_{\tau}) \tau^{t}+\gamma_{\tau} \frac{(1-\lambda)n}{\lambda+(1-\lambda)n} \overline{\alpha}^{t} \label{eq:tausgdupdate2}
\end{align}
where we have introduced a separate learning rate $\gamma_{\tau}$ for updating $\tau$. We find that a separate learning rate $\gamma_{\tau}$ is needed for updating $\tau$, otherwise to keep $\tau$ from being negative in~\eqref{eq:tausgdupdate1} we would need to restrict $\gamma$ to be less than $\frac{1}{\lambda+ (1-\lambda)n}$ which can be small when $\lambda$ is close to zero.
See Algorithm~\ref{alg:MOTAPS} for the resulting method.  We refer to this method as the \emph{Moving Target Stochastic Polyak Stepsize} or \MOTAPS for short. 


\begin{algorithm}
\begin{algorithmic}[1]
\State {\bf Inputs:} Dampening  $\lambda  \in [0, \;1]$ and  learning rates $\gamma, \gamma_{\tau}\in [0,\;1 ]$ 
\State {\bf Default:}  $\gamma =  0.9$, $\gamma_{\tau} = \lambda = 0.1$, $w^0=0, \alpha_i^0 = \overline{\alpha}^0 = 0 = \tau$ for $i=1,\ldots, n$ 
\For{$t =0,\ldots, T-1$}
\State  Sample $i \in \{1,\ldots, n+1\}$ according to some law
\If{$i = n+1$}
\State $\alpha^{t+1}_j  =  \alpha^t_j +\gamma (\tau -\overline{\alpha}^t), $ for $j =1,\ldots, n.$ \Comment \texttt{Updating all $\alpha$'s}
\State 
$\tau = (1-\gamma_{\tau}) \tau+\gamma_{\tau} \frac{(1-\lambda)n}{\lambda+(1-\lambda)n} \overline{\alpha}$
 \Comment \texttt{Updating target $\tau$}
\label{ln:alphaggregatet}
\State   $\overline{\alpha}^{t+1}  =  \overline{\alpha}^t +\gamma (\tau -\overline{\alpha}^t) $
\Else
\State $\displaystyle \alpha^{t+1}_i  = \alpha^t_i +\gamma \frac{f_i(w^t) -\alpha^t_i}{\norm{\nabla f_i(w^t)}^2+1}  $ \label{ln:alphaistept}  \Comment \texttt{Updating $\alpha_i$}
\State $ \displaystyle
w^{t+1}  = w^t -\gamma \frac{f_i(w^t) -\alpha^t_i}{\norm{\nabla f_i(w^t)}^2+1} \nabla f_i(w^t).$\label{ln:wstept}  \Comment \texttt{Updating $w$}
\State $ \displaystyle  \overline{\alpha}^{t+1} = \overline{\alpha}^t +\frac{1}{n}(\alpha^{t+1}_i  - \alpha^{t}_i ) $
\EndIf
\EndFor  
\State {\bf Output:} $w^{T}$
\end{algorithmic}
\caption{MO\TAPS: Moving {\tt TA}rgetted {\tt S}tochastic {\tt P}olyak {\tt S}tep}
\label{alg:MOTAPS}
\end{algorithm}

The dampening parameter $\lambda$ controls how fast the stochastic gradients of $h_t(w,\alpha,\tau)$ can grow, as we show next. As a consequence, later on we will see that the $\lambda$ will later control the rate of convergence of \MOTAPS. 
%

\begin{lemma}\label{lem:MOSTAPSweakgrowthn1}
Consider $h_{t}(w,\alpha,\tau) $ given in~\eqref{eq:proxyobjalphatau}.
If
\begin{equation}
\lambda \leq \frac{2n+1}{2n+3} < 1
\end{equation}
 then
\begin{equation} \label{eq:MOSTAPSweakgrowthn1}
\frac{1}{n+1}\sum_{i=1}^{n+1}\norm{\nabla h_{i,t}(w^t,\alpha,\tau)}^2 \; \leq\;  2(1-\lambda)(2n+1) h_{t}(w^t,\alpha,\tau).
\end{equation}
\end{lemma}

Next we establish a general convergence theory through which we will analyse \SP, \TAPS and \MOTAPS.

\section{Convergence Theory}
\label{sec:theory}
All of our methods presented thus far can be cast as a particular variant of online SGD.
Indeed,  \SP,  \texttt{TAPS} and \texttt{MOTAPS}  given in~\eqref{eq:SPS}, Algorithms~\ref{alg:TAPS} and~\ref{alg:MOTAPS}, respectively,   are equivalent to applying SGD to~\eqref{eq:proxyobj},~\eqref{eq:proxyobjalpha} and~\eqref{eq:proxyobjalphatau}, respectively.
  We will leverage this connection to provide a  convergence theorem for these three methods. Throughout our proofs we use \vspace{-0.2cm}
\begin{equation}\label{eq:proxyobjgen}
 \min_z h_t(z) := \frac{1}{n} \sum_{i=1}^n h_{i,t}(z) 
\end{equation}
as the auxiliary function in consideration. Here $z$ represents the variables of the problem. For instance, for the \SP method~\eqref{eq:SPS} we have that $z = w\in \R^d$, for \TAPS in Algorithm~\ref{alg:TAPS} we have that $z = (w,\alpha) \in \R^{n+d}$ and finally for \texttt{MOTAPS} in Algorithm~\ref{alg:MOTAPS} we have that $z = (w,\alpha,\tau) \in \R^{n+d+1}.$ 

%

Consider the online SGD method applied to minimizing~\eqref{eq:proxyobjgen} given by
\begin{equation}\label{eq:ztupdateht}
z^{t+1} = z^t - \gamma  \nabla h_{i_t,t}(z^t),
\end{equation}
where $i_t \in \{1,\ldots, n\}$ is sampled uniformly and i.i.d at every iteration and $\gamma >0$ is a step size.
For each method we also proved a  \emph{growth condition} (see Lemmas~\ref{lem:weakgrowthsps},~\ref{lem:TAPSweakgrowth}, ~\ref{lem:MOSTAPSweakgrowthn1}) that we now state as an assumption.
\begin{assumption}\label{lem:growthgen}
 There exists $G\geq 0$ such that 
\begin{eqnarray}
\E{\norm{\nabla h_{i_t,t}(z^t)}^2 } &\leq & 2G\, h_t (z^t). \label{eq:hweakgrow}
\end{eqnarray}
\end{assumption}

\subsection{General Convergence Theory}
 
 Here we present two general convergence theorems that will then be applied to our three algorithms.
The first theorem relies on a weak form of convexity known as star convexity.
\begin{theorem} [Sublinear]\label{theo:onlinesgdstar}
Suppose Assumption~\ref{lem:growthgen} holds with $G>0.$
  Let $\gamma < 1/G$ 
and  suppose there exists  $z^*$ such that 
  $h_t$  is star-convex at $z^t$ and around $z^*$, that is
\begin{equation}\label{eq:hstar}
h_t(z^*) \; \geq \; h_t(z^t) + \dotprod{\nabla h_t(z^t), z^* -z^t}. 
\end{equation}
 Then we have that
\begin{align}
\min_{t=1,\ldots, k} \E{h_t(z^t) -h_t(z^*)}
& \leq \; \frac{1}{k} \frac{1}{2\gamma(1-G\gamma) }\E{\norm{z^{0} -z^*}^2} +\frac{G\gamma}{1-G\gamma}\frac{1}{k} \sum_{t=1}^k h_t(z^*).\label{eq:hconv}
\end{align}
\end{theorem}
 Our second theorem relies on a weakened form of strong convexity known as strong star-convexity.
\begin{theorem} [Linear Convergence]\label{theo:onlinesgdstrongstar}
Suppose Assumption~\ref{lem:growthgen} holds with $G>0.$
  Let $\gamma \leq 1/G$.
If there exists  $\mu>0$ and $z^*$ such that $h_t$ is $\mu$--strongly star--convex along $z^t$ and around $z^*$, that is
\begin{equation}\label{eq:hquasistrongmain}
h_t(z^*) \; \geq \; h_t(z^t) + \dotprod{\nabla h_t(z^t), z^* -z^t} + \frac{\mu}{2}\norm{z^*-z^t}, 
\end{equation}
then 
\begin{eqnarray}\label{eq:hconvsstrongmain}
\E{\norm{z^{t+1} -z^*}^2} & \leq &   (1-\gamma \mu)^{t+1} \norm{z^{0} -z^*}^2 +2 G\gamma^2 \sum_{i=0}^{t}(1-\gamma \mu)^i \E{ h_i(z^*)}.
\end{eqnarray}
\end{theorem}
In the next three sections we  specialize these theorems, and their assumptions, to the \SP, \TAPS and \MOTAPS methods, respectively. In particular,
in Section~\ref{sec:SPStheorymain} we show how two previously known convergence results for \SP are special cases of Theorem~\ref{theo:onlinesgdstar} and~\ref{theo:onlinesgdstrongstar}.
In Section~\ref{sec:TAPStheorymain}  we show that the auxiliary functions of \TAPS and \MOTAPS in~\eqref{eq:proxyobjalpha} and~\eqref{eq:proxyobjalphatau} are locally strictly convex under a small technical assumption. In Section~\ref{sec:MOTAPStheorymain} we finally prove convergence of \MOTAPS.
%

\subsection{Convergence of SPS}
\label{sec:SPStheorymain}

Before establishing the convergence of \SP, we start by stating a slightly more general interpolation assumption as follows.
\begin{assumption}[Interpolation]\label{ass:interpolate}
We say that the interpolation assumption holds when\disteq
\begin{equation}\label{eq:interpolate}
\exists w^* \in \cW^* \quad \mbox{such that} \quad  f_i(w^*) = \min_{w \in \R^d} f_i(w), \quad  \mbox{for }i=1,\ldots, n.
\end{equation}
\end{assumption}
%
Here we specialize Theorems~\ref{theo:onlinesgdstar} and~\ref{theo:onlinesgdstrongstar} to the SP method~\eqref{eq:SPS}. Both of these theorems rely on assuming that the proxy function $h_t$ is star-convex or strongly star-convex. Thus first we establish sufficient conditions for this to hold.

\begin{lemma}\label{lem:SPShsconvex}
Let the interpolation Assumption~\ref{ass:interpolate} hold.
If every $f_i$ is star convex   along the iterates $w^t$ given by~\eqref{eq:SPS}, that is, 
\begin{equation}\label{eq:starcvxfi}
f_i(w^*) \geq f_i(w) +
 \dotprod{\nabla f_{i}(w), w^* -w }
\end{equation} 
then $h_{i,t}(w)$ is star convex along the iterates $w^t$ and around $w^*$, that is.
\begin{equation} \label{eq:fiwwstarconvexmain}
h_{i,t}(w^*) \geq h_{i,t}(w^t) +
 \dotprod{\nabla_w  h_{i,t}(w^t), w^* -w }.
\end{equation}

Furthermore if $f_i$ is $\mu_i$-strongly convex and $L_i$--smooth  then 
 $h_t(w)$ is $\frac{1}{2n}\sum_{i=1} \frac{\mu_i}{L_i}$--strongly star-convex, that is
\begin{equation}\label{eq:spshtstrongconvexmain}
 h_{t}(w^*) \geq h_{t}(w^t) +
 \dotprod{\nabla_w  h_{t}(w^t), w^* -w } +\frac{1}{4n}\sum_{i=1}^n\frac{\mu}{L_i} \norm{w^t-w^*}^2.
\end{equation}
\end{lemma}

Using Lemma~\ref{lem:SPShsconvex}, the convergence of \SP is now  a corollary of Theorems~\ref{theo:onlinesgdstar} and~\ref{theo:onlinesgdstrongstar}.
\begin{corollary}[Convergence of \SP]\label{cor:SPSconvexconv}
If  $\gamma <1$ and every $f_i(w)$ is star-convex along the iterates $w^t$ given by~\eqref{eq:SPS} then
\begin{equation} \label{eq:interstarcvxsgdview}
 \frac{1}{k}\sum_{t=0}^k \frac{1}{2n}\sum_{i=1}^n \E{\left(\frac{f_i(w^t)-f_i(w^*)}{\norm{\nabla f_i(w^t)}}\right)^2} \; \leq \; \frac{1}{k}\frac{1}{2\gamma(1-\gamma) }\E{\norm{w^{0} -w^*}^2}.
\end{equation} 
Furthermore if  the interpolation Assumption~\ref{ass:interpolate} holds and if each $f_i(w)$ is $L_i$--smooth 
then 
\begin{equation}\label{eq:spsstarconvfinalresultmain}
 \min_{t=0 , \ldots, k} \E{f(w^t)-f^*}  \; \leq \; \frac{1}{k}\frac{L_{\max}}{2\gamma(1-\gamma) }\E{\norm{w^{0} -w^*}^2},
\end{equation} 
where  $L_{\max} \; \eqdef\; \max_{i=1,\ldots, n} L_i$.
\end{corollary}
The resulting convergence in~\eqref{eq:spsstarconvfinalresultmain} has already appeared in Theorem 4.4 in~\cite{SGDstruct}. In~\cite{SGDstruct}, the authors use a proof technique that relies on a new notion of smoothness (Lemma 4.3 in~\cite{SGDstruct}). Here we have that~\ref{cor:SPSconvexconv} is rather a direct consequence of interpreting \SP as a type of SGD method. 

\begin{corollary} \label{cor:SPSstrongconvexconv}
If $\gamma \leq 1$, the interpolation Assumption~\ref{ass:interpolate} holds, and every $f_i$ is $L_i$--smooth and $\mu$--strongly star-convex then the iterates $w^t$ given by~\eqref{eq:SPS} converge linearly according to
 \begin{eqnarray}\label{eq:SPSlincomnvergemain}
\E{\norm{w^{t+1} -w^*}^2} & \leq &   \left(1-\gamma \frac{1}{2n}\sum_{i=1}^n \frac{\mu_i}{L_i}\right)^{t+1} \norm{w^{0} -w^*}^2 .
\end{eqnarray}
\end{corollary}
This corollary shows that Theorem D.3 in~\cite{SGDstruct} is a  special case of Theorem~\ref{theo:onlinesgdstrongstar}, and again a direct result of interpreting \SP as a type of SGD method. The rate of convergence in ~\eqref{eq:SPSlincomnvergemain} is also tighter than the analysis given in Theorem 3.1 in~\cite{SPS} where the rate is $1-\frac{\gamma}{2} \frac{\frac{1}{n}\sum_{i=1}^n \mu_i}{L_{\max}}.$

\subsection{Convergence of \TAPS}
\label{sec:TAPStheorymain}
Here we explore the consequences of Theorems~\ref{theo:onlinesgdstar} and~\ref{theo:onlinesgdstrongstar} to the \TAPS method.
To this end,  let 
$z \eqdef (w, \alpha)$ and let
\begin{equation}\label{eq:hz}
 h_t(z) =  h_t(w,\alpha) \; \eqdef \; \frac{1}{n+1}\left(\sum_{i=1}^n \frac{1}{2}\frac{(f_i(w) -\alpha_i)^2}{\norm{\nabla f_i(w^t)}^2+1}  + \frac{n}{2}(\overline{\alpha} - \tau)^2 \right). 
\end{equation}
 As a first step, we need to determine sufficient conditions for this auxiliary function $h_t$ of \TAPS
 to be star-convex. 
The relationship between the convexity or star-convexity of $f_i$ and $h_t$ is highly nontrivial.
 This is because the star-convexity $h_t$ is not a consequence of $f_i$ being star-convex, nor the converse. Instead, star-convexity of $h_t$ translates to new nameless assumptions on the $f_i$ functions.
 As an insight into this difficulty, supposing that each $f_i$ is convex is not enough to guarantee that~$h_t$ is convex.
As a simple  counterexample, let $n =1$ and $f_1(w) = w^2$. Thus $\tau =0$ and from~\eqref{eq:hz} we have
\[h_t(z) = \frac{1}{2} \left(\frac{1}{2} \frac{(w^2 -\alpha_1)^2}{\norm{w^t}^2 +1} + \frac{1}{2} \alpha_1 ^2 \right).\]
 It is easy to show that for $\alpha_1$ large enough, the Hessian of $h_t$ has a negative eigenvalue, and thus $h_t$ is non-convex. Conversely, $h_t$ can have local convexity even when the underlying loss function is arbitrarily non-convex, as we show in the next lemma and corollary.
%

\begin{lemma} [Locally Convex]\label{lem:convexTAPS}
Consider the iterates of Algorithm~\ref{alg:MOTAPS}.
Let $(w,\alpha) \in \R^{d+n}$ and consider $h_t(w,\alpha)$ defined in~\eqref{eq:hz} .
 Assume that the  gradients at $w$ span the entire space, that is
\begin{equation}\label{eq:gradspant}
\mbox{span}\left\{\nabla f_1(w), \ldots, \nabla f_n(w)\right\} \; = \; \R^{d}, \quad \forall w.
\end{equation}
If Assumption~\ref{ass:target} holds,  every $f_i(w)$ for $i=1,\ldots, n$ is twice continuously differentiable  and
\begin{equation}\label{eq:hessifiminusalph}
\frac{1}{n+1}\sum_{i=1}^n\nabla^2 f_i(w^t) \frac{f_i(w^t)-\alpha_i^t}{\norm{\nabla f_i(w^t)}^2 +1}  \succeq 0,
\end{equation}
then $h_t$ is strictly convex at $(w^t,\alpha^t)$ with 
\[ \nabla^2  h_t(w^t,\alpha^t) \; \succ \; 0.\]
\end{lemma}

 The condition on the span of the gradients~\eqref{eq:gradspant} typically holds in the setting where we have more data then dimensions (features). Fortunately this occurs in precisely the setting where \TAPS makes most sense since it makes sense to apply \TAPS when $f^* = \tau >0$. This can only occur in the \emph{underparametrized} setting, where we have more data than features.  

 The condition in Lemma~\ref{lem:convexTAPS} that is difficult to verify is~\eqref{eq:hessifiminusalph}.
 A sufficient condition for~\eqref{eq:hessifiminusalph} to hold is  
 \begin{equation}\label{eq:anop8j3oa3}
\nabla^2 f_i(w^t) (f_i(w^t)-\alpha_i^t) \; \succeq \;  0.
\end{equation}
Since $\alpha_i^t$ are essentially tracking $f_i(w^t)$ (see line~\ref{ln:alphaistep} in Algorithm~\ref{alg:TAPS} ),  we can state~\eqref{eq:anop8j3oa3}  in words as: if $\alpha_i^t$ is underestimating $f_i(w^t)$ then $f_i$ should be convex at $w^t$, and conversely if  $\alpha_i^t$ is overestimating $f_i(w^t)$ then  $f_i$ should be concave at $w^t$. 

There is one point where~\eqref{eq:hessifiminusalph} holds trivially, and that is at every point such that  $\alpha_i = f_i(w).$ This includes every  minimizer $(w^*,\alpha^*)$ since by Lemma~\ref{lem:reformsgdtasps} we have that $\alpha_i^* = f_i(w^*).$
 Consequently, as we state in the following corollary, under minor technical assumptions,
we have that $h_t(w,\alpha)$ has no \emph{degenerate} local minimas. This shows that $h_t$ has some local convexity.

\begin{corollary} [ Locally Strictly Convex \TAPS]\label{cor:localconvexTAPS}
Consider the iterates of the \TAPS method given in Algorithm~\ref{alg:TAPS}.
Let  $w^*$ be a minimizer of~\eqref{eq:main} and let $\alpha_i^* = f_i(w^*)$ for $i=1,\ldots, n.$
Assume that the gradients at $w^*$ span the entire space, that is
\begin{equation}\label{eq:gradspan}
\mbox{span}\left\{\nabla f_1(w^*), \ldots, \nabla f_n(w^*)\right\} \; = \; \R^{d}.
\end{equation}
If Assumption~\ref{ass:target} holds and if every $f_i(w)$ for $i=1,\ldots, n$ is twice continuously differentiable 
then
$\nabla^2 h_t(w,\alpha) \succ 0$ and thus $h_t(w,\alpha)$ is \emph{strictly convex} at $(w^*,\alpha^*).$
\end{corollary} 

Next we specialize Theorem~\ref{theo:onlinesgdstar} to the \TAPS method in the following corollary.
\begin{corollary}[Sublinear Convergence of \TAPS]\label{theo:TAPSconvsum}
%
%

\label{cor:TAPSconvsum}
Let $h_t(z)$ in~\eqref{eq:proxyobjalpha} be star--convex~\eqref{eq:hstar} around $z^* = (w^*, \alpha^*)$ and along the  iterates $z^t = (w^t, \alpha^t)$ of Algorithm~\ref{alg:TAPS}.
If $\gamma < 1$ and $f_i(w)$ is $L_{\max}$--Lipschitz then
\begin{equation} \label{eq:TAPSconvL}
\min_{t=1,\ldots, k} \frac{1}{n+1}  \left(\sum_{i=1}^n\frac{\E{f_i(w^t) -\alpha_i^t}^2}{L_{\max}+1}  + \E{\overline{\alpha}^t - \tau}^2 \right)\; \leq \; \frac{1}{k}\frac{1}{\gamma(1-\gamma) }\E{\norm{w^{0} -w^*}^2}.
\end{equation} 

Alternatively, if  $h_t(z)$ is  $\mu$--strongly star--convex~\eqref{eq:hquasistrongmain}  then 
\begin{equation}\label{eq:convTAPSstrconv}
\E{\norm{w^t-w^*}^2 + \sum_{i=1}^n\norm{\alpha^t_i -f_i(w^*)}^2 } \leq (1- \gamma \mu)^t \left(\norm{w^0-w^*}^2 + \sum_{i=1}^n\norm{\alpha^0_i -f_i(w^0)}^2\right).
\end{equation}

\end{corollary}

\subsection{Convergence of \MOTAPS}
\label{sec:MOTAPStheorymain}

Here we explore the consequences of Theorems~\ref{theo:onlinesgdstar} and~\ref{theo:onlinesgdstrongstar} specialized to Algorithm~\ref{alg:MOTAPS}.   In this case, the proxy function $h_t(z) = h_t(w,\alpha,\tau)$ is
\begin{equation}\label{eq:hztaumain}
 h_t(z) \; \eqdef \; \frac{1}{n+1}\left(\sum_{i=1}^n \frac{1-\lambda}{2}\frac{(f_i(w) -\alpha_i)^2}{\norm{\nabla f_i(w^t)}^2+1}  + \frac{n(1-\lambda)}{2}(\overline{\alpha} - \tau)^2+\frac{\lambda}{2}\tau^2 \right).
\end{equation}

Before applying Theorems~\ref{theo:onlinesgdstar} and~\ref{theo:onlinesgdstrongstar} we should verify when $h_t(z)$ is star-convex or convex. This turns out to be much the same task as verifying that the auxiliary function for \TAPS given in~\eqref{eq:hz} is star-convex. This is because the only difference between the two functions is that~\eqref{eq:hztaumain}  has an additional $\frac{\lambda}{2}\tau^2$ which adds strong convexity in the new $\tau$ dimension. Thus the discussion and results around Lemma~\ref{lem:convexTAPS} and Corollary~\ref{cor:localconvexTAPS} remain largely true for~\eqref{eq:hztaumain}. That is, we are only able to establish when $h_t$ is locally convex.

For the remainder of this section we impose that the dampening parameter satisfies
 \begin{equation} \label{eq:lambdabound}
\lambda \leq \frac{2n+1}{2n+3} < 1,
\end{equation}	
so that we can apply Lemma~\ref{lem:MOSTAPSweakgrowthn1}. 
In our forthcoming corollaries we will prove convergence of \MOTAPS to the point $z^* = (w^*,\alpha^*, \tau^*)$ where
 $w^*$ is a minimizer of~\eqref{eq:main} and 
\begin{equation} \label{eq:alphastarandtaustarapp}
\alpha_i^* \eqdef f_i(w^*) \quad \mbox{and} \quad \tau^* = f(w^*) , \quad \mbox{for }i=1,\ldots, n.
\end{equation}

First we develop a corollary based on Theorem~\ref{theo:onlinesgdstar}.
\begin{corollary}\label{cor:MTAPSconvsum}
Let $\lambda\in [0,\;1]$ satisfy~\eqref{eq:lambdabound} and  
let $z^t \eqdef (w^t, \alpha^t, \tau_t)$ be the iterates of Algorithm~\ref{alg:MOTAPS} when using a stepsize $\gamma  =  \frac{1}{2(1-\lambda)(2n+1)}$ and $ \gamma_{\tau} = \gamma \left( \lambda+ (1-\lambda)n\right).$
If  $h_t(z)$ is star convex along the iterates $z^t $ and around $z^* \eqdef (w^*, \alpha^*, \tau^*) $ then
\begin{equation}\label{eq:MTAPSconvh}
\min_{t=0,\ldots, k} \E{h_{t}(z^t)-h_t(z^*)} \; \leq \; \frac{2(1-\lambda)(2n+1)}{k} \norm{z^{0} -z^*}^2+\frac{\lambda f(w^*)^2}{2(n+1)}.
\end{equation} 
Furthermore, if $f_i$ is $L_{\max}$--Lipschitz then
\begin{align}
 \frac{1}{n+1}\E{\sum_{i=1}^n \frac{1}{2}\frac{(f_i(w^t) -\alpha_i^t)^2}{L_{\max}+1}  + \frac{n}{2}(\overline{\alpha}^t - \tau^t)^2+\frac{\lambda}{2}\left( (\tau^t)^2  - f(w^*)^2\right)} \nonumber \\
 \quad \quad  \; \leq \;\frac{2(1-\lambda)(2n+1)}{k} \norm{z^{0} -z^*}^2+\frac{\lambda f(w^*)^2}{2(n+1)}. \label{eq:hztaufilipschitz}
 \end{align}
\end{corollary}
This Corollary~\ref{cor:MTAPSconvsum} shows that $(f_i(w^t),\overline{\alpha}^t,\tau^t)$ converges to $(\alpha_i^t, \tau^t, f(w^*))$ sublinearly up to an additive error $\frac{\lambda f(w^*)^2}{2(n+1)}$ which is controlled by $\lambda$: When $\lambda$ is very small, this additive error is very small. But $\lambda$ also controls the \emph{speed} of convergence. Indeed for $\lambda$ close to $1$ the method converges faster up to this additive error. Thus $\lambda$ controls a trade-off between speed of convergence and radius of convergence.

The next corollary is based on Theorem~\ref{theo:onlinesgdstrongstar}.
\begin{corollary}\label{cor:MTAPSconvhstrongmain} 
Let $\lambda\in [0,\;1]$ satisfy~\eqref{eq:lambdabound} and  
let $z^t \eqdef (w^t, \alpha^t, \tau_t)$ be the iterates of Algorithm~\ref{alg:MOTAPS} when using a stepsize $\gamma  =  \frac{1}{2(1-\lambda)(2n+1)}$ and $ \gamma_{\tau} = \gamma \left( \lambda+ (1-\lambda)n\right).$
If $h_t(z)$ is $\mu$--strongly star--convex along the iterates $z^t $ and around $z^* \eqdef (w^*, \alpha^*, \tau^*) $ then
\begin{eqnarray}\label{eq:MTAPSconvhstrongmain}
\E{\norm{z^{t+1} -z^*}^2} & \leq &  \big(1- \frac{\mu}{(1-\lambda)(2n+1)} \big)^{t+1} \norm{z^{0} -z^*}^2  +  \frac{\lambda f(w^*)^2}{\mu(n+1)}.
\end{eqnarray}
\end{corollary}

In both Corollary~\ref{cor:MTAPSconvsum} and~\ref{cor:MTAPSconvhstrongmain}  the $\lambda$ parameter controls a trade-off between speed of convergence and an additive error term. For example, for the largest value $\lambda =  \frac{2n+1}{2n+3} $ (due to~\eqref{eq:lambdabound}) we have that ~\eqref{eq:MTAPSconvhstrongmain}, after some simplifications, gives
\[
\E{\norm{z^{t+1} -z^*}^2} \; \leq \;  \big(1- \mu/2 \big)^{t+1} \norm{z^{0} -z^*}^2  + \frac{ f(w^*)^2}{\mu(n+1)}.\]
Thus the convergence rates is now $1-\mu/2$ and independent of $n$. But the additive error term $ \frac{ f(w^*)^2}{\mu(n+1)}$ is now larger. On the other end, as $\lambda \rightarrow 0$  the rate of convergence tends to $ (1- \frac{\mu}{2n+1} )$, which now depends on $n$, and the additive error term tends to zero.

By controlling this trade-off, next we use Corollary~\ref{cor:MTAPSconvhstrongmain} to establish a total complexity of Algorithm~\ref{alg:MOTAPS}.

\begin{theorem} \label{theo:complexMOTAPS}
Consider the setting of Corollary~\ref{cor:MTAPSconvhstrongmain}.
For a given $\epsilon >0$ it follows that
\begin{equation}
t \; \geq\;  \frac{(1-\lambda)(2n+1)}{\mu}\log\left(\frac{2\norm{z^0-z^*}^2}{\epsilon} \right) \;  \implies \; \E{\norm{z^{t+1} -z^*}^2} < \frac{\epsilon}{2} +    \frac{\lambda f(w^*)^2}{\mu(n+1)}.
\end{equation}
Consequently if we could choose 
\begin{equation}\label{eq:lambdaboundcomplexmain}
\lambda \;< \;  \min \left\{ \frac{ \mu(n+1)}{f(w^*)^2} \frac{\epsilon}{2} , \; \frac{2n+1}{2n+3} \right\}
\end{equation}
then 
\begin{align}
t & \geq\;  \frac{2n+1}{\mu}\log\left(\frac{2\norm{z^0-z^*}^2}{\epsilon} \right)\quad  \implies \quad \E{\norm{z^{t+1} -z^*}^2} < \epsilon.
\end{align}
\end{theorem}
\begin{proof}
By standard arguments using properties of the logarithm, we have that
\[  t \geq  \frac{(1-\lambda)(2n+1)}{\mu}\log\left(\frac{2}{\epsilon} \right)  \; \implies   \big(1- \frac{\mu}{(1-\lambda)(2n+1)} \big)^{t+1}  < \frac{\epsilon}{2}.  \]
See for instance Lemma 11 in~\cite{GowerThesis}. 
Furthermore, by using~\eqref{eq:lambdaboundcomplexmain} we have that
\begin{align}
t & \geq\;   \left( 1-\frac{ \mu(n+1)}{f(w^*)^2} \frac{\epsilon}{2} \right) \frac{2n+1}{\mu}\log\left(\frac{2\norm{z^0-z^*}^2}{\epsilon} \right) \; \geq \;  \frac{2n+1}{\mu}\log\left(\frac{2\norm{z^0-z^*}^2}{\epsilon} \right) \nonumber \\
& \quad \qquad  \implies \quad \E{\norm{z^{t+1} -z^*}^2} < \epsilon.
\end{align}
\end{proof}

Thus by choosing $\lambda$ small enough, we can show that the \MOTAPS method converges linearly. This is in stark contrast to SGD where, despite the presence of an additive error when using a constant step size (See Theorem 1 in~\cite{moulines2011non}), this additive term only vanishes by setting the stepsize to zero. In contrast for \MOTAPS we can set $\lambda$ arbitrarily small without halting the method.

In practice, we would not know how to set $\lambda$ using~\eqref{eq:lambdaboundcomplexmain} since we would not know $f(w^*).$ Furthermore, we may not have a particular $\epsilon$ in mind, and instead, prefer to monitor the error and stop when resources are exhausted. To address both of these concerns, the next theorem offers another 
way to deal with the additive error by eventually decreasing the step size.

\begin{theorem} \label{theo:complexdecreaseMOTAPSmain}
Consider the setting of Corollary~\ref{cor:MTAPSconvhstrongmain}. For a given $\epsilon >0$ 
if we use an iteration dependent stepsize in Algorithm~\ref{alg:MOTAPS} given by
\begin{align}\label{eq:gammatmotaps}
\gamma_t   =  
\begin{cases} \displaystyle
 \frac{1}{(1-\lambda)(2n+1)} & \mbox{if } t \leq 2  (2n+1)\left \lceil \frac{1-\lambda}{\mu} \right\rceil \\[0.5cm]
  \displaystyle  \frac{(t+1)^2 -t^2}{\mu (t+1)^2} & \mbox{if } t \geq 2  (2n+1)\left \lceil \frac{1-\lambda}{\mu} \right\rceil
\end{cases}
\end{align}
and if
\[\lambda \leq  \min \left\{ 1- \frac{2\mu}{2n+1}, \; \frac{2n+1}{2n+3} \right\} .\]
then
\begin{equation}\label{eq:complexdecreaseMOTAPS}
\EE{}{\norm{z^{t} -z^*}^2}  \leq  \frac{(1-\lambda)\lambda f(w^*)^2}{\mu^2} \frac{16}{t} + \frac{4(2n+1)^2}{e^2 t^2} \left \lceil \frac{1-\lambda}{\mu} \right\rceil^2 \norm{z^{0} -z^*}^2.
\end{equation}
\end{theorem}
This Theorem~\ref{theo:complexdecreaseMOTAPSmain} relies on knowing $\mu$ to set a \emph{switching point} and the step size in~\eqref{eq:gammatmotaps}. In practice it can also be difficult to estimate $\mu$, but this theorem is still useful in that, it suggests that at some point in the execution we should decrease the stepsize 
\[\gamma_t = \cO\left(\frac{(t+1)^2 -t^2}{(t+1)^2} \right) =  \cO\left(\frac{2t +1}{(t+1)^2} \right)  = \cO\left(\frac{1}{t+1}\right),\]
much in the same way that SGD is used in practice.
\section{Experiments}


\begin{figure}
     \centering
     \begin{subfigure}[b]{0.48\textwidth}
         \centering
         \includegraphics[width=0.45\textwidth]{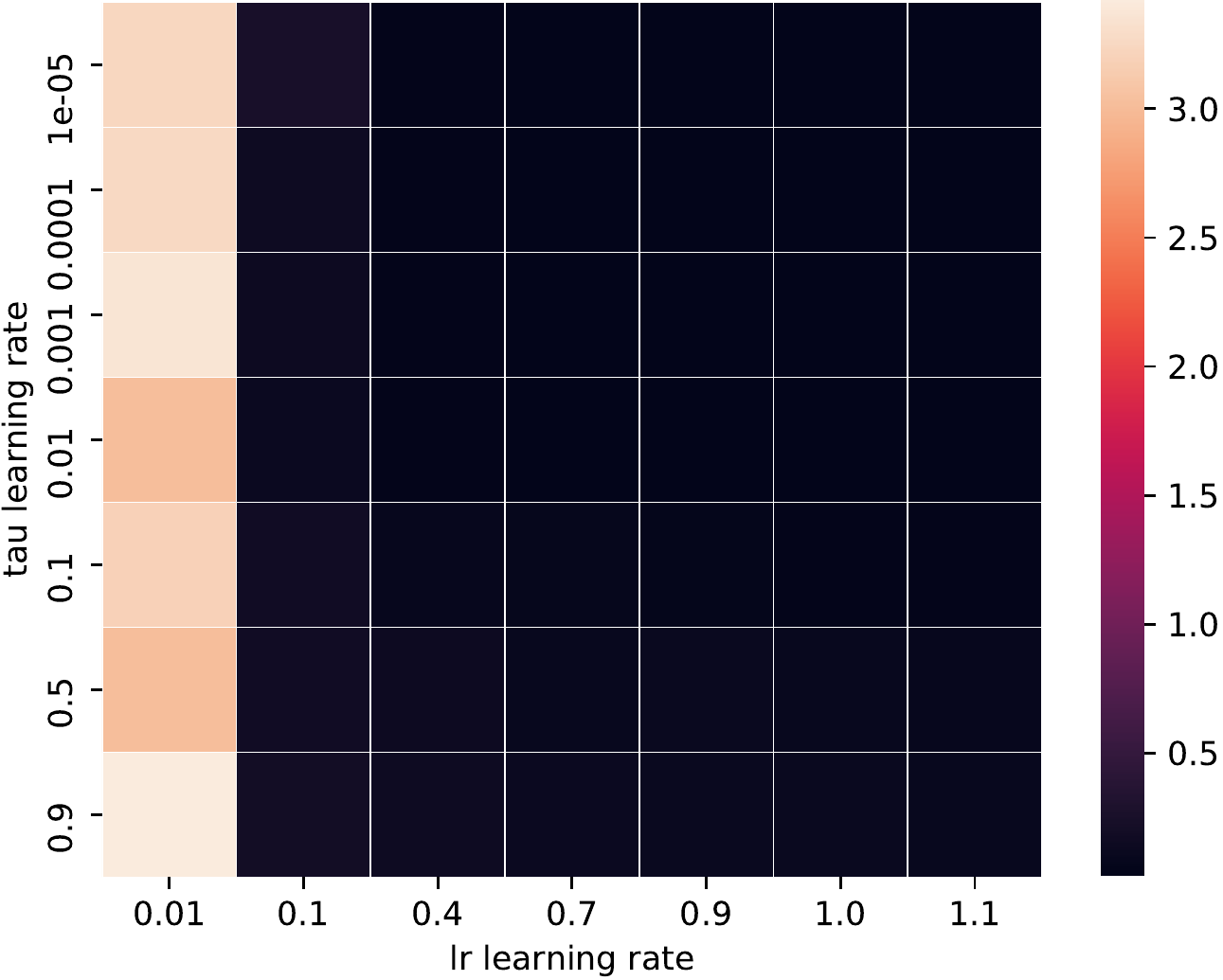}
		\includegraphics[width=0.45\textwidth]{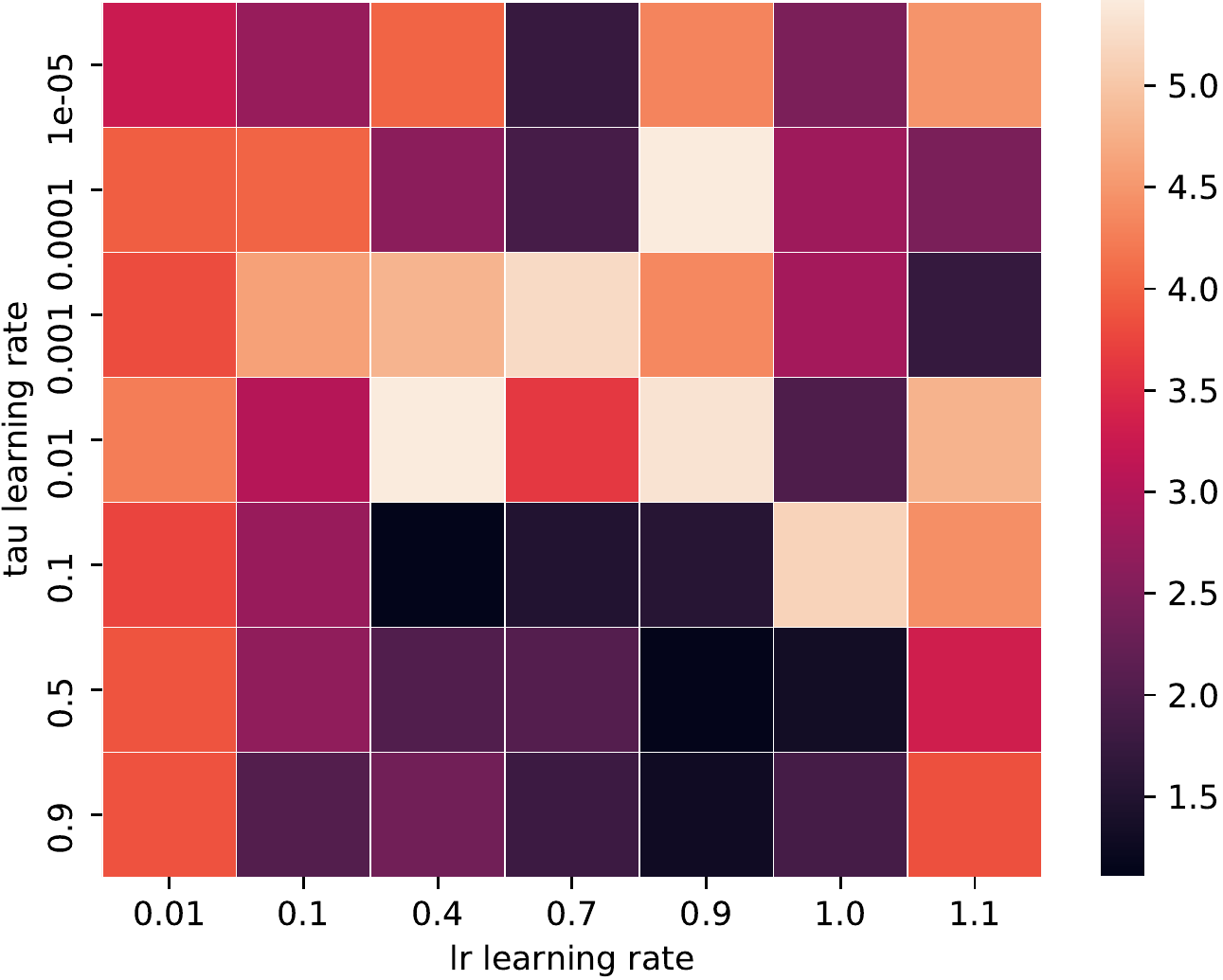}
        \caption{\texttt{duke} $(n,d) =( 7130, 44)$ Left:  $\sigma =0.0$. Right: $\sigma =\min_{i=1,\ldots, n} \norm{x_i}^2/n = 5.06$ }
         \label{fig:lrduke}
     \end{subfigure}
     \hfill
     \begin{subfigure}[b]{0.48\textwidth}
         \centering
		\includegraphics[width=0.45\textwidth]{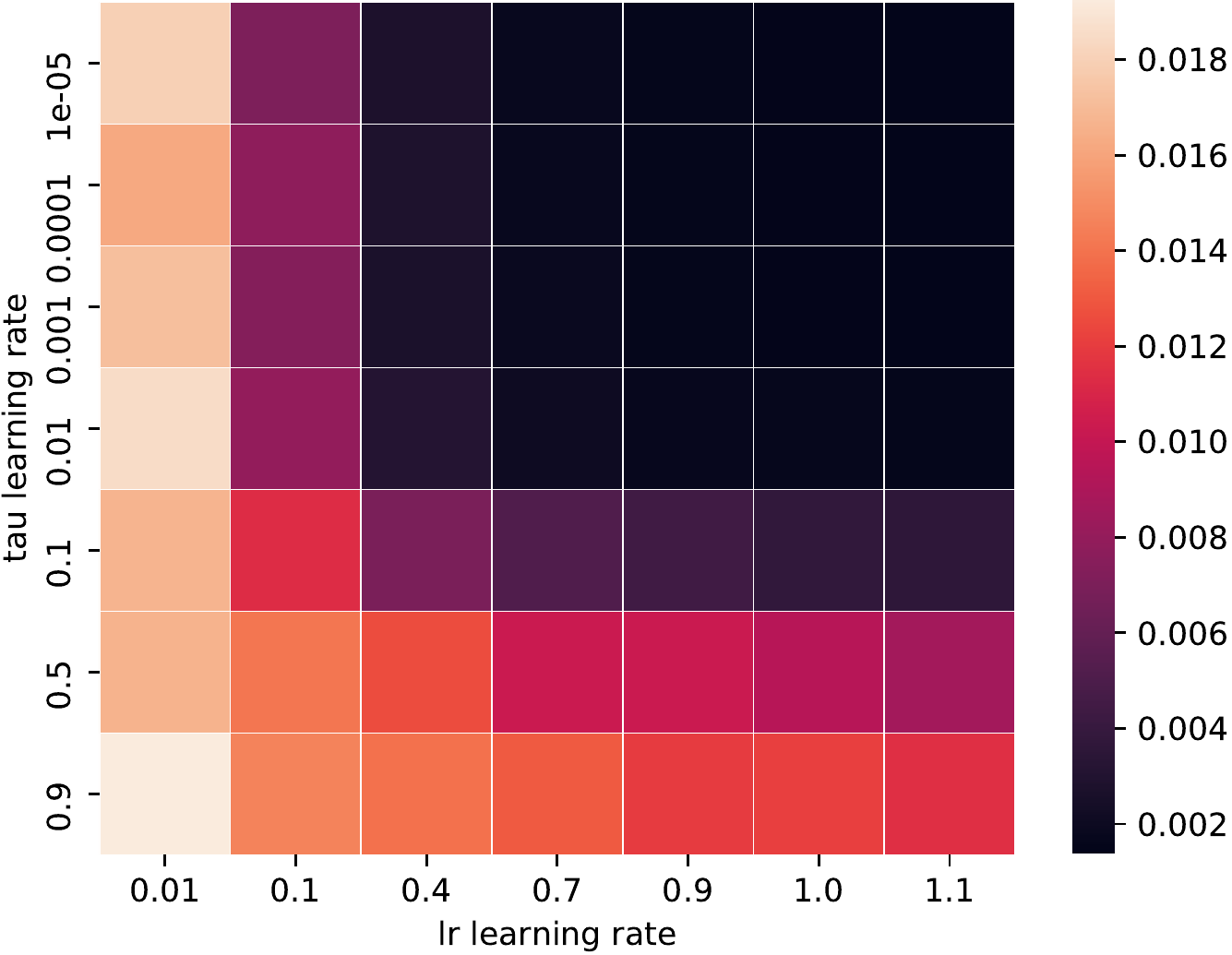}
		\includegraphics[width=0.45\textwidth]{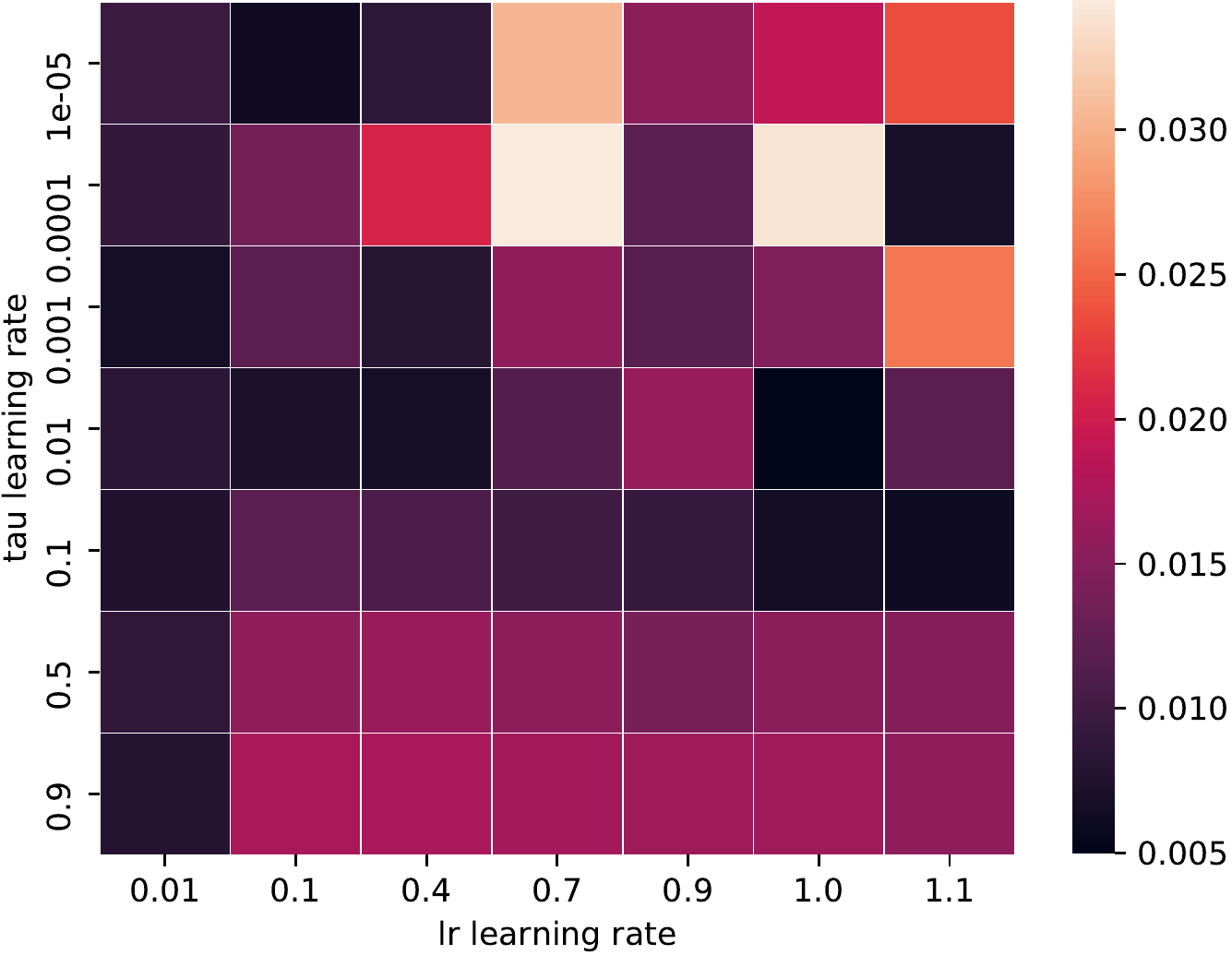}
         \caption{\texttt{mushrooms} $(n,d) = (62, 2001)$ Left:  $\sigma =0.0$. Right: $\sigma =\min_{i=1,\ldots, n} \norm{x_i}^2/n = 2.66$ }
         \label{fig:lrmushrooms}
     \end{subfigure}
        \caption{The resulting gradient norm of \MOTAPS after running 50 epochs on a logistic regression}
        \label{fig:grid}
\end{figure}

\subsection{Convex Classification}
\label{sec:convex}

We first experiment with a classification task using logistic regression. 

For our experiments on convex classification tasks, we focused on logistic regression. That is
\begin{equation} \label{eq:fgenlin}
f(w) = \frac{1}{n} \sum \limits_{i=1}^n \phi (x_i^\top w) +\frac{\sigma}{2}\norm{w}_2^2 \;
\end{equation}
where 
 $
 \phi_i(t) \, =\, \ln\left(1+ e^{-y_i t} \right),
$ 
$(x_i,y_i)\in \R^{d+1}$ are the features and labels for $i=1, \ldots, n$, and $\sigma >0$ is the regularization parameter.
We experimented with the five diverse data sets:  leu~\cite{leu-dataset} , duke~\cite{duke-dataset}, colon-cancer~\cite{coloncancer}, mushrooms~\cite{uci} and phishing~\cite{uci}. Details of these datasets and their properties can be found in Table~\ref{tab:datasetsconv}.
\begin{table}[!h]
\centering
\begin{tabular}{ c|ccc|ccc|cccc }
   \toprule
 &  & & &  \multicolumn{3}{c|}{$\sigma = 0$} &  \multicolumn{4}{c}{$\sigma = \min_{i=1,\ldots, n} \norm{x_i}^2/n$} \\
 dataset     & $d$  & $n$ & $L_{\max}$  & $\gamma^*$ & $\gamma_{\tau}^*$ & $f^*$&  $\gamma^*$ & $\gamma_{\tau}^*$      &  $f^*$& $\sigma$           \\
 \midrule
 leu &  $7130$ &  $38$ &  $824.6$& $1.1$ & $10^{-5}$ & 0.0 & $0.01$ & $0.4$ &0.449& $11.74$  \\
duke & $7130$ & $ 44$ &$683.2$ & $1.1$ & $10^{-3}$ & 0.0 & $0.1$ & $0.4$   &0.4495&  $5.06$ \\
 colon-cancer &   $2001$ &  $62$ & $137.8$  &  $1.1$ & $10^{-5}$    &0.0  &  $0.1$ &$0.9$  & 0.453& $2.66$  \\
 mushrooms   & $112$            &  $8124$ &    $5.5$    &  $1.1$  & $10^{-4}$  &0.0&      $1.0$ & $0.01$      &0.083  &  0.0027 \\ 
  phishing    & $68$             & $11055$   &   $7.75$  &   $0.01 $ & $0.5$   & 0.142& $0.01$ & $0.9$ & 0.188& 0.0028 \\ 
 \bottomrule
\end{tabular}
\caption{Binary datasets used in the logistic regression experiments together with the best parameters settings for $\gamma$ and $\gamma_{\tau}$ for two different regularization settings.}
\label{tab:datasetsconv}
\end{table}

For the sake of simplicity, here we test the \MOTAPS method in Algorithm~\ref{alg:MOTAPS} with $\lambda = 0.5$.
To determine a reasonable parameter setting for  the \MOTAPS methods we performed 
 a grid search over the two parameters $\gamma$ and $\gamma_\tau$. See Figure~\ref{fig:grid}  for the  results of the grid search for an over-parametrized problem \texttt{colon-cancer} and an under-parametrized problem \texttt{mushrooms}.
  Through these grid searches we found that 
the determining factor for setting the best stepsize was the magnitude of the regularization parameter $\sigma>0$.
If $\sigma$ was small or zero then
 $\gamma =1$ and $\gamma_{\tau} =0.001$ resulted in a good performance. On the other hand, if $\sigma$ is large then $\gamma =0.01$ and $\gamma_{\tau} =0.9$ resulted in the best performance.  This is most likely due to the effect that $\sigma$ has on the optimal value $f(w^*),$ as is also clear in Table~\ref{tab:datasetsconv}.

\subsection{Comparison to Variance Reduced Methods}

We compare our methods against SGD, and two variance reduced gradient methods SAG~\cite{SAG,SAGA_Nips} and SVRG~\cite{Johnson2013} which are among the state-of-the-art methods for minimizing logistic regression. For setting the parameters  for SGD, based on~\cite{gower2019sgd} we used the learning rate schedule $\gamma_t = L_{\max}/t$ where $L_{\max}$ is the smoothness constant.  For SVRG and SAG we used $\gamma =1/2L_{\max}$. For \SP and \TAPS we used $\gamma =1 $ and approximated $f_i(w^*) =0$. Because of this the \SP is equivalent to the SPS method given in~\cite{SPS}. Following~\cite{SPS} experimental results, we also  implemented \SP with a max stepsize rule\footnote{In~\cite{SPS} the authors also recommend the use of a further \emph{smoothing} trick, but we opted for simplicity and chose not to use this smoothing. }.  For   \MOTAPS, based on our observations in the grid search,  we used
the rule of thumb  $\gamma = 1.0/(1+ 0.25 \sigma e^{\sigma})$ and  $\gamma_{\tau} = 1-\gamma$.
 We compare all the algorithms in terms of epochs (effective passes over the data) in Figure~\ref{fig:VRcompare}. 
 We found that in under-parametrized problem such as the mushrooms data set in  Figure~\ref{fig:VRcompare}, and problems with a large regularization, SAG and SVRG were often the most efficient methods.  For over-parametrized problems such as \texttt{duke}, with moderate regularization, the \MOTAPS methods was the most efficient. Finally, for over-parametrized problems with very small regularization the \SP method was the most efficient, see Section~\ref{asec:exp-convex}.

\begin{figure}
\centering
\includegraphics[width=0.23\textwidth]{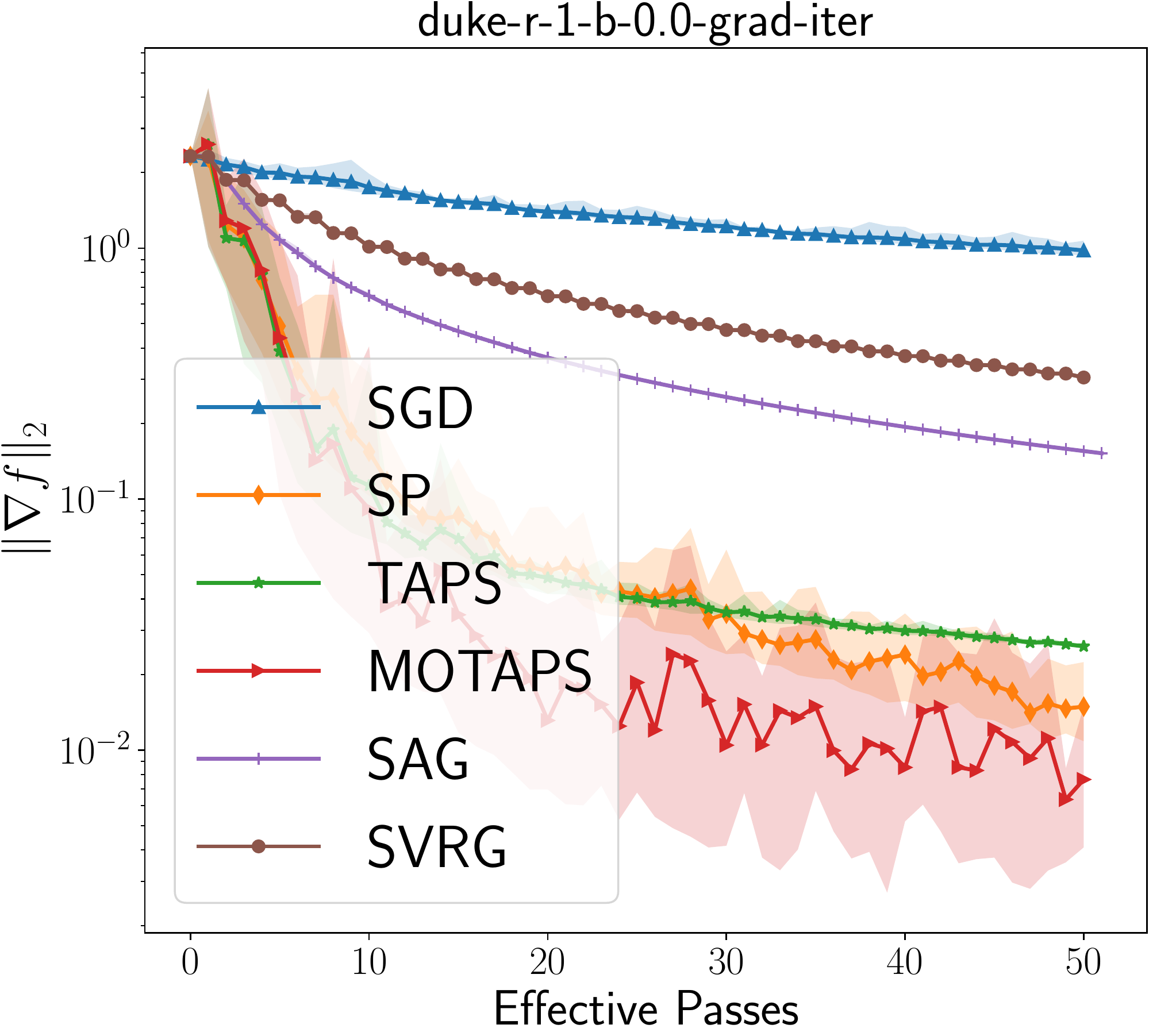}
\includegraphics[width =0.23\textwidth]{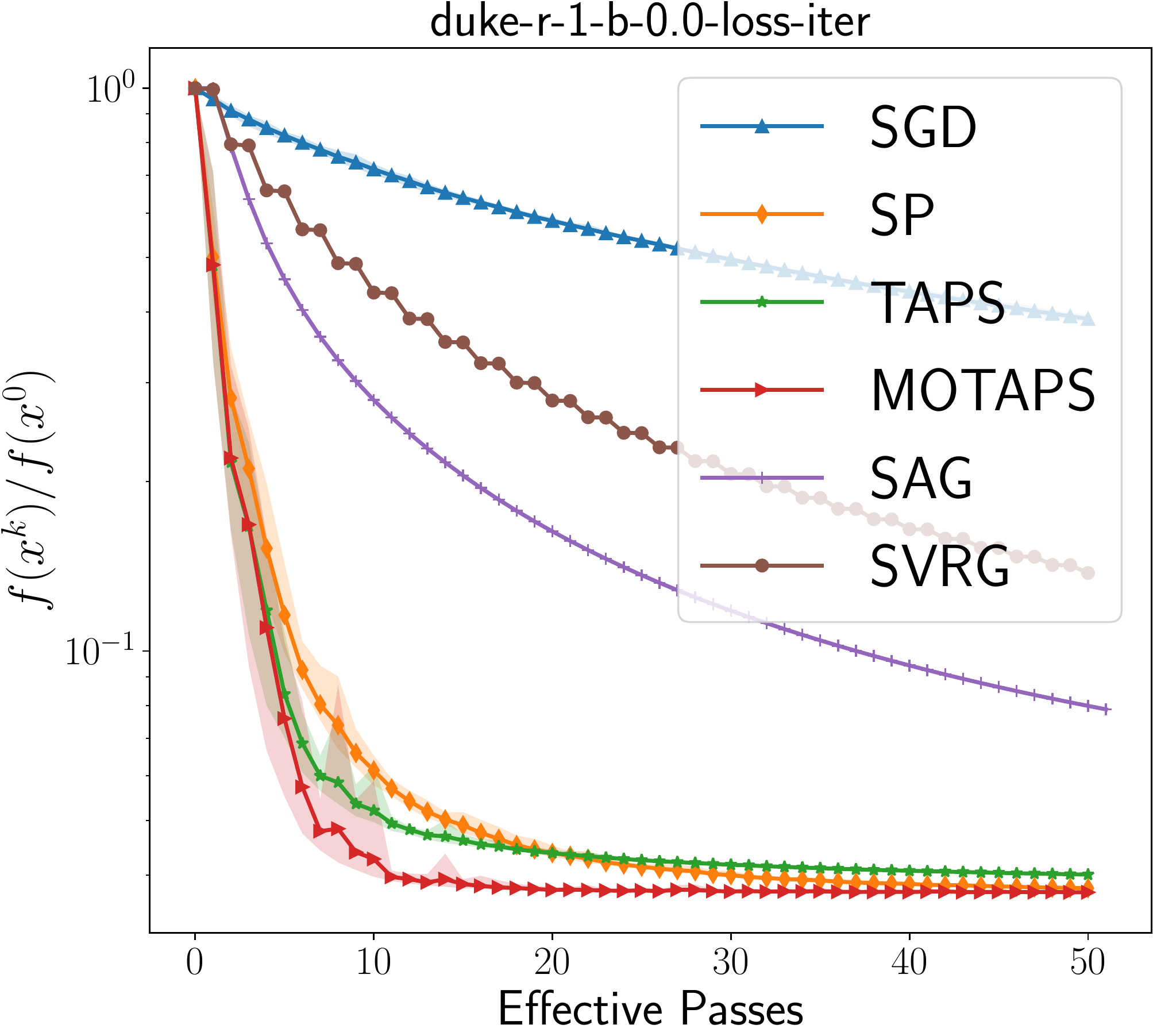}
\includegraphics[width=0.23\textwidth]{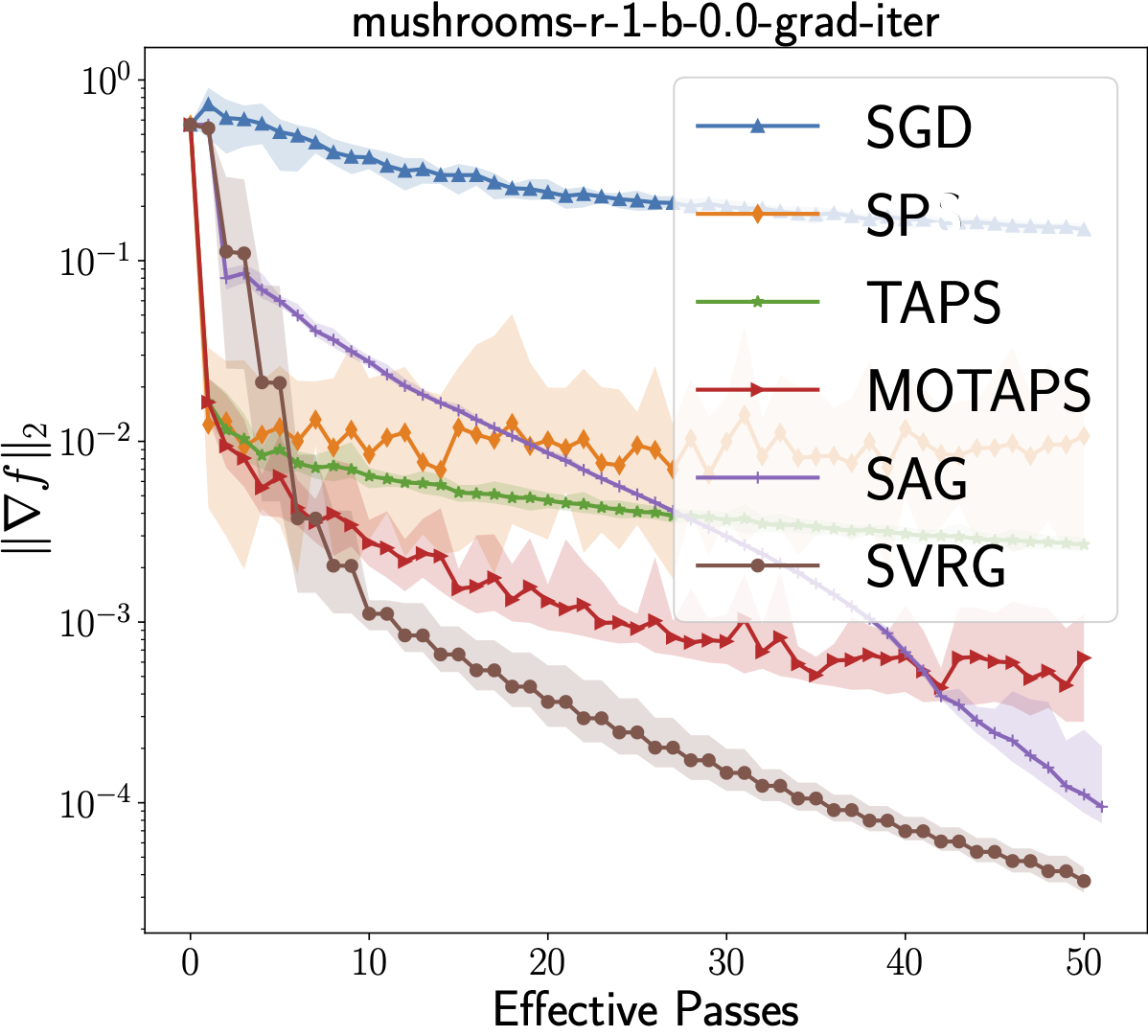}
\includegraphics[width =0.23\textwidth]{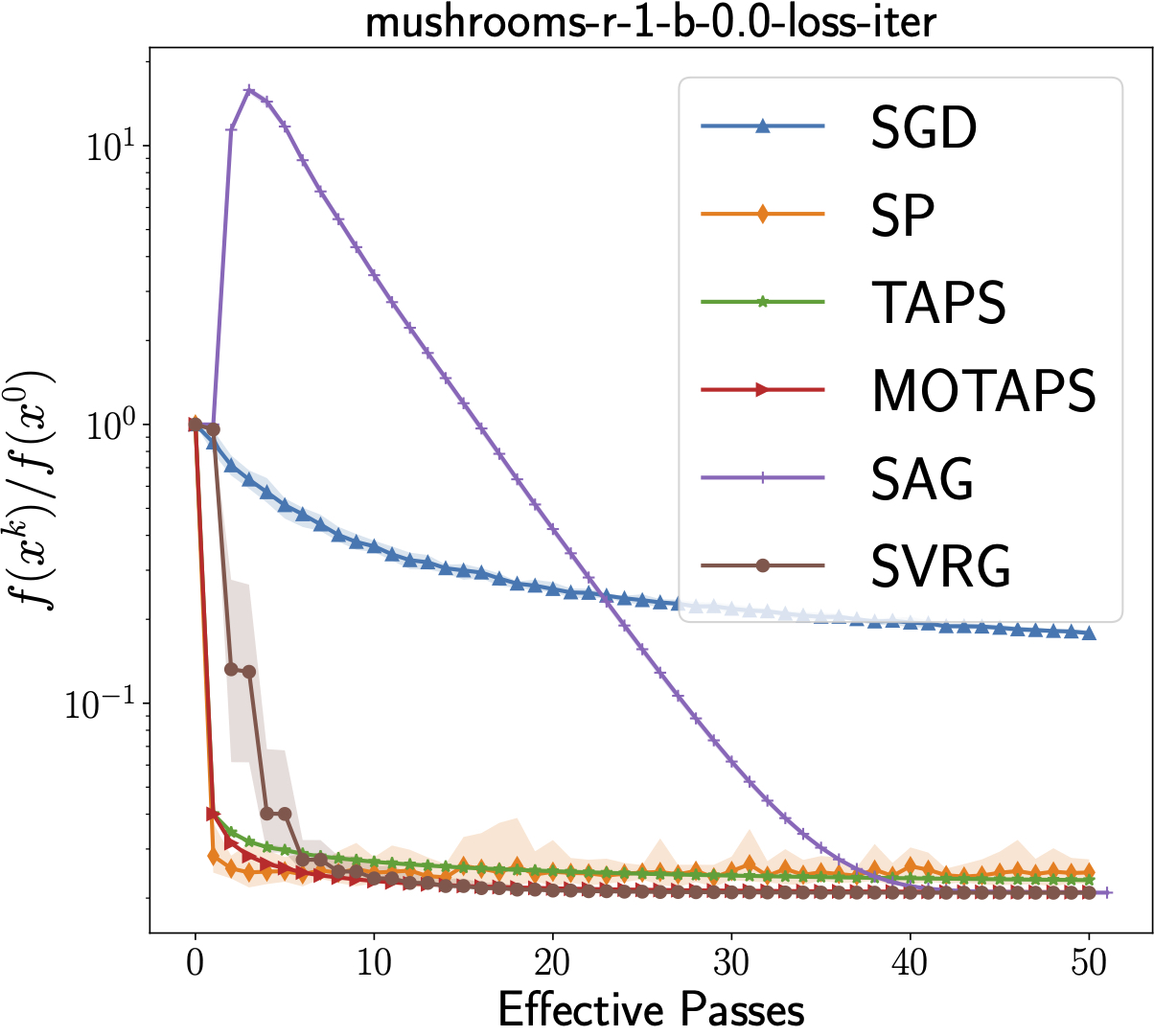}
\caption{Logistic Regression with data set  Left: duke $(n,d) = (44, 7130)$ and Right:   mushrooms $(n,d) = (8124, 113)$
with  regularization $\sigma =1/n$.
}
\label{fig:VRcompare}
\end{figure}

Furthermore  \MOTAPS has two additional advantages over SAG and SVRG 1) setting the stepsize does not 
 require computing the smoothness constant and 2) does not require storing a $n\times d$ table of gradient (like SAG) or doing an occasional full pass over the date (like SVRG). 
  We also found that adding  momentum to  \SP and \MOTAPS could speed up the methods. See Section~\ref{asec:mom} for details on how we added momentum and additional experiments.

\subsection{Deep learning tasks}

We preformed a series of experiments on three benchmark problems commonly used for testing optimization methods for deep learning. CIFAR10 \cite{cifar} is a computer vision classification problem and perhaps the most ubiquitous benchmark in the deep learning.
 We used a large and over-parameterized network for this task, the 152 layer version of the pre-activation ResNet architecture~\cite{preact}, which has over 58 million parameters.

For our second problem, we choose an
 under-parameterized computer vision task. The street-view house numbers dataset \cite{svhn} is similar to the CIFAR10 dataset, consisting of the same number of classes, but with a much larger data volume of over 600k training images compared to 50k. To ensure the network can not completely interpolate the data, we used a much smaller ResNet network with 1 block per layer and 16 planes at the first layer, so that there are fewer parameters than data-points.

For our final comparison we choose one of the most popular NLP benchmarks, the IWSLT14 english-german translation task \cite{iwslt14}, consisting of approximately 170k sentence pairs. This task is relatively small scale and so overfitting is a concern on this problem. We applied a modern Transformer network with embedding size of 512, 8 heads and 3/3 encoding/decoding layers.

In each case the minimum loss is unknown so for the \TAPS method we assume it is 0. Due to a combination of factors including the use of data-augmentation and L2 regularization, this is only an approximation. The learning rate for each method was swept on a power-of-2 grid on a single training seed, and the best value was used for the final comparison, shown over an average of 10 seeds. Error bars indicate 2 standard errors. L2 regularization was used for each task, and tuned for each problem and method separately also on a power-of-2 grid. We found that the optimal amount of regularization was not sensitive to the optimization method used. Results on held-out test data are shown in Figure~\ref{fig:deep_learning}; training loss plots can be found in the appendix in Figure~\ref{fig:deep_learning_train}.
 
Both \TAPS and \MOTAPS show favorable results compared to \SP on all three problems. On the computer vision datasets, neither method quite reaches the generalization performance of SGD with a highly tuned step-wise learning rate schedule (95.2\% for CIFAR10, 95.9\% on SVHN). On the IWSLT14 problem, both \TAPS and \MOTAPS out-perform Adam~\cite{kingma2014adam} which achieved a 
$2.69$ test loss and is the gold-standard for this task. 

\begin{figure}
\centering
\includegraphics[width=4.0cm]{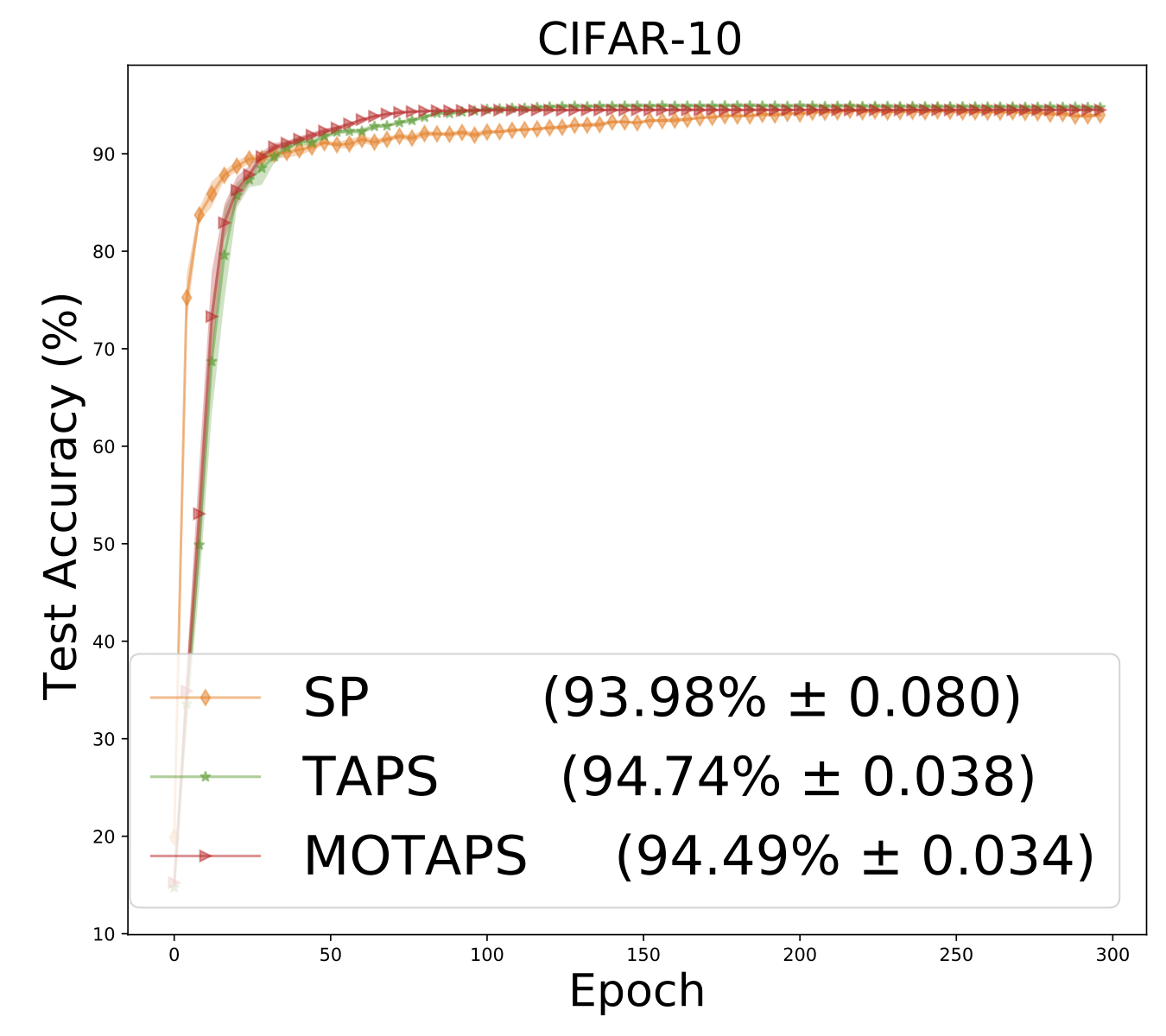}\hspace{0.5cm}
\includegraphics[width=4.0cm]{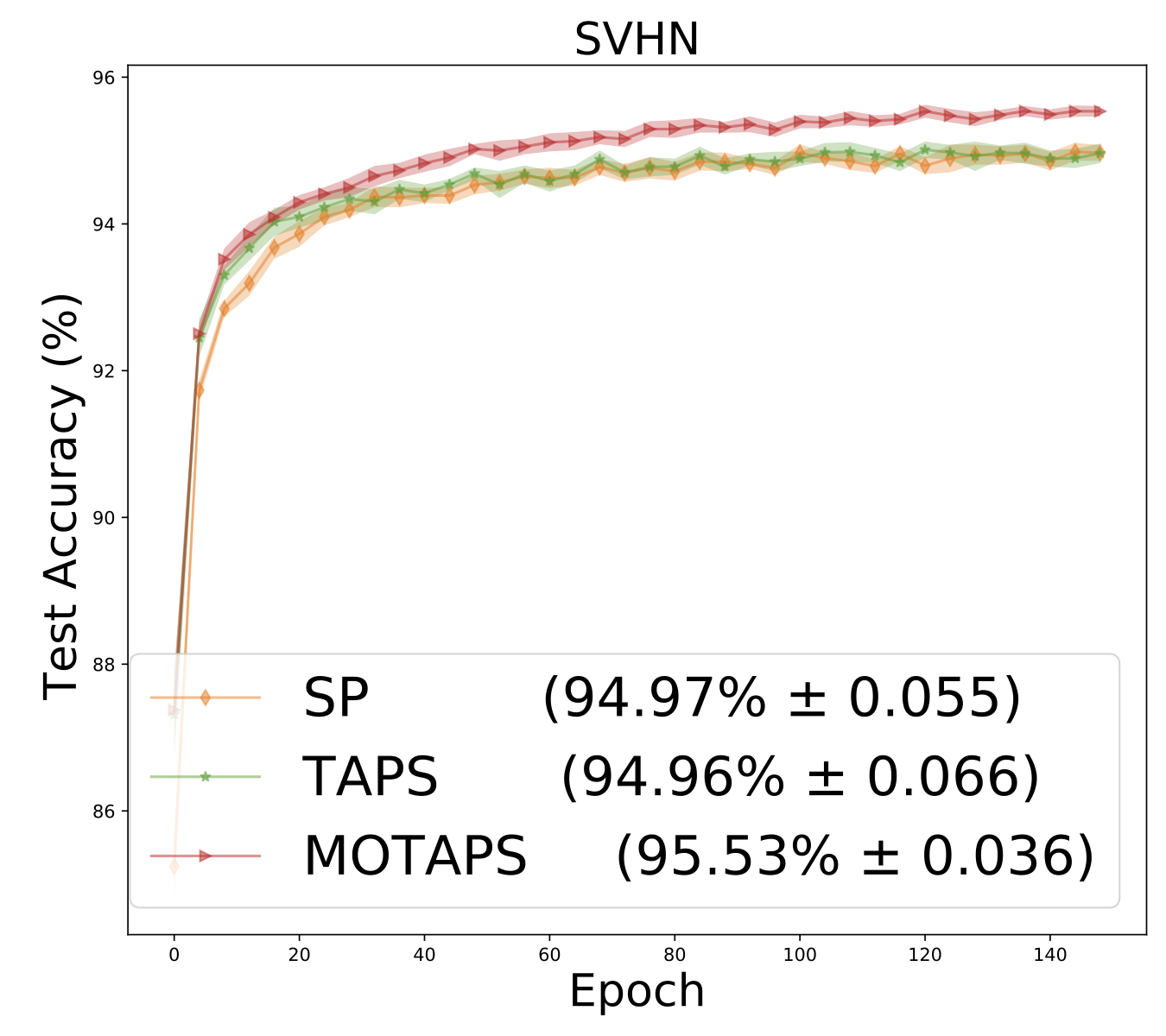}\hspace{0.5cm}
\includegraphics[width=4.0cm]{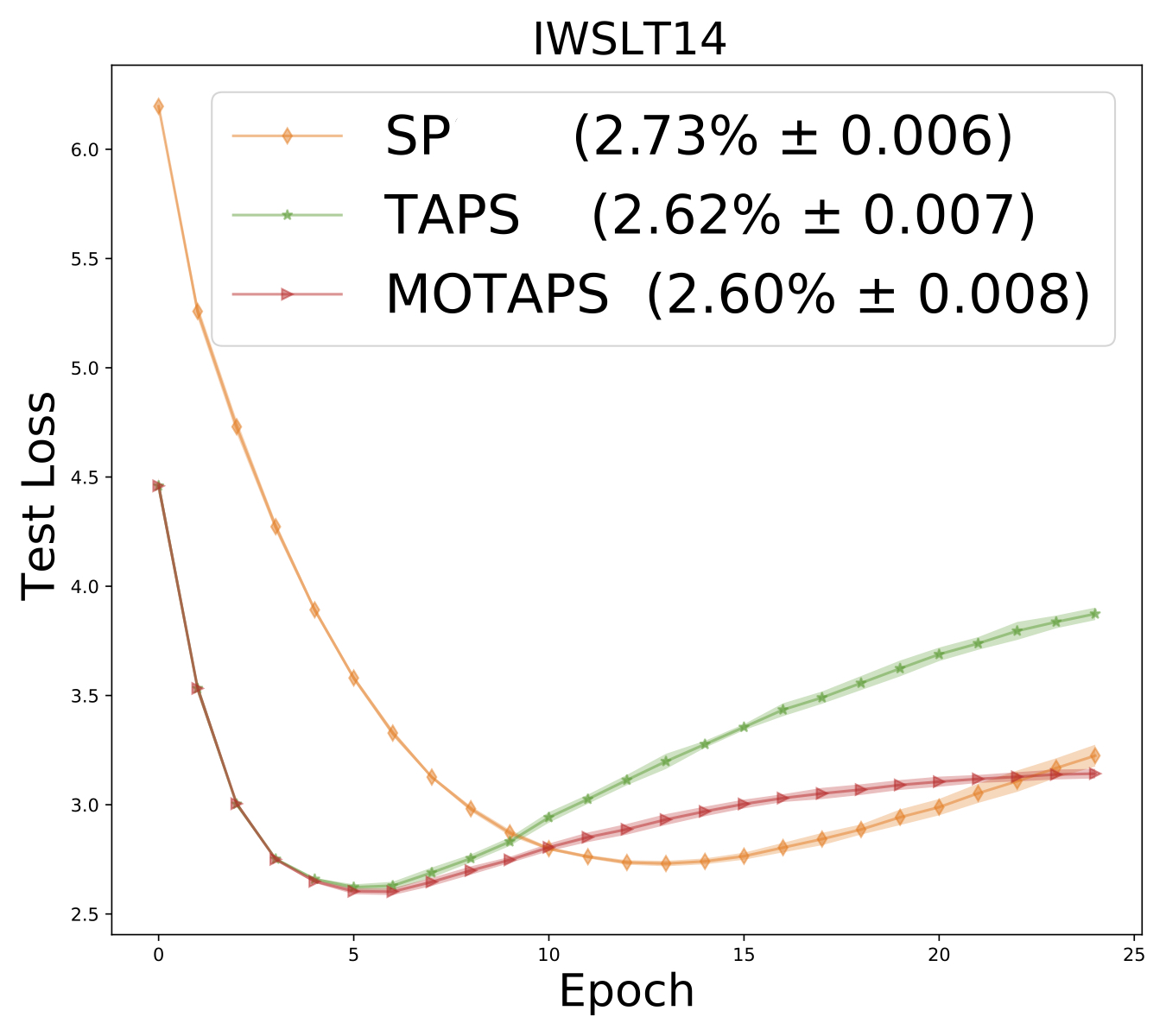}
\caption{Deep learning experiments}
\label{fig:deep_learning} 
\end{figure}

\subsection*{Acknowledgments}
Omitted for review.

\renewcommand*{\bibfont}{}
{ 
\printbibliography
}

\appendix 

\part*{Appendix}
The Appendix is organized as follows: In Section~\ref{sec:auxlemma}, we give some additional lemmas used to establish the closed form update of the methods. In Section~\ref{sec:missingproof} we present the proofs of the lemmas and  theorem  presented in the main paper. 
In Sections~\ref{sec:SPStheory},~\ref{sec:convTAPS} and~\ref{sec:convMOTAPS} we discuss the consequences of Theorem~\ref{theo:onlinesgdstar} to the \SP, \TAPS and \MOTAPS method, respectively.  
In Section~\ref{sec:SPSalternativegamma1theory} we provide another convergence theorem for \SP that shows that 
for smooth, star-convex loss functions the \SP method convergence with a step size $\gamma= 1.$
In Section~\ref{sec:theoryt-smooth} we present another general convergence theorem based on a time dependent smoothness assumption. 
Finally, in Section~\ref{asec:convex} and~\ref{asec:deep} we give further details on our implementations of the methods and the numerical experiments.

\tableofcontents

\section{Comparing \SP to the method given in~\cite{oberman2019stochastic,SPS}}
\label{sec:SPtoSPS}
The \SP method given in~\eqref{eq:SPS} is closely related to the method given in~\cite{oberman2019stochastic,SPS} which is
\begin{equation}\label{eq:SPSoptval}
w^{t+1} = w^t - \frac{f_i(w^t)-f_i^*}{\norm{\nabla f_i(w^t)}^2} \nabla f_i(w^t),
\end{equation}
where  $f_i^* \eqdef \inf_{w} f_i(w)$ for $i=1, \ldots, n.$ Note that the only difference between~\eqref{eq:SPSoptval} and the \SP method~\eqref{eq:SPS} is that $f_i(w^*)$ has been replaced by $f_i^*.$ If the Interpolation Assumption~\ref{ass:interpolate} holds then $f_i^* = f_i(w^*)$ and the two methods are equal. Outside of the interpolation regime, these two methods are not necessarily the same.

In terms of convergence theory, the difference is only cosmetic, since the method~\eqref{eq:SPSoptval}  only converges when $f_i^* = f_i(w^*)$, 
 that is, when the two methods are equal. Indeed, let 
\[\sigma \eqdef \frac{1}{n} \sum_{i=1}^n (f_i(w^*) - f_i^*).\]
Note that $\sigma \geq 0$ by the definition of $f_i^*.$
According to Theorems 3.1 and 3.4 in~\cite{SPS} the method~\eqref{eq:SPSoptval} converges to 
a neighborhood of the solution with a diameter that depends on $\sigma.$ Thus~\eqref{eq:SPSoptval} converges to the solution when $\sigma =0.$ This  only happens when the interpolation Assumption~\ref{ass:interpolate} holds.
Putting convergence aside, the  method~\eqref{eq:SPSoptval} has the advantage that for many machine learning  $f_i^*$ is known. 
   This is in contrast to the \SP method~\eqref{eq:SPS}, where $f_i(w^*)$ is not known for most applications. In our experiments, we set $f_i(w^*) =0$, and thus the two methods are equivalent.


\section{Auxiliary Lemmas}
\label{sec:auxlemma}
\begin{lemma} \label{lem:updateTasps}
The solution to
\begin{eqnarray}
w^{t+1}, \alpha_i^{t+1} & = & \underset{w\in\R^d, \alpha_i \in \R}{\argmin} \norm{w-w^t}^2+\norm{\alpha_i-\alpha_i^t}^2 \nonumber \\
&&\mbox{subject to }f_i(w^t) + \dotprod{\nabla f_i(w^t), w-w^t} =\alpha_i \label{eq:newtraphalphaapp}
\end{eqnarray}
is given by
\begin{eqnarray}
\alpha^{t+1}_i & =& \alpha^t_i +\frac{f_i(w^t) -\alpha^t_i}{\norm{\nabla f_i(w^t)}^2+1}  \nonumber \nonumber \\
w^{t+1} &= & w^t -\frac{f_i(w^t) -\alpha^t_i}{\norm{\nabla f_i(w^t)}^2+1} \nabla f_i(w^t).
 \label{eq:TaSPS1app}
\end{eqnarray}
\end{lemma}
\begin{proof}
Introducing the variable $x = [w,\, \alpha_i] \in \R^{d+1}$ we can re-write~\eqref{eq:newtraphalphaapp} as
\begin{eqnarray}
x^{t+1} & = &\argmin_x \norm{x-x^t}^2 \nonumber \\
&&\mbox{subject to }
 \begin{bmatrix}
\nabla f_i(w^t) \\ -1
\end{bmatrix}^\top
x
= -f_i(w^t)+\dotprod{\nabla f_i(w^t),w^t}.
 \label{eq:newtraphalphaapp}
\end{eqnarray}
Using Lemma~\ref{lem:proj1eq} (just below) we have that the solution to the above is given by
\begin{eqnarray*}
x^{t+1} & =& x^t + \begin{bmatrix}
\nabla f_i(w^t) \\ -1
\end{bmatrix}\frac{1}{\norm{\nabla f_i(w^t) }^2+1}(-f_i(w^t)+\dotprod{\nabla f_i(w^t),w^t}-  
(\dotprod{\nabla f_i(w^t),w^t} - \alpha_i^t)
) 
 \end{eqnarray*} 
Substituting out $x = [w,\, \alpha_i]$ and simplifying we have 
\begin{eqnarray*}
\begin{bmatrix}
w^{t+1} \\
\alpha_i^{t+1}
\end{bmatrix} & =& 
\begin{bmatrix}
w^{t} \\
\alpha_i^{t}
\end{bmatrix} 
+\begin{bmatrix}
\nabla f_i(w^t) \\ -1
\end{bmatrix}\frac{\alpha_i^t-f_i(w^t)}{\norm{\nabla f_i(w^t) }^2+1} ,
\end{eqnarray*} 
which is equal to~\eqref{eq:TaSPS1app}.
\end{proof}

\begin{lemma}\label{lem:proj1eq}
The solution to
\begin{eqnarray}
x^+ & = & \argmin_{x\in\R^d} \norm{x-x^0}^2 \nonumber \\
&&\mbox{subject to }a^\top x = b \label{eq:lemmapseudo}
\end{eqnarray}
is given by
\begin{equation}\label{eq:no9z8h4sdsd}
x^+ \; =\; x^0 + \frac{a}{\norm{a}^2}(b- a^\top x^0) 
\end{equation}
\end{lemma}
\begin{proof}
Substitute $z = x-x^0$ and consider the resulting problem
\begin{eqnarray}
z^+ & = &\argmin_{z\in\R^d} \norm{z}^2 \nonumber \\
&&\mbox{subject to }a^\top z = b -a^\top x^0\label{eq:lemmapseudoz}
\end{eqnarray}
One of the properties of the pseudo-inverse is that the least norm solution to the linear equation in~\eqref{eq:lemmapseudoz} is given by
\begin{equation}\label{eq:no9z8h4}
z^+  = a^{+\top}  (b -a^\top x^0),
\end{equation}
where $a^{+\top}$ is the pseudo-inverse of $a^\top$. It is now easy to show that $a^{+\top}  = \frac{a}{\norm{a}^2}$ is the pseudo-inverse \footnote{This follows by the definition of pseudo-inverse since $a^{+\top} a^{\top} a^{+\top} = a^{+\top}$, $a^\top a^{+\top} a^\top = a^\top$ and both $a^\top a^{\top +} $  and $ a^{\top +} a^\top$  are symmetric. } of $a$. Substituting back $x$ and the definition of $a^{+\top}$ in ~\eqref{eq:no9z8h4} gives~\eqref{eq:no9z8h4sdsd}.
\end{proof}

\subsection{Linear Algebra}

\begin{lemma}\label{lem:matrixblockbnd}
For any matrices $A,B, $ and $C$ of appropriate dimensions we have that
\begin{equation}
\norm{\begin{bmatrix}
A & C \\
C^\top & B
\end{bmatrix} } \; \leq \; \norm{A} + 2\norm{C} + \norm{D}
\end{equation}
\end{lemma}
\begin{proof}
Let $[v w]$ we a vector of unit norm. It follows that
\begin{eqnarray*}
\norm{\begin{bmatrix}
A & C \\
C^\top & B
\end{bmatrix}
\begin{bmatrix}
v \\w
\end{bmatrix}
 }
 & =&\sqrt{ \norm{Av + Cw}^2 + \norm{C^\top v+ Bw}^2} \\
 & \leq & \norm{Av + Cw} +  \norm{C^\top v+ Bw} \\
 & \leq & \norm{Av} + \norm{Cw}+ \norm{C^\top v} + \norm{Bw} \\
 & \leq &  \norm{A}\norm{v} + \norm{C}\norm{w}+ \norm{C}\norm{ v} + \norm{B}\norm{w} \\
 & \leq &  \norm{A} + 2\norm{C} + \norm{D},
\end{eqnarray*}
where in the first inequality we used that, for any $a,b>0$ we have that $\sqrt{a+b} \leq \sqrt{a} + \sqrt{b}$, and in the last inequality we used that $\norm{w},\norm{v} \; \leq  \;\norm{[w \; v]} = 1.$
\end{proof}

\section{Missing Proofs}
\label{sec:missingproof}

Here we present the missing proofs from the main text.

\subsection{Proof of Lemma~\ref{lem:reformsgdtasps} }

First note that for the function in~\eqref{eq:fi} we have that
\begin{align}
&\nabla_w h_{i,t}(w, \alpha)    \;=\; \frac{f_i(w) -\alpha_i}{\norm{\nabla f_i(w^t)}^2+1} \nabla f_i(w),\quad \quad 
\nabla_{\alpha_i} h_{i,t}(w, \alpha)   \; =\;  - \frac{f_i(w) -\alpha_i}{\norm{\nabla f_i(w^t)}^2+1}, \label{eq:gradsTAPS1}\\
&\mbox{and} \quad \nabla_{\alpha_i}h_{n+1,t}(w,\alpha)   \; =\;  (\overline{\alpha} - \tau) .\label{eq:gradsTAPS}
\end{align}
\begin{proof}
The stationarity conditions of~\eqref{eq:proxyobjalpha} are given by setting the gradients to zero, which from~\eqref{eq:gradsTAPS} we have that
\begin{eqnarray}
\nabla_w h_t(w,\alpha) & =&  0  \nonumber\\
\nabla_{\alpha_i} h_t(w,\alpha) & =&  0  , \quad \mbox{for }i=1,\ldots, n \nonumber\\
& \Updownarrow & \nonumber \\
\frac{1}{n+1} \sum_{i=1}^n \frac{f_i(w) -\alpha_i}{\norm{\nabla f_i(w^t)}^2+1} \nabla f_i(w) & =& 0\label{eq:tempoi8j48sj4} \\
 \frac{f_i(w) -\alpha_i}{\norm{\nabla f_i(w^t)}^2+1} &= &  (\overline{\alpha} - \tau) , \quad \mbox{for }i=1,\ldots, n.\label{eq:tempoi8j48sj423}
\end{eqnarray}
 If $\overline{\alpha} = \tau$ then from~\eqref{eq:tempoi8j48sj423}  we have that $f_i(w) =\alpha_i$ for all $i$, and thus from Assumption~\ref{ass:target} we have that $w$ must be a minimizer of~\eqref{eq:main}, and thus a stationary point.

  On the other hand, if $\overline{\alpha} \neq \tau,$ then by substituting~\eqref{eq:tempoi8j48sj423} into~\eqref{eq:tempoi8j48sj4} gives
\begin{equation}
 \frac{1}{n+1} \sum_{i=1}^n (\overline{\alpha} - \tau)\nabla f_i(w)  = \frac{n(\overline{\alpha} - \tau)}{n+1} \left(\frac{1}{n} \sum_{i=1}^n \nabla f_i(w) \right) =0.
\end{equation}
Consequently since $\overline{\alpha} \neq \tau,$ we have  $\frac{1}{n} \sum_{i=1}^n \nabla f_i(w)  =0$ and thus $w$ is a stationary point of~\eqref{eq:main}.

Finally,  if $(w^*,\alpha^*)$ is a minimizer of~\eqref{eq:proxyobjalpha} then by Assumption~\ref{ass:target}  necessarily $h_t(w^*,\alpha^*) =0.$ Thus  $f_i(w^*) =\alpha_i^*$ and  $\overline{\alpha}^* = \tau.$ Thus again by Assumption~\ref{ass:target} we have that  $w^*$ must be a minimizer of~\eqref{eq:main}.
\end{proof}

\subsection{Proof of Lemma~\ref{lem:TAPSweakgrowth}}

\begin{proof}
First note that
\begin{eqnarray}
\norm{\nabla_w h_{i,t}(w^t,\alpha) }^2
& \overset{\eqref{eq:gradsTAPS}}{=}&
\left(\frac{f_{i}(w^t)-\alpha_i}{\norm{\nabla f_{i}(w^t)}^2 +1 }\right)^2\norm{\nabla f_{i}(w^t)}^2 . \label{eq:so8jos9j94}
\end{eqnarray}
Furthermore
\begin{eqnarray}
\norm{\nabla_{\alpha_i} h_{i,t}(w^t,\alpha) }^2
 & \overset{\eqref{eq:gradsTAPS}}{=}&
\left(\frac{f_{i}(w^t)-\alpha_i}{\norm{\nabla f_{i}(w^t)}^2 +1 }\right)^2. \label{eq:sjo40s9j4s4}
\end{eqnarray}
Consequently adding~\eqref{eq:so8jos9j94} and~\eqref{eq:sjo40s9j4s4} gives 
\begin{eqnarray}
\norm{\nabla h_{i,t}(w^t,\alpha) }^2 & =& \norm{\nabla_w h_{i,t}(w^t,\alpha) }^2 +\norm{\nabla_{\alpha_i} h_{i,t}(w^t,\alpha) }^2 \nonumber\\ 
 & \overset{\eqref{eq:so8jos9j94}+\eqref{eq:sjo40s9j4s4}}{=}&
\left(\frac{f_{i}(w^t)-\alpha_i}{\norm{\nabla f_{i}(w^t)}^2 +1 }\right)^2\left(\norm{\nabla f_{i}(w^t)}^2 +1\right)\nonumber \\
& =& \frac{(f_{i}(w^t)-\alpha_i)^2}{\norm{\nabla f_{i}(w^t)}^2 +1 } \; \overset{\eqref{eq:fi}}{=} 2h_{i,t}(w^t,\alpha). \label{eq:tempalooo8õao}
\end{eqnarray}
Furthermore
\begin{eqnarray}
\norm{\nabla h_{n+1,t}(w,\alpha) }^2
& \overset{\eqref{eq:gradsTAPS}}{=}&
\sum_{i=1}^n\left(\overline{\alpha} - \tau\right)^2 \; \overset{\eqref{eq:fi}}{=} \; 2h_{n+1,t}(w,\alpha).  \label{eq:tempanoena9a}
\end{eqnarray}
Consequently
\begin{eqnarray}
\frac{1}{n+1}\sum_{i=1}^{n+1}\norm{\nabla h_{i,t}(w^t,\alpha) }^2
& \overset{\eqref{eq:tempanoena9a}+\eqref{eq:tempalooo8õao} }{=}&  \frac{1}{n+1}\sum_{i=1}^{n+1} 2h_{i,t}(w^t,\alpha)  \; =
\; 2h_{t}(w^t,\alpha).  \nonumber
\end{eqnarray}
%
\end{proof}

\subsection{Proof of Lemma~\ref{lem:movtargequiv}}
\begin{lemma} \label{lem:movtargequivapp}
Let
\begin{equation} \label{eq:alphastarandtaustar}
\alpha_i^* \eqdef f_i(w^*) \quad \mbox{and} \quad \tau^* = f(w^*), \quad \mbox{for }i=1,\ldots, n.
\end{equation}
It follows that
\begin{align}
h_t(w^*,\alpha^*,\tau^*) 
& =\; \frac{\lambda f(w^*)^2}{2(n+1)}. \label{eq:htstarapp}
\end{align}
Furthermore, 
every stationary point of~\eqref{eq:proxyobjalphatau} is a stationary point of~\eqref{eq:main}. Finally if $f(w) \geq 0$ and $(w^*,\hat{\alpha},\hat{\tau})$ is a minima  of~\eqref{eq:proxyobjalphatau} then $w^*$ is a minima of ~\eqref{eq:main}.
\end{lemma}
\begin{proof}
Substituting~\eqref{eq:alphastarandtaustar} into~\eqref{eq:proxyobjalphatau} gives
\[h_t(w^*,\alpha^*,\tau^*) \; \eqdef\;  \; \frac{1}{n+1}\frac{\lambda}{2} (\tau^*)^2 \;= \; \frac{\lambda f(w^*)^2}{2(n+1)}  .\]

Each stationary point of~\eqref{eq:proxyobjalphatau} satisfies
\begin{align}
\nabla_w h_t(w,\alpha,\tau)  & \; =\; \frac{1-\lambda}{n+1}  \sum_{i=1}^n\frac{f_i(w) -\alpha_i}{\norm{\nabla f_i(w^t)}^2+1} \nabla f_i(w) =0,\label{eq:son9ojo9j9oxx} \\
\nabla_{\alpha_i} h_t(w,\alpha,\tau)  & \; =\;  \frac{1-\lambda}{n+1} \frac{\alpha_i-f_i(w) }{\norm{\nabla f_i(w^t)}^2+1} +\frac{1-\lambda}{n+1}(\overline{\alpha} - \tau)=0,\label{eq:so9ms9smsss} \\
\nabla_{\tau} h_t(w,\alpha,\tau) & \; =\;(1-\lambda) n(\tau - \overline{\alpha}) +\lambda \tau=0. \label{eq:dertaulast}
\end{align}
 From the last equation we have that 
 \begin{equation}\label{eq:tauopt} 
   \overline{\alpha}-\tau = \frac{\lambda}{(1-\lambda)n} \tau,
\end{equation}
 and consequently substituting out $\overline{\alpha}-\tau$  in~\eqref{eq:so9ms9smsss} by using~\eqref{eq:tauopt}  gives
\begin{align}
\nabla_{\alpha_i} h_t(w,\alpha,\tau)  & \; =\;  \frac{1-\lambda}{n+1} \frac{\alpha_i-f_i(w) }{\norm{\nabla f_i(w^t)}^2+1} +\frac{1}{n+1}\frac{\lambda}{n} \tau=0. \label{eq:isolatetaueq}
\end{align}
Passing the $\tau$ term to the other side gives
\begin{equation}\label{eq:alphabetaopt}
  \frac{\lambda}{n} \tau=(1-\lambda) \frac{f_i(w) -\alpha_i }{\norm{\nabla f_i(w^t)}^2+1}, \quad \mbox{for }i=1, \ldots, n. 
\end{equation}
This allows us to substitute in~\eqref{eq:son9ojo9j9oxx} giving
\begin{eqnarray}
\nabla_wh_t(w,\alpha,\tau)  &  = &\frac{\lambda \tau }{n+1} \left(\frac{1}{n}  \sum_{i=1}^n \nabla f_i(w)\right) =0.
\end{eqnarray}

From this we can conclude that if $(w, \alpha, \tau)$ is a stationary point of~\eqref{eq:proxyobjalphatau}, then $w$ is a stationary point of our original objective function.
Let $(w, \alpha, \tau)$ be a stationary point. It follows from~\eqref{eq:tauopt} that $\tau  =\frac{(1-\lambda)n}{(1-\lambda)n+\lambda} \overline{\alpha}$,  and thus after substituting into~\eqref{eq:proxyobjalphatau} gives
\[h_t(w,\alpha,\tau) \; \eqdef\;  \; \frac{1}{n+1} \left(\sum_{i=1}^n \frac{1-\lambda}{2}\frac{(f_i(w) -\alpha_i)^2}{\norm{\nabla f_i(w^t)}^2+1}  + \frac{n(1-\lambda)}{2}(\overline{\alpha} - \tau)^2+\frac{\lambda}{2} \tau^2 \right).\]

\begin{align}
  h_t(w,\alpha,\tau) &=   \; \frac{1}{n+1} \left(\sum_{i=1}^n \frac{1-\lambda}{2}\frac{(f_i(w) -\alpha_i)^2}{\norm{\nabla f_i(w^t)}^2+1}  + \frac{n(1-\lambda)}{2}\left( \frac{\lambda}{n(1-\lambda)+\lambda} \overline{\alpha} \right)^2+\frac{\lambda}{2} \frac{(1-\lambda)^2n^2}{(n(1-\lambda)+\lambda)^2} \overline{\alpha}^2 \right)\nonumber \\
  &=  \; \frac{1-\lambda}{n+1} \left(\sum_{i=1}^n \frac{1}{2}\frac{(f_i(w) -\alpha_i)^2}{\norm{\nabla f_i(w^t)}^2+1}  + \frac{1}{2}\frac{n\lambda}{n(1-\lambda)+\lambda}\overline{\alpha}^2 \right) \label{eq:htstationtemp}
\end{align}

Furthermore, $\tau  =\frac{(1-\lambda)n}{(1-\lambda)n+\lambda} \overline{\alpha}$ substituting into~\eqref{eq:isolatetaueq} and multiplying the result by $(n+1)$ gives
\begin{align*}
 \frac{\alpha_i-f_i(w) }{\norm{\nabla f_i(w^t)}^2+1} +\frac{\lambda}{n(1-\lambda)+\lambda} \overline{\alpha}=0, \quad \mbox{for }i=1,\ldots, n.
\end{align*}

This  can be re-arranged and  written more compactly as the linear system
\begin{equation}
\left(\mD^{-1}  +\lambda \frac{\ones \ones^\top}{n(n(1-\lambda)+\lambda)} \right)\alpha = \mD^{-1}F,
\end{equation}
where \begin{align*}
\mD  &\eqdef \diag{\norm{\nabla f_1(w^t)}^2+1, \ldots, \norm{\nabla f_n(w^t)}^2+1} \quad \mbox{and}\quad \\
F &= \left(f_1(w), \ldots, f_n(w)\right).
\end{align*} 
Using the Woodbury identity, the solution to the above is given by
\begin{align}
\alpha & = \left(\mD^{-1}  + \lambda\frac{\ones \ones^\top}{n(n(1-\lambda)+\lambda)} \right)^{-1} \mD^{-1}F, \\
& = \left(\mI -\mD \ones  \left(\frac{n(n(1-\lambda)+\lambda)}{\lambda} +  \ones^\top\mD\ones\right)^{-1}  \ones^{\top} \right)F\\
&=  \left(\mI - \lambda \frac{\mD \ones \ones^{\top}  }{n(n(1-\lambda)+2\lambda) +  \lambda\sum_{i=1}^n \norm{\nabla f_i(w^t)}^2} \right)F.
\end{align}
Which reading line by line gives
\begin{align}
\alpha _i & = f_i(w)-\lambda  \frac{\mD e_i \sum_{j=1} f_j(w)}{n(n(1-\lambda)+2\lambda ) +   \lambda\sum_{j=1}^n \norm{\nabla f_j(w^t)}^2} \nonumber \\
&= f_i(w)-\lambda   \frac{\left(  \norm{\nabla f_i(w^t)}^2 + 1\right) \sum_{j=1} f_j(w)}{n(n(1-\lambda)+2\lambda ) +   \lambda\sum_{j=1}^n \norm{\nabla f_j(w^t)}^2} .
\label{eq:alphaqefistationarymotaps}
\end{align}

Taking the average over $i$ in the above gives
\begin{align}
\overline{\alpha}
&=  f(w)-  \lambda  f(w) \frac{n+\sum_{j=1}^n \norm{\nabla f_j(w^t)}^2 }{n(n(1-\lambda)+2\lambda) + \lambda \sum_{j=1}^n \norm{\nabla f_j(w^t)}^2}  \nonumber \\
&=  f(w)\left( 1-  \lambda\frac{n+\sum_{j=1}^n \norm{\nabla f_j(w^t)}^2 }{n(n(1-\lambda)+2\lambda) + \lambda \sum_{j=1}^n \norm{\nabla f_j(w^t)}^2} \right)  \nonumber \\
&=   f(w) \frac{n(n(1-\lambda)+\lambda) }{n(n(1-\lambda)+2\lambda) + \lambda \sum_{j=1}^n \norm{\nabla f_j(w^t)}^2}  \label{eq:tesdfspmsop4so}
\end{align}

Substituting~\eqref{eq:alphaqefistationarymotaps} and~\eqref{eq:tesdfspmsop4so} into~\eqref{eq:htstationtemp} gives
\begin{align*}
  h_t(w,\alpha,\tau) \frac{n+1}{1-\lambda}
  &=  \; \sum_{i=1}^n \frac{\lambda^2}{2}\frac{ \left( \frac{\left(  \norm{\nabla f_i(w^t)}^2 + 1\right) \sum_{j=1} f_j(w)}{n(n(1-\lambda)+2\lambda ) + \lambda \sum_{j=1}^n \norm{\nabla f_j(w^t)}^2} \right)^2}{\norm{\nabla f_i(w^t)}^2+1}  + \frac{1}{2}\frac{n\lambda}{n(1-\lambda)+\lambda}\overline{\alpha}^2   \\
  & =  \; \sum_{i=1}^n \frac{\lambda^2n^2}{2} f(w)^2\frac{\norm{\nabla f_i(w^t)}^2 + 1}{\left(n(n(1-\lambda)+2\lambda) +  \lambda\sum_{j=1}^n \norm{\nabla f_j(w^t)}^2\right)^2}   + \frac{1}{2}\frac{n\lambda}{n(1-\lambda)+\lambda}\overline{\alpha}^2  \\
    & =  \; \sum_{i=1}^n \frac{\lambda^2n^2}{2}f(w)^2\frac{\norm{\nabla f_i(w^t)}^2 + 1}{\left(n(n(1-\lambda)+2\lambda) +  \lambda\sum_{j=1}^n \norm{\nabla f_j(w^t)}^2\right)^2}  \\
    & \qquad  + \frac{1}{2}\frac{n\lambda}{n(1-\lambda)+\lambda}\left( f(w) \frac{n(n(1-\lambda)+\lambda) }{n(n(1-\lambda)+2\lambda) + \lambda \sum_{j=1}^n \norm{\nabla f_j(w^t)}^2} \right)^2  \\
  & = \frac{\lambda}{2} f(w)^2\frac{ n^2 }{n(n(1-\lambda)+2\lambda)+ \lambda \sum_{j=1}^n \norm{\nabla f_j(w^t)}^2}  ,
\end{align*}
where in  first equality we used~\eqref{eq:alphaqefistationarymotaps} and in the third equality we used~\eqref{eq:tesdfspmsop4so}.
Since $w^t$ is fixed, and every minima of~\eqref{eq:proxyobjalphatau} is a stationary point, we have that the minima in $w$ of the above is given by 
\[w^{*} \in \arg\min f(w)^2 = \arg\min f(w) , \]
where we used the positivity of $f(w).$

%

\end{proof}

\subsection{Proof of Lemma~\ref{lem:MOSTAPSweakgrowthn1} }

Here we prove an extended version of Lemma~\ref{lem:MOSTAPSweakgrowthn1} with some additional intermediary results that make the lemma easier to follow.
\begin{lemma}\label{lem:MOSTAPSweakgrowthn1app}
Consider the functions  
\begin{equation}
h_{i,t}(w,\alpha, \tau) \eqdef \frac{1-\lambda}{2}\frac{(f_i(w) -\alpha_i)^2}{\norm{\nabla f_i(w^t)}^2+1},\quad \mbox{for }i=1,\ldots, n,
\end{equation}
and $h_{n+1,t}(w,\alpha,\tau)$ given in~\eqref{eq:fn1tau}.
It follows that $h_t(w,\alpha,\tau)$ defined in~\eqref{eq:proxyobjalphatau} is equivalent to
\begin{equation}\label{eq:hmotapsdecomp}
 h_{t}(w,\alpha,\tau)  =\frac{1}{n+1} \sum_{i=1}^n  h_{i,t}(w,\alpha,\tau)
\end{equation}
Furthermore,  if 
\begin{equation} \label{eq:lambdarestapp}
\lambda \leq \frac{2n+1}{2n+3} < 1
\end{equation}
 then
\begin{align}
\norm{\nabla h_{i,t}(w^t,\alpha,\tau)}^2 & = 2 h_{i,t}(w^t,\alpha,\tau), \label{eq:zdimziizdap} \\\
\norm{ \nabla h_{n+1,t}(w,\alpha,\tau)}^2 &\leq 2(1-\lambda)(2n+1)h_{n+1,t}(w,\alpha,\tau),  \label{eq:zdimziizd2ap} 
\end{align}
and consequently 
\begin{equation} \label{eq:MOSTAPSweakgrowthn1ap}
\frac{1}{n+1} \sum_{i=1}^{n+1}\norm{\nabla h_{i,t}(w^t,\alpha,\tau)}^2 \; \leq\;  2(1-\lambda)(2n+1) h_{t}(w^t,\alpha,\tau).
\end{equation}
\end{lemma}
\begin{proof}
Using the definitions of $h_t(w^t,\alpha,\tau)$ in~\eqref{eq:proxyobjalphatau} we have that~\eqref{eq:hmotapsdecomp} holds.

Furthermore~\eqref{eq:zdimziizdap} follows from Lemma~\ref{lem:TAPSweakgrowth}. As for  $h_{n+1,t}(w,\alpha,\tau)$ in~\eqref{eq:fn1tau} we have that
\begin{eqnarray}
h_{n+1,t}(w,\alpha,\tau) & = & \frac{n(1-\lambda)}{2}(\overline{\alpha} - \tau)^2 + \frac{\lambda}{2} \tau^2 \nonumber \\
  \nabla _{\tau}h_{n+1,t}(w,\alpha,\tau) & = & (1-\lambda)n( \tau-\overline{\alpha} ) + \lambda \tau. \nonumber \\
    \nabla _{\alpha}h_{n+1,t}(w,\alpha,\tau) & = & (1-\lambda){\bf 1}(\overline{\alpha} - \tau)  \label{eq:fn1tauapp} 
\end{eqnarray} 
Consequently
\begin{eqnarray*}
 \norm{ \nabla h_{n+1,t}(w,\alpha,\tau)}^2 &= & ((1-\lambda)n(\tau - \overline{\alpha} ) + \lambda \tau)^2 + (1-\lambda)^2\norm{\bf 1}^2 (\overline{\alpha} - \tau)^2 \\
 &\leq &  2(1-\lambda)^2n^2(\tau - \overline{\alpha} ) ^2 + 2 \lambda^2\tau^2+ (1-\lambda)^2n(\overline{\alpha} - \tau)^2\\
 & =& 2(1-\lambda)(2n+1)\frac{(1-\lambda)n(\tau - \overline{\alpha} ) ^2}{2}+4\lambda \frac{\lambda \tau^2}{2}\\
 & \leq &  2\max\{(1-\lambda)(2n+1), \;2\lambda  \}h_{n+1,t}(w,\alpha,\tau) . 
\end{eqnarray*}
Due to~\eqref{eq:lambdarestapp} we have that 
\[\max\{(1-\lambda)(2n+1), \;2\lambda  \} \; =\; (1-\lambda)(2n+1) . \]
This proves~\eqref{eq:zdimziizdap}. 
As a consequence
from~\eqref{eq:zdimziizdap} and~\eqref{eq:zdimziizd2ap} we have that
\begin{align*}
\frac{1}{n+1} \sum_{i=1}^{n+1}\norm{\nabla h_{i,t}(w^t,\alpha,\tau)}^2 
& \leq \frac{2\max\left\{1,(1-\lambda)(2n+1) \right\}}{n+1}  \sum_{i=1}^{n+1} h_{i,t}(w^t,\alpha,\tau)   & \mbox{(Using~\eqref{eq:zdimziizdap} and~\eqref{eq:zdimziizd2ap})}  \\
&\leq\;  2(1-\lambda)(2n+1) h_{t}(w^t,\alpha,\tau).  & \mbox{(Using~\eqref{eq:lambdarestapp} and~\eqref{eq:proxyobjalphatau}) }
\end{align*}

%
%
%
%

\end{proof}

\subsection{Proof of Theorem~\ref{theo:onlinesgdstar} }

Here we give the proof of Theorem~\ref{theo:onlinesgdstar}. We prove a slightly more general version of Theorem~\ref{theo:onlinesgdstar} by not requiring that the auxiliary function is zero at the optimal. That is  $h_t(z^*)$ may be non-zero. The exact result in Theorem~\ref{theo:onlinesgdstar} follows from applying the following Theorem~\ref{theo:onlinesgdstarapp} with $h_t(z^*)=0.$
\begin{theorem} [Star-convexity]\label{theo:onlinesgdstarapp} 
Suppose Assumption~\ref{lem:growthgen} holds with $G>0.$
  Let $\gamma < 1/G$ 
and  suppose there exists  $z^* $ such that
 $h_t$  is star-convex at $z^t$ and around $z^*$, that is
\begin{equation}\label{eq:hstar}
h_t(z^*) \; \geq \; h_t(z^t) + \dotprod{\nabla h_t(z^t), z^* -z^t}, 
\end{equation}
 then we have that
\begin{align}
\min_{t=1,\ldots, k} \E{h_t(z^t) -h_t(z^*)}  & \leq \;\frac{1}{k} \sum_{t=0}^k \E{ h_t(z^t)-h_t(z^*)} \nonumber \\
 & \leq \; \frac{1}{k} \frac{1}{2\gamma(1-G\gamma) }\E{\norm{z^{0} -z^*}^2} + \frac{G\gamma}{1-G\gamma}\frac{1}{k} \sum_{t=1}^k h_t(z^*).\label{eq:hconv}
\end{align}

\end{theorem}

\begin{proof}[]

This proof is partially based on  Theorems 4.3 ~\cite{SNR}.
Let $\EE{t}{\cdot} \eqdef \E{\cdot \; | \; z^t}$ denote the expectation conditioned on $z^t.$
   
Expanding the squares we have
\begin{eqnarray}
\EE{t}{\norm{z^{t+1} -z^*}^2} & \leq & \norm{z^{t} -z^*}^2 -2\gamma \dotprod{\nabla_w h_t(z^t), z^{t} -z^*} + \gamma^2 \EE{t}{\norm{\nabla_w h_{t,i_t}(z^t)}^2}\nonumber\\
& \overset{\eqref{eq:hweakgrow}}{ \leq }& \norm{z^{t} -z^*}^2 -2\gamma \dotprod{\nabla_w h_t(z^t), z^{t} -z^*} +2G \gamma^2 h_t(z^t)  \nonumber \\
& \overset{\eqref{eq:hstar}}{\leq}& 
 \norm{z^{t} -z^*}^2 -2\gamma (h_t(z^t)-h_t(z^*)) +2 G\gamma^2 h_t(z^t) \nonumber\\
 & =&  \norm{z^{t} -z^*}^2 -2\gamma(1-G \gamma)  (h_t(z^t)-h_t(z^*))+2 G\gamma^2 h_t(z^*)
\end{eqnarray}

Taking expectation, re-arranging and summing both sides from $t=0 , \ldots, k$ we have  that
\begin{eqnarray}
 \sum_{t=0}^k \E{ h_t(z^t) -h_t(z^*)} & \leq & \frac{1}{2\gamma(1-G \gamma) }\sum_{t=0}^k\left(\E{\norm{z^{t} -z^*}^2}  -\E{\norm{z^{t+1} -z^*}^2} \right) +   \frac{G\gamma}{1-G\gamma}  \sum_{t=0}^k h_t(z^*) \nonumber \\
 & \leq & \frac{1}{2\gamma(1-G\gamma) }\E{\norm{z^{0} -z^*}^2} +\frac{G\gamma}{1-G\gamma}  \sum_{t=0}^k h_t(z^*)   .
\end{eqnarray}
Now dividing through by $k$ gives~\eqref{eq:hconv}.
\end{proof}

\subsection{Proof of Theorem~\ref{theo:onlinesgdstrongstar} }

\begin{theorem}
Suppose Assumption~\ref{lem:growthgen} holds with $G>0.$
  Let $\gamma \leq 1/G$.
If there exists  $\mu>0$ and $z^*$ such that $h_t$ is $\mu$--strongly star--convex along $z^t$ and around $z^*$, that is
\begin{equation}\label{eq:hquasistrong}
h_t(z^*) \; \geq \; h_t(z^t) + \dotprod{\nabla h_t(z^t), z^* -z^t} + \frac{\mu}{2}\norm{z^*-z^t}, 
\end{equation}
then 
\begin{eqnarray}\label{eq:hconvsstrong}
\E{\norm{z^{t+1} -z^*}^2} & \leq &   (1-\gamma \mu)^{t+1} \norm{z^{0} -z^*}^2 +2 G\gamma^2 \sum_{i=0}^{t}(1-\gamma \mu)^i \E{ h_i(z^*)}.
\end{eqnarray}
Finally, if $h_t(z^*) =0$ for all $t$ then we have that~\eqref{eq:hconvsstrong} and~\eqref{eq:hweakgrow} together imply that  $\mu \leq G$ and thus~\eqref{eq:hconvsstrong} gives linear convergence.

\end{theorem}
\begin{proof}
This proof is partially based on  4.10 in~\cite{SNR}, which in turn is based on Theorem 6 in~\cite{vaswani2018fast}, Thereom 4.1 in~\cite{SGDstruct} and Theorem 3.1 in~\cite{gower2019sgd}.

Expanding the squares  we have that
\begin{eqnarray}
\EE{t}{\norm{z^{t+1} -z^*}^2} & \leq & \norm{z^{t} -z^*}^2 -2\gamma \dotprod{\nabla_w h_t(z^t), z^{t} -z^*} + \gamma^2 \EE{t}{\norm{\nabla_w h_{t,i_t}(z^t)}^2}\nonumber\\
& \overset{\eqref{eq:hweakgrow}}{ \leq }& \norm{z^{t} -z^*}^2 -2\gamma \dotprod{\nabla_w h_t(z^t), z^{t} -z^*} +2G \gamma^2 h_t(z^t)  \nonumber \\
& \overset{\eqref{eq:hquasistrong}}{\leq}& 
(1-\gamma \mu) \norm{z^{t} -z^*}^2 -2\underbrace{\gamma(1-G\gamma) ( h_t(z^t) -h_t(z^*))}_{\geq 0}  +2 G\gamma^2 h_t(z^*)\nonumber \\
& \leq & (1-\gamma \mu) \norm{z^{t} -z^*}^2 +2 G\gamma^2 h_t(z^*), \label{eq:theolinearonestep}
\end{eqnarray}
where to get to the last line we used that $(1-G \gamma) ( h_t(z^t) -h_t(z^*)) \geq 0$ which holds because $\gamma \leq \frac{1}{G}$. Taking the expectation and applying the above recursively gives
\begin{eqnarray}
\EE{t}{\norm{z^{t+1} -z^*}^2} 
& \leq & (1-\gamma \mu)^{t+1} \norm{z^{0} -z^*}^2 +2 G\gamma^2 \sum_{i=0}^{t}(1-\gamma \mu)^i  h_i(z^*)
\end{eqnarray}

which is the result~\eqref{eq:hconvsstrong}.

Furthermore,  if $h_t(z^*) =0$ we have that $\mu \leq G$ follows from a small modification of Theorem 4.10 in~\cite{SNR}. Indeed taking expectation over~\eqref{eq:hquasistrong} and using~\eqref{eq:hweakgrow} we have that 
\begin{eqnarray}
h_t(z^*) & \geq &  \frac{1}{2G}\E{\norm{\nabla h_{t,i_t}(z^t)}^2 }+ \dotprod{\nabla h_t(z^t), z^* -z^t} + \frac{\mu}{2}\norm{z^*-z^t} \nonumber \\
& =&  \frac{G}{2}\E{\norm{z^*-z^t - \frac{1}{L}\nabla h_{t,i_t}(z^t)}^2 }- \frac{G-\mu}{2}\norm{z^*-z^t} .
\end{eqnarray}
Rearranging and using that $h_t(z^*) =0$ gives
\[\frac{G-\mu}{2}\norm{z^*-z^t}  \geq  \frac{L}{2}\E{\norm{z^*-z^t - \frac{1}{G}\nabla h_{t,i_t}(z^t)}^2 } \geq 0.  \]
Thus $\mu \leq G.$
\end{proof}

\section{Convergence of The Stochastic Polyak Method}
\label{sec:SPStheory} 

Here we explore sufficient conditions  for the assumptions in Theorems~\ref{theo:onlinesgdstar} and~\ref{theo:onlinesgdstrongstar} to hold for the  \SP method~\eqref{eq:SPS}. To this end,  let 
 \begin{eqnarray}\label{eq:proxyfun}
   h_t(w) &\eqdef&   \frac{1}{n}\sum_{i=1}^n \frac{1}{2}\frac{(f_i(w) -f_i(w^*))^2}{\norm{\nabla f_i(w^t)}^2}, \\
 h_{i,t}(w) &\eqdef&  \frac{1}{2}\frac{(f_i(w) -f_i(w^*))^2}{\norm{\nabla f_i(w^t)}^2} .   \label{eq:proxyfuni}
\end{eqnarray}
We will also explore the consequences of these theorems.
In these section we say that $f_i$ is $L_i$ smooth if
\begin{equation}\label{eq:Lismooth}
 f_i(z) \; \leq \; f_i(w) + \dotprod{\nabla f_i(w), z-w} +\frac{L_i}{2}\norm{z-w}^2, \quad \forall z,w \in \R^{d}.
\end{equation}

We will also use the interpolation Assumption~\ref{ass:interpolate} throughout this section. Thus 
\[f_i(w^*) \; = \; \min_{w\in\R^d} f_i(w) \; \leq \; f(z), \quad \mbox{for all }i \in\{1,\ldots, n\}, \; z\in \R^d. \]

 Using smoothness and interpolation, we first establish the following descent lemma.
\begin{lemma}\label{lem:specialsmooth}
If  the interpolation Assumption~\ref{ass:interpolate} holds and  each $f_i(w)$ is $L_i$--smooth~\eqref{eq:Lismooth}
then 
\begin{equation} \label{eq:descentfi}
f_i(w) -f_i(w^*) \; \geq \;
\frac{1}{2 L_i}\norm{\nabla f_i(w)}^2, \quad \forall w\in\R^d, \; i=1,\ldots, n.
\end{equation}
\end{lemma}
\begin{proof}
Let $w^*$ be a minimizer of $f(w)$. Consequently by the interpolation assumption for every  $z \in \R^d$ we have that  $f_i(w^*) -f_i(z) \leq 0 $ and  for every $w\in\R^d$ we have that
\begin{eqnarray*}
f_i(w^*) -f_i(w) & \leq &  f_i(w^*) -f_i(z) + f_i(z) -f_i(w) \\
& \leq & f_i(z) -f_i(w) \\
& \overset{\eqref{eq:Lismooth}}{\leq} &
\dotprod{\nabla f_i(w), z-w} +\frac{L_i}{2}\norm{z-w}^2
\end{eqnarray*}
Minimizing the right hand side in $z$ gives $z = w -\frac{1}{L_i} \nabla f_i(w)$ which when plugged in the above gives
\begin{eqnarray*}
f_i(w^*) -f_i(w) & \leq & 
-\frac{1}{2 L_i}\norm{\nabla f_i(w)}^2.
\end{eqnarray*}
Re-arranging gives~\eqref{eq:descentfi}.

\end{proof}

\subsection{Proof of Lemma~\ref{lem:SPShsconvex} }

First we show that, under interpolation, if $f_i$ is star-convex, then the auxiliary functions in~\eqref{eq:proxyfun} and~\eqref{eq:proxyfuni}  are also star convex....

\begin{lemma}\label{lem:SPShsconvexapp}
Let the interpolation Assumption~\ref{ass:interpolate} hold.
If every $f_i$ is star convex   along the iterates $(w^t)$ given by~\eqref{eq:SPS}, that is, 
\begin{equation}\label{eq:starcvxfi}
f_i(w^*) \geq f_i(w) +
 \dotprod{\nabla f_{i}(w), w^* -w }
\end{equation} 
then $h_{i,t}(w)$ is star convex along the iterates $(w^t)$ with
\begin{equation} \label{eq:fiwwstarconvex}
h_{i,t}(w^*) \geq h_{i,t}(w^t) +
 \dotprod{\nabla_w  h_{i,t}(w^t), w^* -w },
\end{equation}
so long as $w^t \neq w^*$. Consequently we have that $h_t$ is star convex around $w^*$. 

Furthermore if $f_i$ is $\mu_i$-strongly convex and $L_i$--smooth  then $h_{i,t}$ is $\frac{1}{2}\frac{\mu_i}{L_i}$--strongly star-convex. Consequently $h_t(w)$ is $\frac{1}{2n}\sum_{i=1} \frac{\mu_i}{L_i}$--strongly star-convex 
\begin{equation}\label{eq:spshtstrongconvex}
 h_{t}(w^*) \geq h_{t}(w^t) +
 \dotprod{\nabla_w  h_{t}(w^t), w^* -w } +\frac{1}{4n}\sum_{i=1}^n\frac{\mu}{L_i} \norm{w^t-w^*}^2.
\end{equation}
\end{lemma}

\begin{proof}
Using that $h_{i,t}(w^*)=0$ and that $h_{i,t}(w^t) >0 $ since $w^t \neq w^*$ we have that
\[h_{i,t}(w^*) \geq h_{i,t}(w^t) +
 \dotprod{\nabla_w  h_{i,t}(w^t), w^* -w }\]
 \[\hspace{1cm} \Updownarrow \mbox{ (By definition~\eqref{eq:proxyfuni} })\]
\[0 \geq \frac{1}{2}\left(\frac{f_i(w^t)-f_{i}(w^*)}{\norm{\nabla f_i(w^t)}}\right)^2 +
 \dotprod{\frac{f_{i}(w^t)-f_{i}(w^*)}{\norm{\nabla f_{i}(w^t)}^2}\nabla f_{i}(w^t), w^* -w^t }\]
\[\hspace{4.5cm} \Updownarrow \left(\mbox{Multipling by }\norm{\nabla f_i(w^t)}^2/(f_i(w^t)-f_{i}(w^*)) \geq 0.\right)\]
\[0 \geq \frac{1}{2}(f_i(w^t)-f_i(w^*)) +
 \dotprod{\nabla f_{i}(w^t), w^* -w^t }\]
\[\hspace{2.5cm} \Uparrow \mbox{ (Using that $f_i(w^t)-f_i(w^*) \geq 0$ )}\]
 \[f_i(w^*) \geq f_i(w^t) +
 \dotprod{\nabla f_{i}(w^t), w^* -w^t },\]
 where we used $f_i(w^t)-f_i(w^*) \geq 0$ which is a consequence of interpolation. This proves~\eqref{eq:fiwwstarconvex}
 
 Now if we assume that $f_i$ is $\mu$-strongly star-convex and $L_i$--smooth then we have that by
 
 \begin{equation}\label{eq:SPShitstrongconvex}
 h_{i,t}(w^*) \geq h_{i,t}(w^t) +
 \dotprod{\nabla_w  h_{i,t}(w^t), w^* -w } +\frac{1}{4}\frac{\mu}{L_i} \norm{w^t-w^*}^2
\end{equation}  
 \[\hspace{1cm} \Updownarrow \mbox{ (By definition~\eqref{eq:proxyfuni} })\]
\[0 \geq \frac{1}{2}\left(\frac{f_i(w^t)-f_{i}(w^*)}{\norm{\nabla f_i(w^t)}}\right)^2 +
 \dotprod{\frac{f_{i}(w^t)-f_{i}(w^*)}{\norm{\nabla f_{i}(w^t)}^2}\nabla f_{i}(w^t), w^* -w^t } +\frac{1}{4}\frac{\mu}{L_i} \norm{w^t-w^*}^2\]
\[\hspace{4.5cm} \Updownarrow \left(\mbox{Multipling by }\norm{\nabla f_i(w^t)}^2/(f_i(w^t)-f_{i}(w^*)) \geq 0.\right)\]
\[0 \geq \frac{1}{2}(f_i(w^t)-f_i(w^*)) +
 \dotprod{\nabla f_{i}(w^t), w^* -w^t }+\frac{\norm{\nabla f_i(w^t)}^2}{f_i(w^t)-f_{i}(w^*)}\frac{1}{4}\frac{\mu}{L_i} \norm{w^t-w^*}^2\]
\[\hspace{2.5cm} \Uparrow \mbox{ (Using that $f_i(w^t)-f_i(w^*) \geq 0$ )}\]
 \[f_i(w^*) \geq f_i(w^t) +
 \dotprod{\nabla f_{i}(w^t), w^* -w^t }+\frac{\norm{\nabla f_i(w^t)}^2}{f_i(w^t)-f_{i}(w^*)}\frac{1}{4}\frac{\mu}{L_i} \norm{w^t-w^*}^2.\]
 Finally,
from smoothness and Lemma~\ref{lem:specialsmooth} we have that $1 \geq  \frac{1}{2L_i}\frac{\norm{\nabla f_i(w^t)}^2}{f_i(w^t)-f_{i}(w^*)}$
consequently
 \begin{align}
   f_i(w^*) &  \geq f_i(w^t) + \dotprod{\nabla f_{i}(w^t), w^* -w^t }+\frac{\mu}{2}\norm{w^t-w^*}^2 \nonumber\\
&  \geq f_i(w^t) + \dotprod{\nabla f_{i}(w^t), w^* -w^t }+ \frac{1}{4}\frac{\mu}{L_i}\frac{\norm{\nabla f_i(w^t)}^2}{f_i(w^t)-f_{i}(w^*)}\norm{w^t-w^*}^2.
\end{align}
Consequently the above implications hold, and thus $h_{i,t}$ is $\frac{1}{4}\frac{\mu}{L_i}$-strongly star convex. Taking the average of~\eqref{eq:SPShitstrongconvex} over $i$ gives~\eqref{eq:spshtstrongconvex}, which concludes the proof.
\end{proof}

 \subsection{Proof of Corollary~\ref{cor:SPSconvexconv} and~\ref{cor:SPSstrongconvexconv}} 
\label{sec:SPStheorystar}

Having established when $h_t$  is star convex and strongly star convex, we can now apply Theorems~\ref{theo:onlinesgdstar}  and Theorem~\ref{theo:onlinesgdstrongstar}, which when specialized to \SP gives the following corollaries. This result has already been established in Theorem 4.4 and Theorem D.3 in~\cite{SGDstruct}. Thus here we have showed that the results in ~\cite{SGDstruct}  follow as a direct consequence of the interpretation of \SP as a variant of the online SGD method.
We also extend the following theorem to allow for $\gamma =1$ in Theorem~\ref{theo:polyakmasterstoch}. 

\begin{corollary}\label{cor:SPSconvexconvapp}
If  $\gamma <1$ and every $f_i(w)$ is star-convex along the iterates $(w^t)$ given by~\eqref{eq:SPS} then
\begin{equation} \label{eq:interstarcvxsgdview}
 \frac{1}{k}\sum_{t=0}^k \frac{1}{2n}\sum_{i=1}^n \E{\left(\frac{f_i(w^t)-f_i(w^*)}{\norm{\nabla f_i(w^t)}}\right)^2} \; \leq \; \frac{1}{k}\frac{1}{2\gamma(1-\gamma) }\E{\norm{w^{0} -w^*}^2}.
\end{equation} 
Furthermore if  the interpolation Assumption~\ref{ass:interpolate} holds and if each $f_i(w)$ is $L_i$--smooth 
then 
\begin{equation}\label{eq:spsstarconvfinalresult}
 \min_{t=0 , \ldots, k} \E{f(w^t)-f^*}  \; \leq \; \frac{1}{k}\frac{L_{\max}}{2\gamma(1-\gamma) }\E{\norm{w^{0} -w^*}^2},
\end{equation} 
where  $L_{\max} \; \eqdef\; \max_{i=1,\ldots, n} L_i$.
\end{corollary}
\begin{proof}
The proof of~\eqref{eq:interstarcvxsgdview} follows as a special case of Theorem~\ref{theo:onlinesgdstar} by identifying $h_t$ with~\eqref{eq:proxyfun} and $h_{i_t,t}$ with~\eqref{eq:proxyfuni}. Indeed, according to~\eqref{eq:subsmooth1} we have that $h_t$ satisfies the growth condition~\eqref{eq:hweakgrow} with $G=1$ and according to~\eqref{eq:fiwwstarconvex}  $h_t$ is star-convex~\eqref{eq:hstar} around $w^*$. Finally since $h_t(w^*) =0$ the result~\eqref{eq:interstarcvxsgdview} follows by Theorem~\ref{theo:onlinesgdstar}.

The result~\eqref{eq:spsstarconvfinalresult}  would follow from~\eqref{eq:interstarcvxsgdview} if
 \begin{equation}\label{eq:smoothlwlw}
L_{\max}\,   \frac{1}{n} \sum_{i=1}^n \frac{( f_i(w) -f_i(w^*))^2}{\norm{\nabla f_i(w)}^2} \geq 2 (f(w)-f^*).
\end{equation}
This Assumption has appeared recently in~\cite{SGDstruct} where it was proven that~\eqref{eq:smoothlwlw} is a consequence of each $f_i(w)$ being $L_i$--smooth. We give  a simpler proof next  for completeness. 
That is, assuming that there exists $w$ such that $f_i(w) \neq f_i(w^*)$ and thus $\nabla f_i(w) \neq 0$ (otherwise~\eqref{eq:smoothlwlw} holds trivially) we have from~\eqref{eq:descentfi} that
\[\frac{1}{ \norm{\nabla f_i(w)}^2} \; \geq \; \frac{1}{2L_i (f_i(w)-f_i(w^*)) }.\]
Multiplying both sides by $(f_i(w) -f_i(w^*))^2$ and averaging over  $i = 1,\ldots, n$ gives
\begin{align*}
 \frac{1}{n} \sum_{i=1}^n\frac{(f_i(w) -f_i(w^*))^2}{ \norm{\nabla f_i(w)}^2} & \geq  2 \frac{1}{n} \sum_{i=1}^n \frac{f_i(w) -f_i(w^*)}{L_i} \; \geq \;  \frac{1}{n} \sum_{i=1}^n\frac{ 2  f_i(w) -f_i(w^*) }{\max_{i=1,\ldots, n} L_i}  \\
  & =  \frac{ 2(f(w)-f^*)}{L_{\max}} .
\end{align*}
Using~\eqref{eq:smoothlwlw} and~\eqref{eq:interstarcvxsgdview} we have
\begin{align*}
 \min_{t=0 , \ldots, k} \E{f(w^t)-f^*} & \leq \frac{1}{k} \sum_{t=0}^k   \E{f(w^t)-f^*}\\
&\leq  \frac{1}{k}\sum_{t=0}^k \frac{L_{\max}}{2n}\sum_{i=1}^n \E{\left(\frac{f_i(w^t)-f_i(w^*)}{\norm{\nabla f_i(w^t)}}\right)^2} \\ & \leq  \frac{1}{k}\frac{L_{\max}}{2\gamma(1-\gamma) }\E{\norm{w^{0} -w^*}^2}
\end{align*}
which concludes the proof of
\end{proof}

\begin{corollary}\label{cor:SPSstrongconvexconvapp}
If $\gamma \leq 1$, the interpolation Assumption~\ref{ass:interpolate} holds, and every $f_i$ is $L_i$--smooth and $\mu$--strongly star-convex then the iterates $w^t$ given by~\eqref{eq:SPS} converge linearly according to
 \begin{eqnarray}\label{eq:SPSlincomnverge}
\E{\norm{w^{t+1} -w^*}^2} & \leq &   \left(1-\gamma \frac{1}{2n}\sum_{i=1}^n \frac{\mu_i}{L_i}\right)^{t+1} \norm{w^{0} -w^*}^2 
\end{eqnarray}

\end{corollary}

\begin{proof}
The proof of~\eqref{eq:SPSlincomnverge} follows as a special case of Theorem~\ref{theo:onlinesgdstrongstar} by identifying $h_t$ with~\eqref{eq:proxyfun} and $h_{i_t,t}$ with~\eqref{eq:proxyfuni}. Indeed, according to~\eqref{eq:subsmooth1} we have that $h_t$ satisfies the growth condition~\eqref{eq:hweakgrow} with $G=1$. Furthermore $f_i$ is $\mu_i$--strongly star convex and $L_i$--smooth, then from Lemma~\ref{lem:SPShsconvexapp} we have that $h_{t}$ is $\frac{1}{2n}\sum_{i=1}^n \frac{\mu_i}{L_i}$--strongly star convex. Finally since $h_t(w^*) =0$ the result~\eqref{eq:interstarcvxsgdview} follows by Theorem~\ref{theo:onlinesgdstrongstar}.
\end{proof}

\section{Convergence of the Targeted Stochastic Polyak Stepsize}
\label{sec:convTAPS}

Here we explore  the consequences  and conditions 
of Theorem~\ref{theo:onlinesgdstar}  for the  TAPS method given in Algorithm~\ref{alg:TAPS}.

\subsection{Proof of Corollary~\ref{theo:TAPSconvsum} and more } \label{sec:starconvexTAPS}

First we re-state Theorem~\ref{theo:onlinesgdstar} specialized to Algorithm~\ref{alg:TAPS}.
\begin{corollary}\label{theo:TAPSconvsum}
Let $h_t(z) $ be defined in~\eqref{eq:proxyobjalpha} and
suppose that $h_t(z)$ is star convex~\eqref{eq:hstar} around $z^* = (w^*, \alpha^*)$ and along the  iterates $z^t = (w^t, \alpha^t)$ of Algorithm~\ref{alg:TAPS}.

If $\gamma < 1$ and in addition $f_i(w)$ is $L_{\max}$--Lipschitz then
\begin{equation} \label{eq:TAPSconvL}
\min_{t=1,\ldots, k} \frac{1}{n+1}  \left(\sum_{i=1}^n\frac{\E{f_i(w^t) -\alpha_i^t}^2}{L_{\max}+1}  + \E{\overline{\alpha}^t - \tau}^2 \right)\; \leq \; \frac{1}{k}\frac{1}{\gamma(1-\gamma) }\E{\norm{w^{0} -w^*}^2}.
\end{equation} 

Alternatively, if  $h_t(z)$ is  $\mu$--strongly star--convex~\eqref{eq:hquasistrong} then 
\begin{equation}\label{eq:convTAPSstrconv}
\E{\norm{w^t-w^*}^2 + \sum_{i=1}^n\norm{\alpha^t_i -f_i(w^*)}^2 } \leq (1- \gamma \mu)^t \left(\norm{w^0-w^*}^2 + \sum_{i=1}^n\norm{\alpha^0_i -f_i(w^0)}^2\right).
\end{equation}

\end{corollary}
Theorem~\ref{theo:TAPSconvsum} provides us with a $\cO(1/k)$ convergence in expectation when $h_t(z)$ is star convex.
Indeed,  the bound in~\eqref{eq:TAPSconvL} shows that $\overline{\alpha}$ converges to $\tau$ at a rate of $\cO(1/k).$ Finally from the target assumption~\eqref{eq:tau} we  have that $h_t(z^*) = 0$, thus $f_i(w^t)$ and $\alpha_i^t$ converge to $f_i(w^*)$ at a rate of  $\cO(1/k).$
\begin{proof}
The proof follows by applying Theorem~\ref{theo:onlinesgdstar}. Indeed, by letting $h_{i,t}(z)  =  \frac{1}{2}\frac{(f_i(w) -\alpha_i)^2}{\norm{\nabla f_i(w^t)}^2+1}$ for $i=1,\ldots, n$ and $h_{n+1,t}(z) =  \frac{n}{2}(\overline{\alpha} - \tau)^2.$ Thus  $h_t(z) = \frac{1}{n+1} \sum_{i=1}^{n+1} h_{i,t}(z).$ By Lemma~\ref{lem:TAPSweakgrowth} we have that 
\begin{eqnarray*}
 \EE{i \sim \frac{1}{n+1}}{ \| \nabla h_{i,t}(z^t) \|^2} 
  & = & \frac{1}{n+1} \sum_{i=1}^{n+1} \| \nabla h_{i,t}(z^t) \|^2 \\
 & =&   \frac{1}{n+1} \left(\sum_{i=1}^{n}  \norm{\nabla f_{i,w^t}(w^t,\alpha^t) }^2  +\norm{\nabla h_{n+1}(\alpha^t) }^2 \right)\\
 & \overset{\eqref{eq:smooth1}}{=} & \frac{1}{n+1} \left(\sum_{i=1}^{n}   2 f_{i,w^t}(w^t,\alpha^t) +   2 f_{n_1}(\alpha^t) \right) \\
 & \overset{\eqref{eq:hz}+\eqref{eq:fi}  }{=} &
  2 h_t(z^t).
\end{eqnarray*} 
Consequently
$h_t$ satisfies the growth condition~\eqref{eq:hweakgrow} with $G =1$. By assumption $h_t$ is star convex along the iterates $z^t$, thus the two condition required for Theorem~\ref{theo:onlinesgdstar} to hold are satisfied, and as a consequence, we have that~\eqref{eq:hconv} holds.
Substituting out $h_{t}(z^t)$  we have that
\begin{equation}\label{eq:zojapjo9sj41}
 \frac{1}{k}\sum_{t=0}^k \frac{1}{n+1} \left(\sum_{i=1}^n \frac{1}{2}\frac{(f_i(w^t) -\alpha_i^t)^2}{\norm{\nabla f_i(w^t)}^2+1}  + \frac{n}{2}(\overline{\alpha}^t - \tau)^2 \right)\; \leq \; \frac{1}{k}\frac{1}{2\gamma(1-\gamma) }\E{\norm{w^{0} -w^*}^2}.
\end{equation} 

Furthermore, if $f_i$ is $L_{\max}$--Lipschitz, that is if $\| \nabla f_i(w^t)\| \leq L_{\max}$ then from~\eqref{eq:zojapjo9sj41} we have that
\begin{equation}\label{eq:zojapjo9sj4}
 \frac{1}{k}\sum_{t=0}^k \frac{1}{n+1} \left(\sum_{i=1}^n \frac{1}{2}\frac{(f_i(w^t) -\alpha_i^t)^2}{L_{\max}+1}  + \frac{n}{2}(\overline{\alpha}^t - \tau)^2 \right)\; \leq \; \frac{1}{k}\frac{1}{2\gamma(1-\gamma) }\E{\norm{w^{0} -w^*}^2},
\end{equation} 
from which~\eqref{eq:TAPSconvL} follows by lower bounding the average over $k$ by the minimum.

Finally,  if  there exists $\mu >0$ such that $h_t(z)$ is strongly star-convex~\eqref{eq:hquasistrong}, then by noting that \[ \norm{z - z^*}^2 = \norm{w^t-w^*}^2 + \norm{\alpha^t -\alpha^*}^2  = \norm{w^t-w^*}^2 + \sum_{i=1}^n\norm{\alpha^t_i -f_i(w^*)}^2 \] we have that~\eqref{eq:hconvsstrong}
gives~\eqref{eq:convTAPSstrconv}.

\end{proof}

\subsection{Proof of Lemmas~\ref{lem:convexTAPS} and Corollary~\ref{cor:localconvexTAPS} }
\label{asec:suffcondconvTAPS}

For ease of reference, we first re-state the lemmas.
\begin{lemma} [Locally Convex]\label{lem:convexTAPSapp}
Consider the iterates of Algorithm~\ref{alg:MOTAPS}.
Let $(w,\alpha) \in \R^{d+n}$ and consider $h_t(w,\alpha)$ defined in~\eqref{eq:hz} .
 Assume that the  gradients at $w$ spans the entire space, that is
\begin{equation}\label{eq:gradspant}
\mbox{span}\left\{\nabla f_1(w), \ldots, \nabla f_n(w)\right\} \; = \; \R^{d}, \quad \forall w.
\end{equation}
If Assumption~\ref{ass:target} holds,  every $f_i(w)$ for $i=1,\ldots, n$ is twice continuously differentiable  and
\begin{equation}\label{eq:hessifiminusalphapp}
\frac{1}{n+1}\sum_{i=1}^n\nabla^2 f_i(w^t) \frac{f_i(w^t)-\alpha_i^t}{\norm{\nabla f_i(w^t)}^2 +1}  \succeq 0, \quad \forall t,
\end{equation}
then $h_t$ is strictly convex with at $(w^t,\alpha^t)$ that is
\[ \nabla^2  h_t(w^t,\alpha^t) \;\succ \; 0, \quad \forall t.\]
\end{lemma}
\begin{proof}
We have $ (f_i(w)-\alpha_i)^2$ is locally convex, and thus star convex,  iff its Hessian is positive definite around $(w^*,\alpha^*)$. 
Computing the
gradient of $ (f_i(w)-\alpha_i)^2$  we have that
\[  \nabla (f_i(w)-\alpha_i)^2 = 2\begin{bmatrix}
\nabla f_i(w) \\
-1
\end{bmatrix}(f_i(w)-\alpha_i)\]
Computing the Hessian gives
\begin{align}
 \nabla^2 (f_i(w)-\alpha_i)^2 &=
  2\begin{bmatrix}
\nabla f_i(w) \\
-1
\end{bmatrix}\begin{bmatrix}
\nabla f_i(w) ^\top & - 1
\end{bmatrix}
+
2\begin{bmatrix}
\nabla^2 f_i(w) & 0\\
0 & 0
\end{bmatrix}(f_i(w)-\alpha_i) \nonumber \\
& =  
2\begin{bmatrix}
\nabla f_i(w)\nabla f_i(w) ^\top & -\nabla f_i(w) \\
-\nabla f_i(w) ^\top & 1
\end{bmatrix}
+
2\begin{bmatrix}
\nabla^2 f_i(w) & 0\\
0 & 0
\end{bmatrix}(f_i(w)-\alpha_i) \label{eq:fialphai2hess}
\end{align}

Now let $\mI_n \in \R^{n\times n}$ be the identity matrix in $\R^{n\times n}$, let
\begin{align}
  \mD_t & \eqdef \; 
\diag{\frac{1}{\norm{\nabla f_1(w^t)}^2+1}, \ldots, \frac{1}{\norm{\nabla f_n(w^t)}^2+1}} \in \R^{n\times n} \nonumber \\
  \mH_t(w, \alpha) & \eqdef \;  \sum_{i=1}^{n+1} \nabla^2 f_i(w)  \frac{f_i(w)-\alpha_i}{\norm{\nabla f_i(w^t)}^2+1} \label{eq:zloen49zj9jzz4}
\end{align}
and let
\[D F(w) \; \eqdef
 \; \begin{bmatrix}\nabla f_1(w), \ldots, \nabla f_n(w)\end{bmatrix} \in \R^{d\times n}.\]
Using~\eqref{eq:fialphai2hess} and by the definition of $h_t$ in~\eqref{eq:hz} we have that
\begin{equation}\label{eq:ht2hessstar}
 \nabla^2 h_t(w,
 \alpha) \; = \; 
 \frac{1}{n+1} \underbrace{ \begin{bmatrix}
 DF(w) \mD_t DF(w)^\top & -DF(w)\mD_t  \\
-(DF(w)\mD_t)^\top & \mI_n (1+\frac{1}{n}),
\end{bmatrix} }_{\eqdef \mM_t(w)}
+
\begin{bmatrix}
\mH_t(w,\alpha)  & 0  \\
0 & 0
\end{bmatrix} 
\end{equation}
where we used the $\nabla^2 \frac{n}{2}(\overline{\alpha} - \tau)^2  = \frac{1}{n} \mI_n$.
Thus the matrix~\eqref{eq:ht2hessstar} is a sum of two terms. By the assumption~\eqref{eq:hessifiminusalphapp} 
we have that the second part that contains $\mH_t(w, \alpha)$ is positive semi-definite. Next we will show that the first matrix $\mM_t(w)$
is  symmetric positive definite. Indeed, 
left and right multiplying the above by $[x,\,a]  \in \R^{d+n}$ gives
\begin{eqnarray*}
 \begin{bmatrix}
x & a
\end{bmatrix}^\top 
\mM_t(w)
\begin{bmatrix}
x \\ a
\end{bmatrix} 
&\overset{\eqref{eq:ht2hessstar}}{ =} &
 \begin{bmatrix}
x & a
\end{bmatrix}^\top 
\begin{bmatrix}
 D F(w)  \mD_t D F(w) ^\top x  -D F(w) \mD_t a \\
-(D F(w) \mD_t)^\top x + a (1+\frac{1}{n}).
\end{bmatrix}  \\
&= &
\norm{\mD_t^{1/2}D F(w) ^\top x }^2 -2a (D F(w) \mD_t)^\top x+ (1+\frac{1}{n}) \norm{a}^2\\
& = & \norm{\mD_t^{1/2}(D F(w) ^\top x -a) }^2-\norm{\mD_t^{1/2} a}^2 +  (1+\frac{1}{n}) \norm{a}^2,
\end{eqnarray*}
or in  short
\begin{equation} \label{eq:aoa3oau3nua}
\begin{bmatrix}
x & a
\end{bmatrix}^\top 
\mM_t(w)
\begin{bmatrix}
x \\ a
\end{bmatrix}  \;= \;\norm{D F(w) ^\top x -a }_{\mD_t}^2 +  \norm{a}_{(1+\frac{1}{n}) \mI_n-\mD_t}^2
\end{equation}
Next we show that~\eqref{eq:aoa3oau3nua} is strictly positive for every $(x,a)\neq 0$. To this end, first note that the matrix $(1+\frac{1}{n}) \mI_n-\mD_t $ is positive definite, which follows since the $i$th diagonal element is positive with
\[ \begin{bmatrix}
(1+\frac{1}{n}) \mI_n-\mD_t
\end{bmatrix}_{ii} \; = \; 1 + \frac{1}{n} - \frac{1}{\norm{\nabla f_i(w^t)}^2 +1} \; >\; 0.
\]
Consequently if $a \neq 0$ we have that~\eqref{eq:aoa3oau3nua} is strictly positive. On the other hand, if $a =0$
let us prove by contradiction that~\eqref{eq:aoa3oau3nua} is still positive for $x\neq 0$. Indeed  suppose that $x\neq 0$ and
\[  \norm{D F(w) ^\top x }_{\mD_t}^2 =0 \; \overset{\mD_t \succ 0 }{\implies} \; \sum_{i=1}^n \nabla f_i(w)^\top x =0.\]
 But due to our assumption~\eqref{eq:gradspan},  we have that $D F(w) ^\top$ has full column rank, and thus $x =0$, which is a contradiction. Thus~\eqref{eq:aoa3oau3nua} is positive for every $(x,a)\neq 0$ from which we conclude that the Hessian $\nabla^2  h_t(w,\alpha)  $ in~\eqref{eq:ht2hessstar} is positive definite. 

 \end{proof}

%

The proof of Corollary~\ref{cor:localconvexTAPS} then follows from Lemma~\ref{lem:convexTAPSapp} by plugging in $\alpha^*_i = f_i(w^*)$ into~\eqref{eq:zloen49zj9jzz4}.
%
 
\section{Convergence of the Moving Target Stochastic Polyak Stepsize}
\label{sec:convMOTAPS}

Here we explore the consequences of Theorems~\ref{theo:onlinesgdstar} and ~\ref{theo:onlinesgdstrongstar} specialized to Algorithm~\ref{alg:MOTAPS}.  Throughout this section let 
 \begin{equation} \label{eq:lambdaboundapp}
\lambda \leq \frac{2n+1}{2n+3} < 1
\end{equation}
 and let  $z^t \eqdef (w^t, \alpha^t, \tau_t)$ be the iterates of Algorithm~\ref{alg:MOTAPS} when using a stepsize $\gamma = \gamma_{\tau}.$
 Let
\begin{equation}\label{eq:hztau}
 h_t(z) \; \eqdef \; \frac{1}{n+1}\left(\sum_{i=1}^n \frac{1-\lambda}{2}\frac{(f_i(w) -\alpha_i)^2}{\norm{\nabla f_i(w^t)}^2+1}  + \frac{n(1-\lambda)}{2}(\overline{\alpha} - \tau)^2+\frac{\lambda}{2}\tau^2 \right). 
\end{equation}
and let $w^*$ be a minimizer of~\eqref{eq:main} and let
\begin{equation} \label{eq:alphastarandtaustarapp}
\alpha_i^* \eqdef f_i(w^*) \quad \mbox{and} \quad \tau^* = f(w^*) , \quad \mbox{for }i=1,\ldots, n.
\end{equation}

 \subsection{Proof of Corollary~\ref{cor:MTAPSconvsum}}
\begin{corollary}\label{cor:MTAPSconvsumapp}
If  $\gamma = \gamma_{\tau} = \frac{1}{2(1-\lambda)(2n+1)}$
and if  $h_t(z)$ is star convex along the iterates $z^t $ and around $z^* \eqdef (w^*, \alpha^*, \tau^*) $ then
\begin{equation}\label{eq:MTAPSconvh}
\min_{t=0,\ldots, k} \E{h_{t}(z^t)-h_t(z^*)} \; \leq \; \frac{2(1-\lambda)(2n+1)}{k} \norm{z^{0} -z^*}^2+\frac{\lambda f(w^*)^2}{2(n+1)}.
\end{equation} 
Furthermore, if $f_i$ is $L_{\max}$--Lipschitz then
\begin{align}
 \frac{1}{n+1}\E{\sum_{i=1}^n \frac{1}{2}\frac{(f_i(w^t) -\alpha_i^t)^2}{L_{\max}+1}  + \frac{n}{2}(\overline{\alpha}^t - \tau^t)^2+\frac{\lambda}{2}\left( (\tau^t)^2  - f(w^*)^2\right)} \nonumber \\
 \quad \quad  \; \leq \;\frac{2(1-\lambda)(2n+1)}{k} \norm{z^{0} -z^*}^2+\frac{\lambda f(w^*)^2}{2(n+1)}. \label{eq:hztaufilipschitz}
 \end{align}
\end{corollary}
\begin{proof}
The proof follows by applying Theorem~\ref{theo:onlinesgdstar} and Lemmas~\ref{lem:MOSTAPSweakgrowthn1} and~\ref{lem:movtargequiv}.
Indeed $h_t$ satisfies the growth condition~\eqref{eq:hweakgrow} with $G = (1-\lambda)(2n+1)$. By assuming that $h_t$ is star convex along the iterates $z^t$ we have satisfied the two condition required for Theorem~\ref{theo:onlinesgdstar} to hold, 
which when substituting in $G$ and  $\gamma = \frac{1}{2(1-\lambda)(2n+1)}$ gives
\begin{align}
\min_{t=1,\ldots, k} \E{h_t(z^t) -h_t(z^*)} 
 & \leq \; \frac{2(1-\lambda)(2n+1)}{k} \E{\norm{z^{0} -z^*}^2} +\frac{1}{k} \sum_{t=1}^k h_t(z^*).\label{eq:mottapstarapptemp}
\end{align}

Furthermore using the bound~\eqref{eq:htstarmain}  in Lemma~\ref{lem:movtargequiv} 
we have that
\begin{align*}
\frac{1}{k} \sum_{t=1}^k h_t(z^*)
& =   \frac{\lambda f(w^*)^2}{2(n+1)} 
\end{align*}
and thus~\eqref{eq:MTAPSconvh} holds.  Finally, if $f_i$ is $L_{\max}$--Lipschitz, that is if $\| \nabla f_i(w^t)\| \leq L_{\max}$, then 
using the definition of $h_t(z)$ in ~\eqref{eq:hztau}  we can lower bound $h_t(z^t)-h_t(z^*)$ by the left-hand side of~\eqref{eq:hztaufilipschitz}.
\end{proof}

 \subsection{Proof of Corollary~\ref{cor:MTAPSconvhstrong}}
 
\begin{corollary}\label{cor:MTAPSconvhstrong}
If  $\gamma = \gamma_{\tau} = \frac{1}{(1-\lambda)(2n+1)}$
and if  $h_t(z)$ is $\mu$--strongly star--convex along the iterates $z^t $ and around $z^* \eqdef (w^*, \alpha^*, \tau^*) $ then
\begin{eqnarray}\label{eq:MTAPSconvhstrong}
\E{\norm{z^{t+1} -z^*}^2} & \leq &  \big(1- \frac{\mu}{(1-\lambda)(2n+1)} \big)^{t+1} \norm{z^{0} -z^*}^2  +  \frac{\lambda f(w^*)^2}{\mu(n+1)}.
\end{eqnarray}
\end{corollary}

 \begin{proof}
The proof follows by applying Theorem~\ref{theo:onlinesgdstrongstar} and Lemmas~\ref{lem:MOSTAPSweakgrowthn1} and~\ref{lem:movtargequiv}.
Indeed by Lemma~\ref{lem:MOSTAPSweakgrowthn1}
$h_t$ satisfies the growth condition~\eqref{eq:hweakgrow} with $(1-\lambda)(2n+1)$. By assuming that $h_t$ is $\mu$--strongly star convex along the iterates $z^t$ we have satisfied the two condition required for Theorem~\ref{theo:onlinesgdstrongstar} to hold. Finally
using~\eqref{eq:htstarmain}  in Lemma~\ref{lem:movtargequiv}
we have that
\begin{align*}
h_t(z^*)
& =  \frac{\lambda f(w^*)^2}{2(n+1)} 
\end{align*}
  and as a consequence  Theorem~\ref{theo:onlinesgdstrongstar} gives
  \begin{eqnarray}\label{eq:MTAPSconvhproof}
\E{\norm{z^{t+1} -z^*}^2} & \leq &   (1- \frac{\mu}{(1-\lambda)(2n+1)} )^{t+1} \norm{z^{0} -z^*}^2 \nonumber \\
&&\quad  + \frac{2}{(1-\lambda)(2n+1)} \sum_{i=0}^{t}(1-\gamma \mu)^i  \frac{\lambda f(w^*)^2}{2(n+1)}. \nonumber \\
& \leq &  (1- \frac{\mu}{(1-\lambda)(2n+1)} )^{t+1} \norm{z^{0} -z^*}^2 \nonumber \\
&&\quad  + \frac{2}{(1-\lambda)(2n+1)}\frac{1}{\gamma \mu}   \frac{\lambda f(w^*)^2}{2(n+1)}. \nonumber \\
& =&  (1- \frac{\mu}{(1-\lambda)(2n+1)} )^{t+1} \norm{z^{0} -z^*}^2  +  \frac{\lambda f(w^*)^2}{\mu(n+1)},\label{eq:tempusedlater}
\end{eqnarray}
where in the last equality we used that $\gamma = \frac{1}{(1-\lambda)(2n+1)}.$
\end{proof}

 \subsection{Proof of Theorem~\ref{theo:complexdecreaseMOTAPS}}
 
\begin{theorem} \label{theo:complexdecreaseMOTAPS}
Let $h_t(z)$
be  $\mu$-strongly star--convex along the iterates $z^t $ and around $z^* \eqdef (w^*, \alpha^*, \tau^*).$ Let $\epsilon >0.$ 
If we use an iteration dependent stepsize in Algorithm~\ref{alg:MOTAPS} given by
\begin{align}\label{eq:gammatmotaps}
\gamma_t   =  
\begin{cases} \displaystyle
 \frac{1}{(1-\lambda)(2n+1)} & \mbox{if } t \leq 2  (2n+1)\left \lceil \frac{1-\lambda}{\mu} \right\rceil \\[0.5cm]
  \displaystyle  \frac{(t+1)^2 -t^2}{\mu (t+1)^2} & \mbox{if } t \geq 2  (2n+1)\left \lceil \frac{1-\lambda}{\mu} \right\rceil
\end{cases}
\end{align}
and if
\[\lambda \leq  \min \left\{ 1- \frac{2\mu}{2n+1}, \; \frac{2n+1}{2n+3} \right\} .\]
then
\begin{equation}\label{eq:complexdecreaseMOTAPS}
\EE{}{\norm{z^{t} -z^*}^2}  \leq  \frac{(1-\lambda)\lambda f(w^*)^2}{\mu^2} \frac{16}{t} + \frac{4(2n+1)^2}{e^2 t^2} \left \lceil \frac{1-\lambda}{\mu} \right\rceil^2 \norm{z^{0} -z^*}^2.
\end{equation}
\end{theorem}
\begin{proof}

Following the proof of Theorem~\ref{theo:onlinesgdstrongstar} upto~\eqref{eq:theolinearonestep}, we have that for $\gamma \leq \frac{1}{G} = \frac{1}{(1-\lambda)(2n+1)} $ and 
$h_t(z^*)
 =  \frac{\lambda f(w^*)^2}{2(n+1)} $
 that
\begin{align}
\EE{t}{\norm{z^{t+1} -z^*}^2} &  \leq  (1-\gamma \mu) \norm{z^{t} -z^*}^2 +2 \gamma^2(1-\lambda)(2n+1) \frac{\lambda f(w^*)^2}{2(n+1)} \nonumber \\
& \leq  (1-\gamma \mu) \norm{z^{t} -z^*}^2 +4 \gamma^2(1-\lambda)\lambda f(w^*)^2.
\label{eq:onestepmidproofMOTAPScomplex}
\end{align}
Taking expectation and using the abbreviations
\begin{equation}\label{eq:Lzloe8ozjrz}
r_t \eqdef \EE{}{\norm{z^{t} -z^*}^2}  \quad \mbox{and} \quad \sigma^2 \eqdef 2(1-\lambda) \lambda f(w^*)^2,
\end{equation}
gives that
\begin{equation}\label{eq:onestepmidproofMOTAPScomplex2}
r^{t+1} \leq  (1-\gamma \mu) r^t +2 \gamma^2\sigma^2.
\end{equation}
With this notation, this is now identical to the setting of Theorem 3.2 in~\cite{gower2019sgd}. Using the notation of  Theorem 3.2 in~\cite{gower2019sgd} we have that  $2\mathcal{L} = (1-\lambda)(2n+1)$ and consequently $\mathcal{K} =\frac{ \mathcal{L}}{\mu}  =\frac{1}{2}(2n+1)\left \lceil \frac{1-\lambda}{\mu} \right\rceil$.  As a result of Theorem 3.2 in~\cite{gower2019sgd} we have that
\begin{equation}\label{eq:sgdgenanaltysisreults}
r^{t} \leq  \frac{\sigma^2}{\mu^2} \frac{8}{t} + \frac{16 \lceil \mathcal{K} \rceil^2}{e^2 t^2} r^0.
\end{equation}
Substituting back the definitions given in~\eqref{eq:Lzloe8ozjrz} gives~\eqref{eq:complexdecreaseMOTAPS}. Though one detail in the proof of Theorem 3.2 in~\cite{gower2019sgd} is that $\mathcal{K} \geq 1$, which in our case holds it
\[\lambda \leq 1- \frac{2\mu}{2n+1}.\]
\end{proof}
 \section{Convergence of \SP Through Star Convexity with $\gamma =1$}

\label{sec:SPSalternativegamma1theory}
For completeness, we present yet another viewpoint of the \SP method that is closely related to Polyak's original motivation. We also prove convergence of \SP with a large stepsize of $\gamma =1.$ This complements both our convergence result for the \SP method in Corollary~\ref{cor:SPSconvexconv} which holds for $\gamma <1$.

Consider the stochastic gradient method given by
\begin{equation}\label{eq:SGD}
 w^{t+1} = w^t - \gamma^t  \nabla f_i(w^t),
\end{equation}
where $\gamma^t>0$ is a step size which we will now  choose.
Expanding the squares we have that
\begin{align}\label{eq:iterateexpsgd}
\norm{w^{t+1} -w^*}^2  = \norm{w^t -w^*}^2 - 2\dotprod{\gamma^t  \nabla f_i(w^t), w^t-w^* } + \norm{\gamma^t  \nabla f_i(w^t)}^2.
\end{align}

We would like to choose $\gamma^t$ so as to give the best possible upper bound in the above. Unfortunately we cannot directly minimize the above in $\gamma^t$ since we do not know $w^*.$ However, if each loss function is \emph{star convex}, then there is hope.
\begin{assumption}\label{ass:starconvexfi}
We say that $f_i$ is star convex if 
\begin{equation}\label{eq:stari}
f_i(w^*) \geq f_i(w) + \dotprod{\nabla f_{i}(w), w^*-w }, \quad \mbox{for }i=1,\ldots, n.
\end{equation}
\end{assumption}

Using star convexity~\eqref{eq:stari} in the above gives
\begin{align}\label{eq:polyakmot}
\norm{w^{t+1} -w^*}^2  & \overset{\eqref{eq:stari}}{\leq} \norm{w^t -w^*}^2 - 2\gamma^t(f_i(w^t) -f(w^*)) + (\gamma^t)^2 \norm{  \nabla f_i(w^t)}^2.
 \end{align}
We can now minimize the righthand side in $\gamma^t$, which gives exactly the \emph{Polyak stepsize}
\begin{equation}\label{eq:polyak}
\gamma^t \quad =\quad \frac{f_i(w^t) -f_i(w^*)}{ \norm{  \nabla f_i(w^t)}^2} .
\end{equation}
With this stepsize, the resulting update is given by
\begin{equation}\label{eq:polyakupdate}
w^{t+1} = w^t - \frac{f_i(w^t) -f_i(w^*)}{ \norm{  \nabla f_i(w^t)}^2} \nabla f_i(w^t).
\end{equation}
The iterative scheme~\eqref{eq:polyakupdate} is now completely scale invariant. That is, the iterates are invariant to replacing $f_i(w)$ by $c_i\,f_i(w)$ where $c_i>0.$ 

\begin{theorem}[Convergence for $\gamma_t \equiv 1$] \label{theo:polyakmasterstoch}
Let Assumptions~\ref{ass:starconvexfi} and~\ref{ass:interpolate} hold.
The iterates~\eqref{eq:SPS} satisfy
\begin{equation}  \label{eq:aisdninasdi} 
\norm{w^{t+1} -w^*}^2 \; \leq \; \norm{w^t -w^*}^2 -  \frac{f_{i_t}(w^t)^2}{\norm{ \nabla f_{i_t}(w^t)}^2} .
\end{equation}
Furthermore, if we assume that there exists $\cL>0$ such that the following 
\emph{expected smoothness} bound holds
\begin{equation}\label{eq:smoothi}
 \EE{i \sim p_i}{\frac{ f_i(w^t) ^2}{\norm{\nabla f_i(w^t)}^2}} \geq \frac{2}{ \cL} f(w^t),
\end{equation}
then
\begin{equation}\label{eq:convbsmooth}
\boxed{\min_{j=1,\ldots, k-1} \E{f(w^j) -f^*}  \;\; \leq  \;\;\frac{\cL }{2k} \norm{w^0-w^*}^2}.
\end{equation}
\end{theorem}

\begin{proof}
Substituting~\eqref{eq:polyak} into~\eqref{eq:polyakmot} gives~\eqref{eq:aisdninasdi}. 
Summing up both sides of~\eqref{eq:aisdninasdi} from $t=0,\ldots, k-1$, using telescopic cancellation and re-arranging gives
\begin{align}
\sum_{t=0}^{k-1}  \frac{ f_{i_t}(w^t)^2}{\norm{\nabla f_{i_t}(w^t)}^2} \leq \norm{w^0-w^*}^2. \label{eq:laststepansui}
\end{align}

Taking expectation and using~\eqref{eq:smoothi} in~\eqref{eq:laststepansui} gives
\begin{align}
\sum_{t=0}^{k-1} \frac{2}{\cL }\E{f(w^t)-f^*}  \overset{\eqref{eq:smoothi}}{\leq}\; \sum_{t=0}^{k-1}\E{ \frac{ f_{i_t}(w^t)^2}{\norm{  \nabla f_{i_t}(w^t)}^2}} \; \leq \;\norm{w^0-w^*}^2. \label{eq:laststepansu}
 \end{align}
Consequently
\begin{align*}
\min_{j=1,\ldots, k-1} \E{f(w^j) -f^*}  & \leq  \;\sum_{t=0}^{k-1} \frac{\E{f(w^t)-f^*}}{k} \; \leq \; \frac{1}{k} \frac{\cL }{2} \norm{w^0-w^*}^2.
\end{align*}
\end{proof}
Note that the Expected Smoothness bound in~\eqref{eq:smoothi} has been proven to be a consequence of standard smoothness and interpolation in Lemma~\ref{lem:specialsmooth}.

\section{Convergence using $t$--Smoothness}
\label{sec:theoryt-smooth}

Here we present an alternative convergence theorem for our variant of online SGD~\eqref{eq:ztupdateht} that is based on a smoothness type assumption.

\begin{theorem}[$t$--Smoothness] \label{theo:onlinesgdsmooth}
If there exists $L>0$ such that 
\begin{equation}\label{eq:hsmooht}
\EE{}{h_{t+1}(z^{t+1})} \; \leq \; \E{h_t(z^t)} +\EE{}{ \dotprod{\nabla h_t(z^t), z^{t+1} -z^t} }+ \frac{L}{2}\EE{}{\norm{z^{t+1}-z^t}^2}, 
\end{equation}
and  if $\gamma \leq \frac{1}{\sqrt{LG T}}$ then
\begin{eqnarray}
\min_{t=0,\ldots, T-1} \E{\norm{\nabla h_t(z^t)}^2}  & \leq &   3\sqrt{\frac{LG}{T} }h_0(z^0). 
\end{eqnarray}
Furthermore, 
\begin{equation}
T > \frac{9 LG}{\epsilon^2} h_0(z^0)^2 \; \implies \;\min_{t=0,\ldots, T-1} \E{\norm{\nabla h_t(z^t)}^2}  < \epsilon.
\end{equation}

If~\eqref{eq:hsmooht} holds and there exists $\mu_{PL}>0$ such that 
\begin{equation}\label{eq:hPL}
\norm{\nabla h_t(z^t)}^2 \geq 2 \mu_{PL} h_t(z^t),
\end{equation}
and if the stepsize satisfies
\begin{equation}\label{eq:hPLgammastep}
\gamma \; \leq \; \frac{\mu_{PL}}{LG},
\end{equation}
then the iterates converge linearly according to
\begin{equation}\label{eq:hconvPL}
\E{h_{t+1}(z^{t+1})} \; \leq \; (1-\mu_{PL}\gamma) )\E{h_t(z^t)}.
\end{equation}
\end{theorem}
 This smoothness assumption~\eqref{eq:hsmooht} is unusual in the literature since on the left hand side we have $h_{t+1}$, the auxiliary function at time $(t+1)$, and on the right we have $h_t.$ In Appendix~\eqref{sec:SPStheory} we show that~\eqref{eq:hsmooht} does holds for the \SP auxiliary functions when the underlying loss functions satisfy a property that is similar to self-concordancy. But first, the proof. 
\begin{proof}
This first part of the proof is adapted from~\cite{Khaled-nonconvex-2020}.  The second part of the proof that uses the PL condition is based on~\cite{SGDstruct}.
%
%
From~\eqref{eq:hsmooht}  and~\eqref{eq:ztupdateht} we have
\begin{eqnarray}
\EE{}{h_{t+1}(z^{t+1})} & \leq &\E{ h_t(z^t)} - \gamma \E{\dotprod{\nabla h_t(z^t),  \nabla h_{i_t,t}(z^t)}} + \frac{L\gamma^2}{2}\EE{}{\norm{ \nabla h_{i_t,t}(z^t)}^2}.
\end{eqnarray}
By the law of total expectation we have that
\[ \E{\dotprod{\nabla h_t(z^t),  \nabla h_{i_t,t}(z^t)}}  = \E{\EE{t}{\dotprod{\nabla h_t(z^t),  \nabla h_{i_t,t}(z^t)}}}
= \E{\dotprod{\nabla h_t(z^t), \EE{t}{ \nabla h_{i_t,t}(z^t)}}} \]
and since $\EE{t}{\nabla h_{i_t,t}(z^t)} = \nabla h_t(z^t)$ we have that
\begin{eqnarray}
\EE{}{h_{t+1}(z^{t+1})} & \leq & \E{h_t(z^t)} - \gamma \E{\norm{\nabla h_t(z^t)}^2 }+ \frac{L\gamma^2}{2}\EE{}{\norm{ \nabla h_{i_t,t}(z^t)}^2} \nonumber \\
& \overset{\eqref{eq:hweakgrow} }{\leq } &\E{h_t(z^t)} - \gamma\E{ \norm{\nabla h_t(z^t)}^2} + LG\gamma^2 \E{h_t(z^t)}. \label{eq:tempsnlo8joze}
\end{eqnarray}
Re-arranging~\eqref{eq:tempsnlo8joze} we have that
\begin{equation}\label{eq:tempsnlo8joze2}
 \gamma \E{\norm{\nabla h_t(z^t)}^2}  \; \leq \;  (1+ LG\gamma^2) \E{h_t(z^t)} - \E{h_{t+1}(z^{t+1})}.
\end{equation}
We now introduce a sequence of weights $w_{-1}, w_0, w_1, \ldots, w_k$ based on a technique developed by~\cite{Stich2019sgd}. Let $w_{-1}>0$ be arbitrary and fixed. We define the  remaining weight recurrently 
\[w_t \;=\; \frac{w_{t-1}}{1+LG\gamma^2} \; = \;  \frac{w_{-1}}{(1+ LG\gamma^2)^{t+1}}.\]
 Multiplying~\eqref{eq:tempsnlo8joze2} by $w_t/\gamma$ gives
\begin{eqnarray}
w_t \E{\norm{\nabla h_t(z^t)}^2} & \leq &  \frac{w_t}{\gamma}(1+ LG\gamma^2) \E{h_t(z^t)} - \frac{w_t}{\gamma}\E{h_{t+1}(z^{t+1})} \nonumber \\
& =& \frac{w_{t-1}}{\gamma} \E{h_t(z^t)} - \frac{w_t}{\gamma}\E{h_{t+1}(z^{t+1})} .\nonumber
\end{eqnarray}
Summing up both sides for $t=0, \ldots, T-1$ gives
\begin{eqnarray}
\sum_{i=0}^{T-1} w_t \E{\norm{\nabla h_t(z^t)}^2} & \leq &  
 \frac{w_{-1}}{\gamma} h_0(z^0) - \frac{w_{T-1}}{\gamma}\E{h_{T}(z^{T})}  \nonumber \\
 & \leq &  \frac{w_{-1}}{\gamma} h_0(z^0).
\label{eq:tempsnlo8joze3}
\end{eqnarray}
Now dividing both sides by $W_T \eqdef \sum_{t=0}^{T-1} w_t$ we have that
\begin{eqnarray}
\min_{t=0,\ldots, T-1} \E{\norm{\nabla h_t(z^t)}^2}  & \leq & \sum_{i=0}^{T-1} \frac{w_t}{W_T} \E{\norm{\nabla h_t(z^t)}^2}  \nonumber \\
& \overset{\eqref{eq:tempsnlo8joze3}}{\leq} & \frac{w_{-1}}{\gamma W_T}h_0(z^0). \label{eq:tempsnlo8joze4}
\end{eqnarray}
To conclude the proof, note that
\begin{equation}
W_T = \sum_{t=0}^{T-1}  w_t \geq T w_{T-1} = \frac{Tw_{-1}}{(1+ LG\gamma^2)^{T}} ,\label{eq:o9zjo9jze}
\end{equation}
where we used a standard  integral test. Inserting~\eqref{eq:o9zjo9jze} into~\eqref{eq:tempsnlo8joze4} gives
\begin{eqnarray}
\min_{t=0,\ldots, T-1} \E{\norm{\nabla h_t(z^t)}^2}  & \leq &   \frac{(1+LG\gamma^2)^{T}}{\gamma T} h_0(z^0). \label{eq:hconvsmoothproof}
\end{eqnarray}
Now let $\gamma \leq 1/\sqrt{ LGT}$. Using that $1+a \leq e^{a}$ we have that
\[(1+LG\gamma^2)^{T} \leq e^{ LGT\gamma^2} \leq e^{1} \leq 3.\]
Using this in~\eqref{eq:hconvsmoothproof} gives
\begin{eqnarray}
\min_{t=0,\ldots, T-1} \E{\norm{\nabla h_t(z^t)}^2}  & \leq &   3\sqrt{\frac{LG}{T} }h_0(z^0). \label{eq:hconvsmoothcmp1}
\end{eqnarray}
 We can now ensure that the left hand side is less than a given $\epsilon>0$ so long as 
 \[T \;> \;\frac{9 LG}{\epsilon^2} h_0(z^0)^2 \; \sim \; \cO\left(\frac{1}{\epsilon^2} \right).\]
%

Finally,  if we assume the PL condition~\eqref{eq:hPL} holds, then from~\eqref{eq:tempsnlo8joze} we have that
\begin{eqnarray}
\EE{}{h_{t+1}(z^{t+1})} &\leq &h_t(z^t) - \gamma \E{ \norm{\nabla h_t(z^t)}^2} + LG\gamma^2 \E{h_t(z^t)} \nonumber \\
& \overset{\eqref{eq:hPL} }{\leq } &(1-\gamma(2\mu_{PL}- LG\gamma) )\E{h_t(z^t)}.
\end{eqnarray}
Restricting $\gamma$ according to~\eqref{eq:hPLgammastep},
taking full expectation and unrolling the recurrence  gives~\eqref{eq:hconvPL}.
\end{proof}

\subsection{Sufficient Conditions for SPS}
\label{sec:SPStheoryusmooth}
Here we give sufficient condition on the $f_i$ functions  for Theorem~\ref{theo:onlinesgdsmooth} to hold where  $h_t$ is given by~\eqref{eq:proxyfun}. In particular, we need to establish when the auxiliary function $h_t$~\eqref{eq:proxyloss} is $t$--smooth.

\begin{lemma}\label{lem:scalesmoothsuff}
If there exists $L_2>0$ and $L_3>0$ such that
\begin{eqnarray}
\label{eq:spsscaledsmooth}
 \frac{|f_i(w) -f_i(w^*)|}{\norm{\nabla f_i(w)}^2} \norm{\nabla^2 f_i(w) }& \leq & L_2, \\
 \frac{(f_i(w) -f_i(w^*))^2}{\norm{\nabla f_i(w)}^3} \norm{ \nabla^3 f_i(w)}& \leq & L_3 \label{eq:spsscaledsmooth3}
\end{eqnarray}
then $h_t$ is $L$-$t$--smooth with $L =(1+ L_2)(1+2L_2) + L_3+\frac{n L_2}{\gamma}.$
\end{lemma}
Note that $L_2$ and $L_3$ are independent of the scaling of $f_i$. That is, if we multiply $f_i$ by a  constant, it has no affect on the bounds in~\eqref{eq:spsscaledsmooth3}. 
\begin{proof}
The proof has two step 1) we show that if the auxiliary function  $\phi(w) \eqdef h_w(w)$ is $\cL$--smooth then $h_t(w)$ satisfies~\eqref{eq:hsmooht}  after which 2) we show that  $\phi(w)$ is $\cL$--smooth.

\noindent{\bf Part I.}  If $\phi(w) \eqdef h_w(w)$ is smooth then $h_t(w)$ satisfies~\eqref{eq:hsmooht} .

Note that $\phi(w^t) = h_t(w^t)$, and that if $\phi(w)$ is $\cL$--smooth then
\begin{align}
h_{t+1}(w^{t+1})&=   \phi(w^{t+1}) \nonumber \\
& \leq  \phi(w^{t}) + \dotprod{\nabla \phi(w^t), w^{t+1} -w^t} + \frac{\cL}{2}\norm{w^{t+1}-w^t}^2 \nonumber\\
&=  h_t(w^t) +\dotprod{\nabla_w h_t(w)|_{w=w^t}, w^{t+1}-w^t} + \dotprod{\nabla_y h_y(w^t)|_{y=w^t}, w^{t+1}-w^t} \nonumber \\
& \qquad + \frac{\cL}{2}\norm{w^{t+1}-w^t}^2.\label{eq:o8ahlh8h8la31}
\end{align}
Consequently if we could show that there exists $C>0$ such that
\[\dotprod{\nabla_y h_y(w^t)|_{y=w^t}, w^{t+1}-w^t} \leq \frac{C}{2}\norm{w^{t+1}-w^t}^2,\]
then we could establish that~\eqref{eq:hsmooht} holds with $L = \cL +C$.
Let us show that this $C>0$ exists.
First note that
\begin{equation}
\nabla_y h_y(w^t) = \nabla_y \frac{1}{n}\sum_{i=1}^n \frac{1}{2}\frac{(f_i(w) -f_i(w^*))^2}{\norm{\nabla f_i(y)}^2} 
=- \frac{1}{n}\sum_{i=1}^n \frac{1}{2}\frac{(f_i(w) -f_i(w^*))^2}{\norm{\nabla f_i(y)}^4}  \nabla^2 f_i(y)  \nabla f_i(y).\label{eq:mssoz9jeo9ze}
\end{equation}
From~\eqref{eq:SPS} we have that
\begin{eqnarray}
\EE{t}{\norm{w^{t+1} - w^t}} &= &  \frac{\gamma }{  n} \sum_{i=1}^n \frac{|f_i(w^t)-f_i(w^*)|}{\norm{\nabla f_i(w^t)}} . \label{eq:spsEnormwt}
\end{eqnarray}


Now we can bound the gradient given in~\eqref{eq:mssoz9jeo9ze} as follows
\begin{eqnarray}
 \norm{\nabla_y h_y(w^t)|_{y=w^t}} &\leq& \frac{1}{n}\sum_{i=1}^n \frac{1}{2}\frac{(f_i(w^t) -f_i(w^*))^2}{\norm{\nabla f_i(w^t)}^4} \norm{ \nabla^2 f_i(w^t)  \nabla f_i(w^t)} \nonumber \\
 &\leq  & \frac{1}{n}\sum_{i=1}^n \frac{1}{2}\frac{|f_i(w) -f_i(w^*)|}{\norm{\nabla f_i(w^t)}^2} \norm{ \nabla^2 f_i(w^t)} \frac{|f_i(w^t) -f_i(w^*)|}{\norm{\nabla f_i(w^t)}}  \nonumber\\
 & \leq &  \frac{1}{ \gamma}\sum_{i=1}^n \frac{1}{2}\frac{|f_i(w) -f_i(w^*)|}{\norm{\nabla f_i(w^t)}^2} \norm{ \nabla^2 f_i(w^t)}\sum_{j=1}^n \frac{\gamma}{n} \frac{|f_j(w^t) -f_j(w^*)|}{\norm{\nabla f_j(w^t)}} \nonumber \\
 &\overset{\eqref{eq:spsEnormwt}}{=} &  \frac{1}{ \gamma}\sum_{i=1}^n \frac{1}{2}\frac{|f_i(w) -f_i(w^*)|}{\norm{\nabla f_i(w^t)}^2} \norm{ \nabla^2 f_i(w^t)} \EE{t}{\norm{w^{t+1}-w^t}}.
  \label{eq:tempscaleheaooze1}
\end{eqnarray}
Consequently
\begin{eqnarray*}
\EE{t}{\dotprod{\nabla_y h_y(w^t)|_{y=w^t}, w^{t+1}-w^t}} & \leq& \norm{\nabla_y h_y(w^t)|_{y=w^t}}\EE{t}{\norm{w^{t+1}-w^t}} \\
  &\overset{\eqref{eq:tempscaleheaooze1}}{ \leq } &\frac{1}{ \gamma}\sum_{i=1}^n \frac{1}{2}\frac{|f_i(w) -f_i(w^*)|}{\norm{\nabla f_i(w^t)}^2} \norm{ \nabla^2 f_i(w^t)}\EE{t}{\norm{w^{t+1}-w^t}}^2 \\
& \overset{\eqref{eq:spsscaledsmooth}}{\leq} &\frac{n L_2}{2\gamma} \EE{t}{\norm{w^{t+1}-w^t}}^2 \\
& \overset{\mbox{\small Jensen's Ineq.}}{\leq} &\frac{n L_2}{2\gamma} \EE{t}{\norm{w^{t+1}-w^t}^2}.
\end{eqnarray*}

Thus from the above and~\eqref{eq:o8ahlh8h8la31} we have that
\begin{eqnarray}
h_{t+1}(w^{t+1})&\overset{\eqref{eq:o8ahlh8h8la31} } \leq & h_t(w^t) +\dotprod{\nabla h_t(w^t), w^{t+1}-w^t} +\frac{1}{2}\left(\frac{n L_2}{\gamma} +\cL \right)\norm{w^{t+1}-w^t}^2.\label{eq:o8ahlh8h8la33}
\end{eqnarray}
 
 \noindent{\bf Part II}. Verifying that $\phi(w) = h_w(w)$ is an $\cL$--smooth function.

We will first verify  that $\phi_i(w) \eqdef \frac{1}{2}\frac{(f_i(w) -f_i(w^*))^2}{\norm{\nabla f_i(w)}^2}$  is a smooth function, then use that $\phi(w) = \frac{1}{n} \sum_{i=1}^n \phi_i(w). $  To do this, we will examine the Hessian of $\phi_i(w)$ and determine that it is bounded.
\begin{align}
\nabla \phi_i(w)  & =  \nabla \frac{1}{2}\frac{(f_i(w) -f_i(w^*))^2}{\norm{\nabla f_i(w)}^2} \nonumber \\
& =   \frac{f_i(w) -f_i(w^*)}{\norm{\nabla f_i(w)}^2} \nabla f_i(w)
-  \frac{(f_i(w) -f_i(w^*))}{\norm{\nabla f_i(w)}^4} \nabla^2 f_i(w)  \nabla f_i(w) \nonumber \\
& =   \frac{f_i(w) -f_i(w^*)}{\norm{\nabla f_i(w)}^2}\left(   \nabla f_i(w) - \frac{f_i(w) -f_i(w^*)}{\norm{\nabla f_i(w)}^2} \nabla^2 f_i(w)  \nabla f_i(w)\right).
\end{align}
The second derivative has two terms 
\[\phi''_i(w) =  \nabla^2 \frac{1}{2}\frac{(f_i(w) -f_i(w^*))^2}{\norm{\nabla f_i(w)}^2}  = I + II\]
where
\begin{eqnarray}
I & =&\frac{1}{\norm{\nabla f_i(w)}^2} \left(\nabla f_i(w)- \frac{f_i(w) -f_i(w^*)}{\norm{\nabla f_i(w)}^2} \nabla^2 f_i(w)  \nabla f_i(w) \right)\left(   \nabla f_i(w) - \frac{f_i(w) -f_i(w^*)}{\norm{\nabla f_i(w)}^2} \nabla^2 f_i(w)  \nabla f_i(w)\right)^\top \nonumber \\
&= & \frac{1}{\norm{\nabla f_i(w)}^2} \left(\mI- \frac{f_i(w) -f_i(w^*)}{\norm{\nabla f_i(w)}^2} \nabla^2 f_i(w)  \right)  \nabla f_i(w)\nabla f_i(w)^\top  \left(\mI- \frac{f_i(w) -f_i(w^*)}{\norm{\nabla f_i(w)}^2} \nabla^2 f_i(w)  \right)  \label{eq:Itempso9j4o9sj4}
\end{eqnarray}
and
\begin{eqnarray}
II &= & \frac{f_i(w) -f_i(w^*)}{\norm{\nabla f_i(w)}^2}\left(   \nabla^2 f_i(w) - \frac{1}{\norm{\nabla f_i(w)}^2} \nabla^2 f_i(w)  \nabla f_i(w)\nabla f_i(w)^\top \right. \nonumber \\
& &
\quad + \frac{f_i(w) -f_i(w^*)}{\norm{\nabla f_i(w)}^4} \nabla^2 f_i(w)  \nabla f_i(w) \nabla f_i(w)^\top \nabla^2 f_i(w) \nonumber \\
&&  \quad \left.- \frac{f_i(w) -f_i(w^*)}{\norm{\nabla f_i(w)}^2} (\nabla^2 f_i(w) ^2 +   \nabla^3 f_i(w) \nabla f_i(w) )\right)\nonumber \\
&= &\frac{f_i(w) -f_i(w^*)}{\norm{\nabla f_i(w)}^2}\left(   \nabla^2 f_i(w)\left(\mI - \frac{1}{\norm{\nabla f_i(w)}^2}   \nabla f_i(w)\nabla f_i(w)^\top\right) \right. \nonumber \\
& & 
\quad + \frac{f_i(w) -f_i(w^*)}{\norm{\nabla f_i(w)}^2} \nabla^2 f_i(w) \left(\frac{1}{\norm{\nabla f_i(w)}^2} \nabla f_i(w) \nabla f_i(w)^\top  -\mI \right)\nabla^2 f_i(w)\nonumber \\
&& \quad \left. - \frac{f_i(w) -f_i(w^*)}{\norm{\nabla f_i(w)}^2} \nabla^3 f_i(w) \nabla f_i(w) ) \right)\nonumber \\
&= & \frac{f_i(w) -f_i(w^*)}{\norm{\nabla f_i(w)}^2}   \nabla^2 f_i(w)\left(\mI - \frac{1}{\norm{\nabla f_i(w)}^2}   \nabla f_i(w)\nabla f_i(w)^\top\right)\left( \mI- \frac{f_i(w) -f_i(w^*)}{\norm{\nabla f_i(w)}^2} \nabla^2 f_i(w)\right) \nonumber \\
& & -\frac{(f_i(w) -f_i(w^*))^2}{\norm{\nabla f_i(w)}^4}  \nabla^3 f_i(w) \nabla f_i(w)   \label{eq:IItempsnlj3s4}
\end{eqnarray}

For $I$ we have from~\eqref{eq:spsscaledsmooth} that
\begin{eqnarray}
 \norm{I} &\leq& \norm{\mI- \frac{f_i(w) -f_i(w^*)}{\norm{\nabla f_i(w)}^2} \nabla^2 f_i(w)} \leq 
 1+L_2.\label{eq:twmppozj9oejz}
\end{eqnarray}

Furthermore, note that
\begin{equation}\label{eq:oz9je9zj9z4z4}
\norm{\left(\frac{1}{\norm{\nabla f_i(w)}^2} \nabla f_i(w) \nabla f_i(w)^\top  -\mI \right)} = 1 
\end{equation} 
since  $\frac{1}{\norm{\nabla f_i(w)}^2} \nabla f_i(w) \nabla f_i(w)^\top $ is a projection matrix onto $\Range{\nabla f_i(w)}.$

Thus finally we have that
\begin{eqnarray}
\norm{I + II} & \leq & \norm{I} + \norm{II} \nonumber \\
& \overset{\eqref{eq:Itempso9j4o9sj4}+\eqref{eq:IItempsnlj3s4}}{ \leq} &(1+ L_2)^2 +  L_2(1+ L_2) + L_3
\end{eqnarray}
\end{proof}

Using Lemma~\ref{lem:scalesmoothsuff} and Theorem~\ref{theo:onlinesgdsmooth} we can establish the following convergence theorem for the \SP method.

\begin{theorem}[$t$--Smoothness] \label{theo:onlinesgdsmoothSPS}
Suppose that there exists $L_2, L_3$ such that~\eqref{eq:spsscaledsmooth} and~\eqref{eq:spsscaledsmooth3} holds.
Consider the iterates $w^t$ of SPS~\eqref{eq:SPS} and let $h_t$ be given by~\eqref{eq:proxyfun}.
If
\begin{equation}\label{eq:gammaSPSsmooth}
 \gamma \; \leq\;  \frac{n L_2}{2\ell}\left(\sqrt{ 1 + \frac{4\ell }{T( n L_2)^2}} -1 \right)
\end{equation}
where $ \ell \eqdef (1+ L_2)(1+2L_2) + L_3$
then
\begin{eqnarray}
\min_{t=0,\ldots, T-1} \E{\norm{ \frac{1}{n}\sum_{i=1}^n\frac{f_i(w^t) -f_i(w^*)}{\norm{\nabla f_i(w^t)}^2} \nabla f_i(w^t)}^2}  & \leq &   3\sqrt{\frac{L}{T} }h_0(z^0). 
\end{eqnarray}

\end{theorem}
\begin{proof}
Consider the statement of Theorem~\ref{theo:onlinesgdsmooth}. First note that $G = 1$ from Lemma~\ref{lem:weakgrowthsps}. According to Lemma~\ref{lem:scalesmoothsuff} we have that $h_t$ 
satisfies~\eqref{eq:hsmooht} with
$$L =(1+ L_2)(1+2L_2) + L_3+\frac{n L_2}{\gamma}.$$
Let $\ell \eqdef (1+ L_2)(1+2L_2) + L_3$
 Furthermore from  from Theorem~\ref{theo:onlinesgdsmooth} we need $\gamma < \frac{1}{\sqrt{LT}}$ in other words
 \begin{align*}
\gamma & < \frac{1}{\sqrt{T}\sqrt{\ell+\frac{n L_2}{\gamma}}} \quad
\Leftrightarrow  \quad
\gamma \; < \frac{\sqrt{\gamma}}{\sqrt{T}\sqrt{\ell \gamma+n L_2}} \quad
\Leftrightarrow  \quad 
\sqrt{\gamma}\; < \frac{1}{\sqrt{T}\sqrt{\ell\gamma +n L_2}}  \\
 \quad
& \Leftrightarrow  \quad
\sqrt{\gamma} \sqrt{T}\sqrt{\ell\gamma +n L_2}  <  1 \quad
\Leftrightarrow  \quad 
 \gamma^2 T \ell +\gamma T n L_2 -1 <  0.
\end{align*} 
The roots of the above quadratic are given by
$$ \gamma = -\frac{T n L_2}{2T \ell } \pm \frac{\sqrt{(T n L_2)^2 +4T \ell }}{2T \ell}  $$
Thus 
$$ \gamma < \frac{\sqrt{(T n L_2)^2 +4T \ell } -T n L_2 }{2T \ell}  \quad
\Leftrightarrow  \quad 
\gamma \leq \frac{\sqrt{(T n L_2)^2 +4T \ell } -T n L_2 }{2T \ell} $$
$$
\quad
\Leftrightarrow  \quad 
\gamma \leq \frac{1}{\sqrt{T}}\frac{\sqrt{ T( n L_2)^2 +4 \ell } -\sqrt{T} n L_2 }{2 \ell}
\quad
\Leftrightarrow  \quad  \gamma \leq \frac{\sqrt{T} n L_2}{\sqrt{T}}\frac{\sqrt{ 1 +4 \frac{\ell }{T( n L_2)^2}} -1 }{2\ell},$$
which after cancellation is equal to~\eqref{eq:gammaSPSsmooth}

\end{proof}

\subsubsection{Examples of scaled smoothness}
Now we provide a class of functions for which our sufficient conditions given in Lemma~\ref{lem:scalesmoothsuff} hold.
\begin{example}[Monomials]\label{exe:spsmon}
Let $\phi_i(t) = a_i (t-b_i)^{2r}$ where $ r,a_i>0$  and $b_i \in \R$ for $i=1,\ldots, n.$ 
It follows that
\begin{eqnarray}
 \frac{|\phi_i(t) )|}{\norm{\phi_i'(t)}^2} \norm{\phi''_i(t) }& \leq & 1, \\
 \frac{\phi_i(t) ^2}{\norm{\phi_i'(t)}^3} \norm{ \phi'''_i(t)}& \leq &1+\frac{1}{2r^2}.
 \end{eqnarray}

\end{example}
\begin{proof}
Verifying the conditions in Lemma~\ref{lem:scalesmoothsuff} we have that
\begin{eqnarray*}
 \frac{|\phi_i(t) |}{\norm{\phi_i'(t)}^2} \norm{\phi_i''(t) }& \leq & \frac{ 2r(2r-1)(t-b_i)^{2r} (t-b_i)^{2r-2}}{ 4r^2 (t-b_i)^{4r-2}  } \\
 &= & \frac{ (4r^2-2r)}{4r^2 } \; =\;  1- \frac{1}{2r} \;\leq \;1.
 \end{eqnarray*}
Furthermore
\begin{eqnarray*}
 \frac{\phi_i(t) ^2}{\norm{\phi_i'(t)}^3} \norm{ \phi'''_i(t)}& \leq & \frac{2r(2r-1)(2r-2) (t-b_i)^{4r}(t-b_i)^{2r-3} }{8r^3(t-b_i)^{6r-3} }  \\
 &= & \frac{2r(2r-1)(2r-2)}{8r^3} \;= \;\frac{(2r-1)(2r-2) }{4r^2}  \leq 1+\frac{1}{2r^2}.
 \end{eqnarray*}
\end{proof}

Note that for $r<1$ we have that $\phi_i(t) = a_i (t-b_i)^{2r}$ is a non-convex function. In the following example we generalize the above example to a non-convex generalized linear model.

\begin{example}[Generalized Linear model]
Let $x_1,\ldots, x_n \in \R^d$ be $n$ given data points.
Let $f_i(w) = \phi_i(x_i^\top w)$ where $\phi_i: \R \mapsto \R$ is a $C^3$ real valued loss function. Furthermore, suppose that there exists 
a hyperplane that separates the datapoints. In other words, the 
problem is over-parametrized so that the solution $w^*\in \R^d$ is such that $\phi_i(x_i^\top w^*) =0.$ It follows that
\begin{eqnarray}
 \frac{|f_i(w) -f_i(w^*)|}{\norm{\nabla f_i(w)}^2} \norm{\nabla^2 f_i(w) }& \leq & \max_{t}\frac{\phi_i(t) |\phi_i''(t)|}{|\phi_i'(t)|^2  } , \\
 \frac{(f_i(w) -f_i(w^*))^2}{\norm{\nabla f_i(w)}^3} \norm{ \nabla^3 f_i(w)}& \leq & \max_{t}\frac{\phi_i(t)^2|\phi_i'''(t)| }{|\phi_i'(t)|^3 } 
 \end{eqnarray}
 Thus if $\phi_i(t) = a_i(t-b_i)^{2r}$ where $r>0$   then according to Example~\ref{exe:spsmon} we have that
 \begin{eqnarray}
 \frac{|f_i(w) -f_i(w^*)|}{\norm{\nabla f_i(w)}^2} \norm{\nabla^2 f_i(w) }& \leq &1 , \\
 \frac{(f_i(w) -f_i(w^*))^2}{\norm{\nabla f_i(w)}^3} \norm{ \nabla^3 f_i(w)}& \leq &1+\frac{1}{2r^2}.
 \end{eqnarray} 
  Consequently Theorem~\ref{theo:onlinesgdsmooth} holds with $L_2 =1$ and $L_3 = 1+\frac{1}{2r^2}$.
\end{example}
\begin{proof}
Indeed since
\begin{eqnarray*}
\nabla f_i(w) & =& x_i \phi'(x_i^\top w) \\
\nabla^2 f_i(w) & =& x_i x_i^\top \phi''(x_i^\top w)  \\
\nabla^3 f_i(w) & =& x_i \otimes x_i  \otimes x_i \phi'''(x_i^\top w) 
\end{eqnarray*}
Consequently
\begin{eqnarray*}
 \frac{|f_i(w) -f_i(w^*)|}{\norm{\nabla f_i(w)}^2} \norm{\nabla^2 f_i(w) }& \leq & \frac{\phi(x_i^\top w) }{|\phi'(x_i^\top w)|^2  } |\phi''(x_i^\top w)|, \\
 \frac{(f_i(w) -f_i(w^*))^2}{\norm{\nabla f_i(w)}^3} \norm{ \nabla^3 f_i(w)}& \leq & \frac{\phi(x_i^\top w)^2 }{|\phi'(x_i^\top w)|^3 } |\phi'''(x_i^\top w)|
 \end{eqnarray*}
 The result now follows by taking the max over the arguments on the left hand side.
\end{proof}

\subsection{Sufficient conditions on TAPS}

Here we explore sufficient conditions 
 for the smoothness assumption in Theorem~\ref{theo:onlinesgdsmooth}  to hold for the  TAPS method given in Algorithm~\ref{alg:TAPS}.

First we provide a sufficient condition for the $t$--smoothness assumption in Theorem~\ref{theo:onlinesgdsmooth}  to hold.


\begin{lemma}\label{lem:scalesmoothsuffmotasps}
If $f_i$ is $L_2$--scaled smooth, that is
\begin{eqnarray}
\label{eq:spsscaledsmoothtasps}
 \frac{|f_i(w^t) -\alpha_i^t|}{1+\norm{\nabla f_i(w^t)}^2} \norm{\nabla^2 f_i(w^t) }& \leq & L_2, \\
 \frac{(f_i(w^t) -\alpha_i^t)^2}{(\norm{\nabla f_i(w^t)}^2+1)^2} \norm{ \nabla^3 f_i(w^t) \nabla f_i(w) }& \leq & L_3 \label{eq:spsscaledsmooth3tasps}
\end{eqnarray}
then $h_t$ is $L$-$t$--smooth with $L =(3+L_2)(1+2L_2) + 1 + L_3+\frac{n L_2}{\gamma}.$
\end{lemma}
\begin{proof}
The proof has two step 1) we show that if the auxiliary function  $\phi(z) \eqdef h_z(z)$ is $\cL$--smooth then $h_t(z)$ satisfies~\eqref{eq:hsmooht} and then  2)  show that  $\phi(z)$ is smooth.

\noindent{\bf Part I.}  If $\phi(z) \eqdef h_z(z)$ is smooth then $h_t(z)$ satisfies~\eqref{eq:hsmooht} .

Note that $\phi(z^t) = h_t(z^t)$, and that if $\phi(z)$ is $\cL$--smooth then
\begin{align}
h_{t+1}(z^{t+1})&=   \phi(z^{t+1}) \nonumber \\
& \leq  \phi(z^{t}) + \dotprod{\nabla \phi(z^t), z^{t+1} -z^t} + \frac{\cL}{2}\norm{z^{t+1}-z^t}^2 \nonumber\\
&=  h_t(z^t) +\dotprod{\nabla_z h_t(z)|_{z=z^t}, z^{t+1}-z^t} + \frac{\cL}{2}\norm{z^{t+1}-z^t}^2+  \dotprod{\nabla_y h_y(z^t)|_{y=w^t}, w^{t+1}-w^t} .\label{eq:o8ahlh8h8la343000}
\end{align}
Using the SGD interpretation os \texttt{TAPS} given in Section~\ref{sec:SGDviewpoint} we have that $z^{t+1} = z^t - \gamma \nabla h_{i,t}(z^t)$  where $\E{ \nabla h_{i,t}(z^t)}= \nabla h_t(z^t),$
thus the above gives
\begin{align}
\EE{t}{h_{t+1}(z^{t+1})} &\leq   h_t(z^t) -\gamma \norm{\nabla h_t(z^t)}^2 + \frac{\gamma^2 \cL}{2}\EE{t}{\norm{\nabla h_t(z^t)}^2}+  \EE{t}{\dotprod{\nabla_y h_y(z^t)|_{y=w^t}, w^{t+1}-w^t}} .\label{eq:o8ahlh8h8la343}
\end{align}

Consequently if we could show that there exists $C>0$ such that
\[\dotprod{\nabla_y h_y(z^t)|_{y=w^t}, w^{t+1}-w^t} \leq \frac{C}{2}\norm{z^{t+1}-z^t}^2,\]
then we could establish that~\eqref{eq:hsmooht} holds with $L = \cL +C$. For this, we have that

\begin{equation}
\left. \nabla_y h_y(z^t)\right|_{y=w^t} = \nabla_y \frac{1}{n+1}\sum_{i=1}^n \frac{1}{2}\frac{(f_i(w^t) -\alpha_i)^2}{\norm{\nabla f_i(w^t)}^2+1} 
=- \frac{1}{n+1}\sum_{i=1}^n \frac{1}{2}\frac{(f_i(w) -\alpha_i)^2}{(\norm{\nabla f_i(w^t)}^2 + 1)^2}  \nabla^2 f_i(w^t)  \nabla f_i(w^t).\label{eq:mssoz9jeo9ze2}
\end{equation}

Furthermore, 
from the SGD viewpoint in Section~\ref{sec:SGDviewpoint} and~\eqref{eq:updatesgd} 
we have that

\begin{eqnarray}
\EE{t}{\norm{w^{t+1} - w^t}} &= &  \frac{\gamma }{  n} \sum_{i=1}^n \frac{|f_i(w^t)-\alpha_i^t)|}{\norm{\nabla f_i(w^t)}^2+1}\norm{\nabla f_i(w^t)} \label{eq:s4jpsk04ksk4} \\
\EE{t}{\norm{z^{t+1} - z^t} ^2}&= & \EE{t}{\norm{w^{t+1} - w^t}^2} + \EE{t}{\norm{\alpha^{t+1}-\alpha^t}^2 }\nonumber \\
&\overset{\eqref{eq:updatesgd}+\eqref{eq:updatessgdavalpha}}{=} & 
 \frac{\gamma^2 }{ n+1} \sum_{i=1}^n \frac{(f_i(w^t)-\alpha_i)^2}{(\norm{\nabla f_i(w^t)}^2+1)^2}\norm{\nabla f_i(w^t)}^2 \nonumber \\
 & & \quad  +
 \frac{\gamma^2 }{ n+1}\left( \sum_{i=1}^n \frac{(f_i(w^t)-\alpha_i)^2}{(\norm{\nabla f_i(w^t)}^2+1)^2}  + (\overline{\alpha}^t - \tau)^2\right)\nonumber \\
 &= &  \frac{\gamma^2 }{ n+1} \left(\sum_{i=1}^n \frac{(f_i(w^t)-\alpha_i)^2}{\norm{\nabla f_i(w^t)}^2+1}+ (\overline{\alpha}^t - \tau)^2 \right). \label{eq:spsEnormwt22}
\end{eqnarray}


Now we can bound the gradient given in~\eqref{eq:mssoz9jeo9ze} as follows
\begin{eqnarray}
 \norm{\nabla_y h_y(w^t, \alpha^t)|_{y=w^t}} &\leq& \frac{1}{n}\sum_{i=1}^n \frac{1}{2}\frac{(f_i(w^t) -\alpha_i^t)^2}{(\norm{\nabla f_i(w^t)}^2+1)^2} \norm{ \nabla^2 f_i(w^t)  \nabla f_i(w^t)} \nonumber \\
 & \leq & \frac{1}{n}\sum_{i=1}^n \frac{1}{2}\frac{|f_i(w) -\alpha_i^t|}{\norm{\nabla f_i(w^t)}^2+1}  \norm{ \nabla^2 f_i(w^t)}\norm{  \nabla f_i(w^t)} \frac{|f_i(w^t) -\alpha_i^t|}{\norm{\nabla f_i(w^t)}^2 +1}  \nonumber\\
 & \leq &  \frac{1}{ \gamma}\sum_{i=1}^n \frac{1}{2}\frac{|f_i(w) -\alpha_i^t|}{\norm{\nabla f_i(w^t)}^2+1}  \norm{ \nabla^2 f_i(w^t)  } \sum_{j=1}^n \frac{\gamma}{n} \frac{|f_j(w^t) -\alpha_j^t|}{\norm{\nabla f_j(w^t)}^2+1} \norm{  \nabla f_i(w^t)}\nonumber \\
 &\overset{\eqref{eq:s4jpsk04ksk4}}{=} &  \frac{1}{ \gamma}\sum_{i=1}^n \frac{1}{2}\frac{|f_i(w) -\alpha_i^t|}{\norm{\nabla f_i(w^t)}^2+1} \norm{ \nabla^2 f_i(w^t)} \EE{t}{\norm{w^{t+1}-w^t}}.
  \label{eq:tempscaleheaooze}
\end{eqnarray}
Consequently
\begin{eqnarray*}
\EE{t}{\dotprod{\nabla_y h_y(w^t, \alpha^t)|_{y=w^t}, w^{t+1}-w^t}} & \leq& \norm{\nabla_y h_y(w^t, \alpha^t)|_{y=w^t}}\EE{t}{\norm{w^{t+1}-w^t}} \\
  &\overset{\eqref{eq:tempscaleheaooze}}{ \leq } &\frac{1}{ \gamma}\sum_{i=1}^n \frac{1}{2}\frac{|f_i(w) -\alpha_i^t|}{\norm{\nabla f_i(w^t)}^2+1} \norm{ \nabla^2 f_i(w^t)}\EE{t}{\norm{w^{t+1}-w^t}}^2 \\
& \overset{\eqref{eq:spsscaledsmoothtasps}}{\leq} &\frac{n L_2}{2\gamma} \EE{t}{\norm{w^{t+1}-w^t}}^2 \\
& \overset{\mbox{\small Jensen's Ineq.}}{\leq} &\frac{n L_2}{2\gamma} \EE{t}{\norm{w^{t+1}-w^t}^2} \\
& \leq & \frac{n L_2}{2\gamma} \EE{t}{\norm{z^{t+1}-z^t}^2}.
\end{eqnarray*}

Thus from the above and~\eqref{eq:o8ahlh8h8la343000} we have that
\begin{eqnarray}
h_{t+1}(z^{t+1})&\overset{\eqref{eq:o8ahlh8h8la343000} } \leq & h_t(z^t) +\dotprod{\nabla h_t(z^t), z^{t+1}-z^t} +\frac{1}{2}\left(\frac{n L_2}{\gamma} +\cL \right)\norm{z^{t+1}-z^t}^2.\label{eq:o8ahlh8h8la33}
\end{eqnarray}

 \noindent{\bf Part II}. Verifying that $\phi(z) = h_z(z)$ is an $\cL$--smooth function.
To this end, note  that
\begin{equation}
\nabla^2 \phi(w,\alpha) \; = \; 
\begin{bmatrix}
\nabla_w^2 \phi(w,\alpha)  & \nabla_{w,\alpha}^2 \phi(w,\alpha)  \\
\nabla_{w,\alpha}^2 \phi(w,\alpha)^\top  & \nabla_{\alpha}^2 \phi(w,\alpha) .
\end{bmatrix}
\end{equation}
To show that $\norm{\nabla^2 \phi(w,\alpha)}$ is bounded we will use that
\begin{eqnarray}
\norm{\nabla^2 \phi(w,\alpha)} & \leq &  \norm{\nabla_w^2 \phi(w,\alpha)} +2\norm{ \nabla_{w,\alpha}^2 \phi(w,\alpha)}  + \norm{\nabla_{\alpha}^2 \phi(w,\alpha)  },  \label{eq:matboundrtempmso9j}
\end{eqnarray}
which relies on Lemma~\ref{lem:matrixblockbnd} proven in the appendix.

We will first verify  that $\phi_i(w,\alpha) \eqdef \frac{1}{2}\frac{(f_i(w) -\alpha_i)^2}{\norm{\nabla f_i(w)}^2+1}$  is a smooth function, then use that $\phi(w,\alpha) = \frac{1}{n} \sum_{i=1}^n \phi_i(w,\alpha_i). $  To do this, we will examine the Hessian of $\phi_i(w,\alpha_i)$ and determine that it is bounded.

\begin{eqnarray}
\nabla_w \phi_i(w,\alpha_i) &=& \nabla_w \frac{1}{2}\frac{(f_i(w) -\alpha_i)^2}{\norm{\nabla f_i(w)}^2+1} \nonumber \\
& =&   \frac{f_i(w) -\alpha_i}{\norm{\nabla f_i(w)}^2+1} \nabla f_i(w)
-  \frac{(f_i(w) -\alpha_i)}{(\norm{\nabla f_i(w)}^2+1)^2} \nabla^2 f_i(w)  \nabla f_i(w) \nonumber \\
& =&   \frac{f_i(w) -\alpha_i}{\norm{\nabla f_i(w)}^2+1}\left(   \nabla f_i(w) - \frac{f_i(w) -\alpha_i}{\norm{\nabla f_i(w)}^2+1} \nabla^2 f_i(w)  \nabla f_i(w)\right). \label{eq:tempmseos9jpa}
\end{eqnarray}
The second derivative has two terms 
\[\nabla_w^2 \phi_i(w,\alpha_i)  \; =\; \nabla^2_w \frac{1}{2}\frac{(f_i(w) -\alpha_i)^2}{\norm{\nabla f_i(w)}^2+1}  = I + II\]
where

\begin{eqnarray}
I & =&\frac{1}{\norm{\nabla f_i(w)}^2+1} \left(\nabla f_i(w)- \frac{f_i(w) -\alpha_i}{\norm{\nabla f_i(w)}^2+1} \nabla^2 f_i(w)  \nabla f_i(w) \right)\left(   \nabla f_i(w) - \frac{f_i(w) -\alpha_i}{\norm{\nabla f_i(w)}^2+1} \nabla^2 f_i(w)  \nabla f_i(w)\right)^\top \nonumber \\
&= & \frac{1}{\norm{\nabla f_i(w)}^2+1}  \left(\mI- \frac{f_i(w) -\alpha_i}{\norm{\nabla f_i(w)}^2+1} \nabla^2 f_i(w)  \right) \nabla f_i(w)\nabla f_i(w)^\top   \left(\mI- \frac{f_i(w) -\alpha_i}{\norm{\nabla f_i(w)}^2+1} \nabla^2 f_i(w)  \right) \label{eq:Itempso9j4o9sj422}
\end{eqnarray}

and
\begin{eqnarray}
II &= & \frac{f_i(w) -\alpha_i}{\norm{\nabla f_i(w)}^2+1}\left(   \nabla^2 f_i(w) - \frac{1}{\norm{\nabla f_i(w)}^2+1} \nabla^2 f_i(w)  \nabla f_i(w)\nabla f_i(w)^\top \right. \nonumber \\
& & \quad \left.
+ \frac{f_i(w) -\alpha_i}{(\norm{\nabla f_i(w)}^2+1)^2} \nabla^2 f_i(w)  \nabla f_i(w) \nabla f_i(w)^\top \nabla^2 f_i(w)
- \frac{f_i(w) -\alpha_i}{\norm{\nabla f_i(w)}^2+1} (\nabla^2 f_i(w) ^2 +   \nabla^3 f_i(w) \nabla f_i(w) )\right)\nonumber \\
&= &\frac{f_i(w) -\alpha_i}{\norm{\nabla f_i(w)}^2+1}\left(   \nabla^2 f_i(w)\left(\mI - \frac{1}{\norm{\nabla f_i(w)}^2+1}   \nabla f_i(w)\nabla f_i(w)^\top\right) \right. \nonumber \\
& &\quad  \left.
+ \frac{f_i(w) -\alpha_i}{\norm{\nabla f_i(w)}^2+1} \nabla^2 f_i(w) \left(\frac{1}{\norm{\nabla f_i(w)}^2+1} \nabla f_i(w) \nabla f_i(w)^\top  -\mI \right)\nabla^2 f_i(w)  \right.\nonumber\\
&& \quad \left.- \frac{f_i(w) -\alpha_i}{\norm{\nabla f_i(w)}^2+1} \nabla^3 f_i(w) \nabla f_i(w) ) \right)\nonumber \\
&= & \frac{f_i(w) -\alpha_i}{\norm{\nabla f_i(w)}^2+1}   \nabla^2 f_i(w)\left(\mI - \frac{1}{\norm{\nabla f_i(w)}^2+1}   \nabla f_i(w)\nabla f_i(w)^\top\right)\left( \mI- \frac{f_i(w) -\alpha_i}{\norm{\nabla f_i(w)}^2+1} \nabla^2 f_i(w)\right) \nonumber \\
& & \quad -\frac{(f_i(w) -\alpha_i)^2}{(\norm{\nabla f_i(w)}^2+1)^2}  \nabla^3 f_i(w) \nabla f_i(w)   \label{eq:IItempsnlj3s422}
\end{eqnarray}

For $I$ we have that
\begin{eqnarray}
 \norm{I} &\overset{s}{ \leq}& \norm{\mI- \frac{f_i(w) -\alpha_i}{\norm{\nabla f_i(w)}^2+1} \nabla^2 f_i(w)} \leq 
 1+L_2.\label{eq:twmppozj9oejz2}
\end{eqnarray}

Furthermore, note that
\begin{equation}\label{eq:oz9je9zj9z4z42}
\norm{\left(\frac{1}{\norm{\nabla f_i(w)}^2+1} \nabla f_i(w) \nabla f_i(w)^\top  -\mI \right)} = 1 
\end{equation} 
since  $\frac{1}{\norm{\nabla f_i(w)}^2+1} \nabla f_i(w) \nabla f_i(w)^\top $ is a projection matrix onto $\Range{\nabla f_i(w)}.$

Thus  we have that
\begin{eqnarray}
\norm{\nabla_w^2 \phi_i(w,\alpha_i) } = \norm{I + II} & \leq & \norm{I} + \norm{II} \nonumber \\
& \overset{\eqref{eq:oz9je9zj9z4z42}+\eqref{eq:twmppozj9oejz2}}{ \leq} &(1+ L_2)^2 +  L_2(1+ L_2) + L_3 \label{eq:IplusIItasps}
\end{eqnarray}

\begin{eqnarray}
\nabla_{\alpha_i}^2 \phi(w,\alpha_i) & =& \frac{1}{2}\nabla_{\alpha_i}^2 \frac{(f_i(w) -\alpha_i)^2}{\norm{\nabla f_i(w)}^2+1} \nonumber \\
& =& -\nabla_{\alpha_i} \frac{(f_i(w) -\alpha_i)}{\norm{\nabla f_i(w)}^2+1} 
\;= \;  \frac{1 }{\norm{\nabla f_i(w)}^2+1} \leq 1. \label{eq:ahlpateiswdeibnsai}
\end{eqnarray}

\begin{eqnarray}
\nabla_{\alpha_i,w}^2 \phi(w,\alpha_i) &\overset{\eqref{eq:tempmseos9jpa}}{ =}&  \nabla_{\alpha_i}\frac{f_i(w) -\alpha_i}{\norm{\nabla f_i(w)}^2+1}\left(   \nabla f_i(w) - \frac{f_i(w) -\alpha_i}{\norm{\nabla f_i(w)}^2+1} \nabla^2 f_i(w)  \nabla f_i(w)\right) \nonumber \\
& =& -\frac{1}{\norm{\nabla f_i(w)}^2+1}\left(   \nabla f_i(w) - \frac{f_i(w) -\alpha_i}{\norm{\nabla f_i(w)}^2+1} \nabla^2 f_i(w)  \nabla f_i(w)\right) \nonumber\\
& &\frac{f_i(w) -\alpha_i}{\norm{\nabla f_i(w)}^2+1}\frac{1}{\norm{\nabla f_i(w)}^2+1} \nabla^2 f_i(w)  \nabla f_i(w) \nonumber \\
& = & \frac{1}{\norm{\nabla f_i(w)}^2+1}\left(   2\frac{f_i(w) -\alpha_i}{\norm{\nabla f_i(w)}^2+1} \nabla^2 f_i(w)  \nabla f_i(w) -  \nabla f_i(w)\right)\label{eq:crossderivninasd}
\end{eqnarray}
Thus 
\begin{equation}
\norm{\nabla_{\alpha_i,w}^2 \phi(w,\alpha_i) } \; \overset{\eqref{eq:spsscaledsmoothtasps}}{ \leq} \; 2L_2 +1
\end{equation}
Finally,
 by~\eqref{eq:matboundrtempmso9j} we have that
 \begin{eqnarray}
 \norm{\nabla^2 \phi(w,\alpha)} & \leq &  \norm{\nabla_w^2 \phi(w,\alpha)} +2\norm{ \nabla_{w,\alpha}^2 \phi(w,\alpha)}  + \norm{\nabla_{\alpha}^2 \phi(w,\alpha)  }  \nonumber \\
 & \leq & \frac{1}{n} \sum_{i=1}^n \left( \norm{\nabla_w^2 \phi_i(w,\alpha_i)} +2\norm{ \nabla_{w,\alpha_i}^2 \phi(w,\alpha_i)}  + \norm{\nabla_{\alpha_i}^2 \phi(w,\alpha_i)  } \right) \nonumber \\
 & \overset{\eqref{eq:IplusIItasps},~\eqref{eq:ahlpateiswdeibnsai},~\eqref{eq:crossderivninasd}}{\leq} & (1+ L_2)^2 +  L_2(1+ L_2) + L_3+  2( 1+2L_2)+1  \nonumber \\
 &= & (3+L_2)(1+2L_2) + 1 + L_3.
 \end{eqnarray}
\end{proof}

\section{Convex Classification: Additional Experiments}
\label{asec:convex}

\subsection{Grid search and Parameter Sensitivity}
\label{asec:parameterset}

To investigate how sensitve \MOTAPS is to setting its two parameters $\gamma\in [0,\,1]$ and $\gamma_{\tau} \in [0,\,1]$ we did a parameter sweep.
We searched over the grid given by $$ \gamma \in \{0.01, 0.1, 0.4, 0.7, 0.9, 1.0, 1.1\}$$ and $$\gamma_{\tau} \in \{ 0.00001, 0.0001, 0.001, 0.01, 0.1, 0.5, 0.9\}$$ and ran \MOTAPS for 50 epochs over the data, and recorded the resulting norm of the gradient. 
See Figures ~\ref{fig:lrmushrooms},
 ~\ref{fig:lrphishing},~\ref{fig:lrcolon} and  ~\ref{fig:grid}~\ref{fig:lrduke}  for the results 
of the grid search on the datasets \texttt{mushrooms}, \texttt{phishing},~ \texttt{colon-cancer}  and~\texttt{duke}  respectively.
 In Table~\ref{tab:datasetsconv} we resume the results of the parameter search, together with the details of each data set.


\begin{figure}
\centering
       \includegraphics[width=0.35\textwidth]{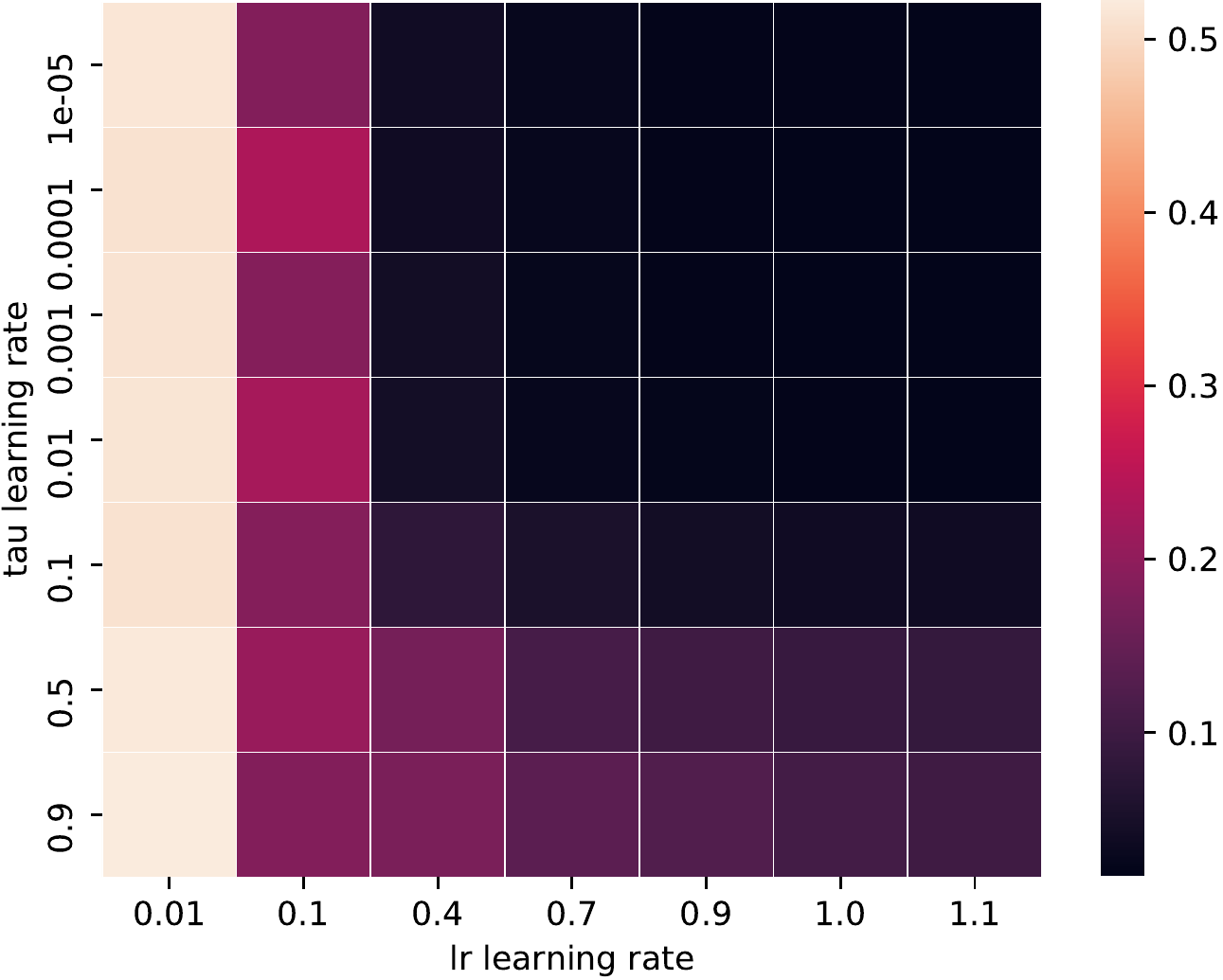}\hspace{1cm}
       \includegraphics[width=0.35\textwidth]{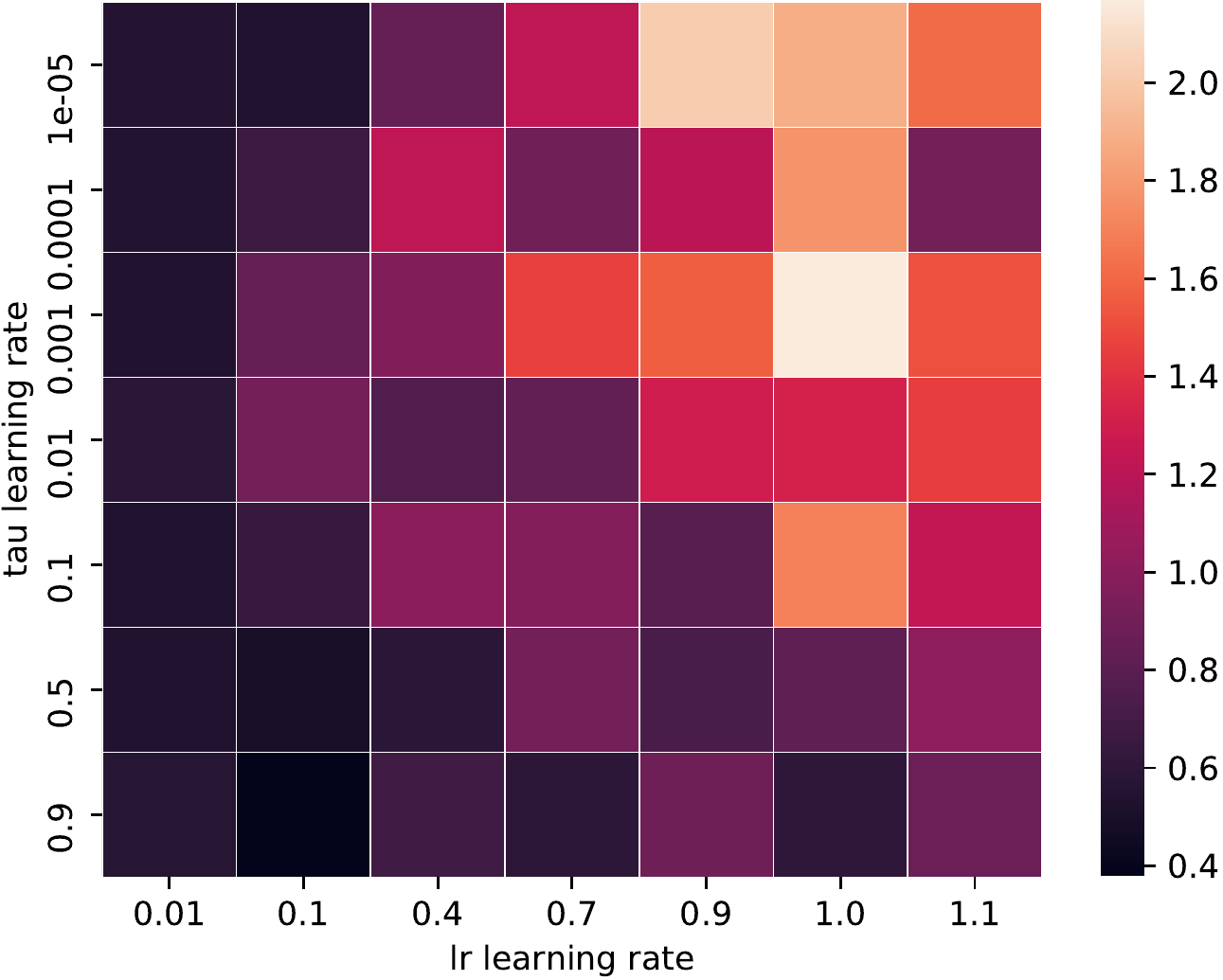}
         \caption{\texttt{colon-cancer} $(n,d) = (62, 2001)$ Left:  $\sigma =0.0$. Right: $\sigma =\min_{i=1,\ldots, n} \norm{x_i}^2/n = 2.66$.}
         \label{fig:lrcolon}
 \end{figure}        
         
\begin{figure}
\centering
\includegraphics[width=0.35\textwidth]{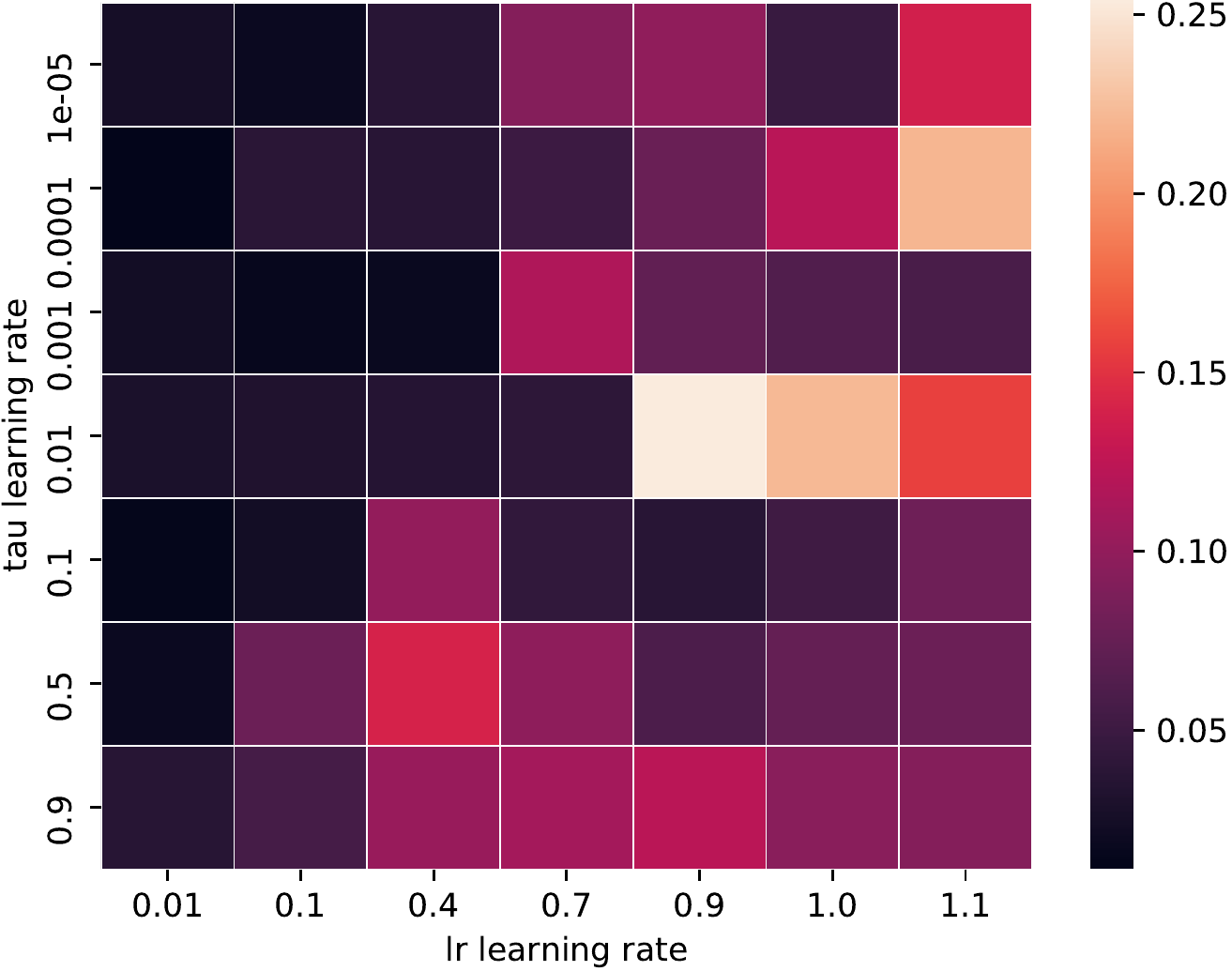}\hspace{1cm}
\includegraphics[width=0.35\textwidth]{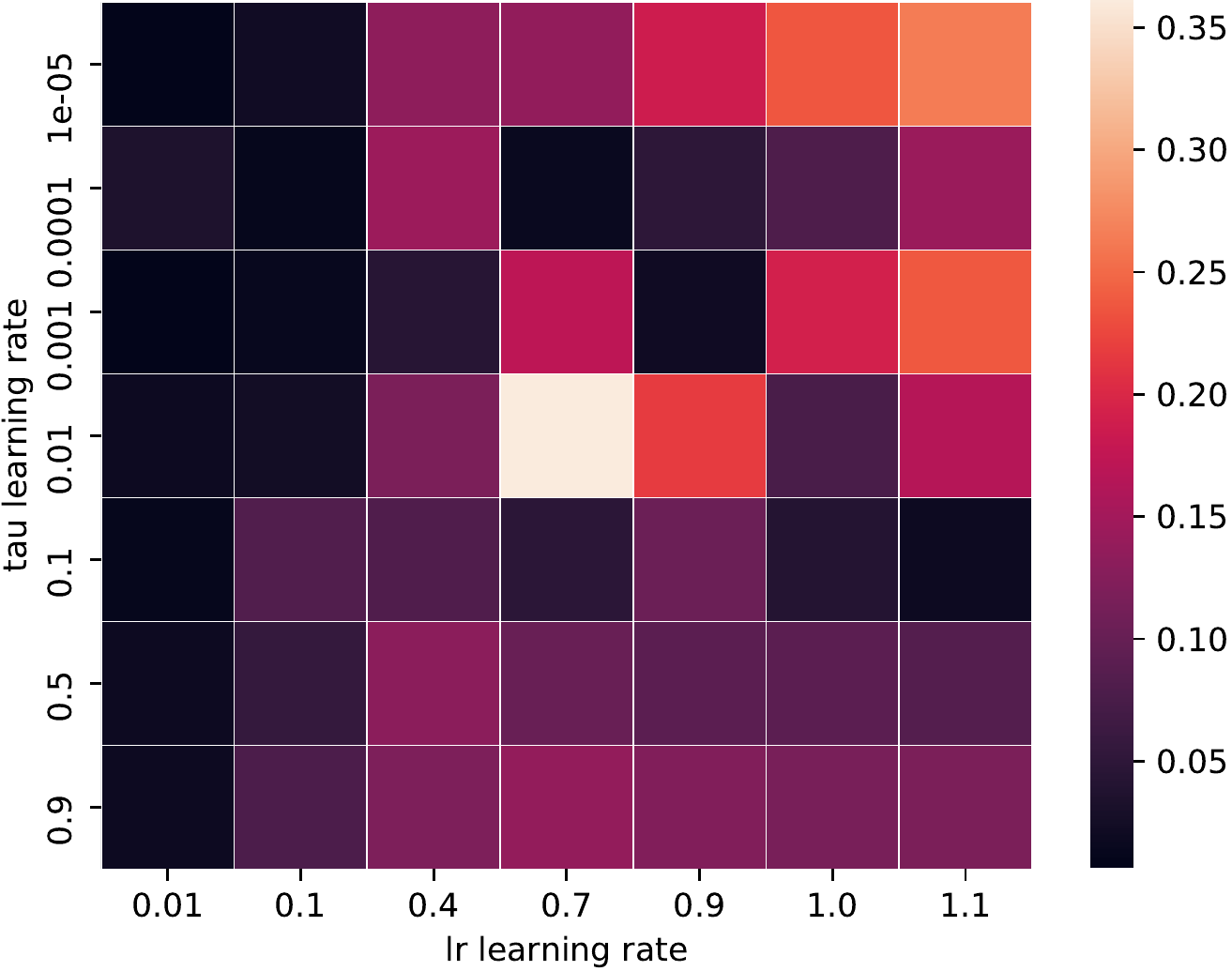}
\caption{Logistic Regression with data set  phishing $ (n,d) = ( 11055,  68)$  and regularization. Left:  $\sigma =0.0$ and Right: $\sigma =\min_{i=1,\ldots, n} \norm{x}_i^2/n$ . 
}
\label{fig:lrphishing}
\end{figure}


Ultimately the determining factor for finding the best parameter was the magnitude of the optimal value $f(w^*).$ Since this quantity is unknown to us a priori, we used the size of the regularization parameter as an proxy.
Based on these parameter results we devised the following rule-of-thumb for setting $\gamma$ and $\gamma_{\tau}$ with
\begin{equation}\label{eq:gammatau}
\gamma  = 1.0/(1+ 0.25 \sigma e^{\sigma}) \quad \mbox{and} \quad \gamma_{\tau} \; = \;  1 - \gamma
\end{equation}
where $\sigma$ is regularization parameter. 

\subsection{Comparing to Variance Reduced Gradient Methods}
\label{asec:exp-convex}

In Figures~\ref{fig:mushrooms},~\ref{fig:colon}, \ref{fig:duke} and~\ref{fig:phishing} we present further comparisons between \SP, \TAPS and \MOTAPS against SGD, SAG and SVRG.
We found that the variance reduced gradients methods were able to better exploit strong convexity, in particular for problems with a large regularization, and problems that were under-parameterized, with the phishing problem in Figure~\ref{fig:phishing} being the most striking example.
For problems with moderate regularization, and that were over-parametrized, the \MOTAPS performed the best. See for example the left of Figure~\ref{fig:colon} and Figure~\ref{fig:duke}. When the regularization is very small, and the problem is over-parameterized, thus making interpolation much more likely, the \SP converged the fastest. See for example the right of Figure~\ref{fig:colon} and~\ref{fig:mushrooms}.

%
%

\begin{figure}
\centering
\includegraphics[width=0.23\textwidth]{./convex/mushrooms-r-1-b-0-grad-iter}
\includegraphics[width =0.23\textwidth]{./convex/mushrooms-r-1-b-0-loss-iter}
\includegraphics[width=0.23\textwidth]{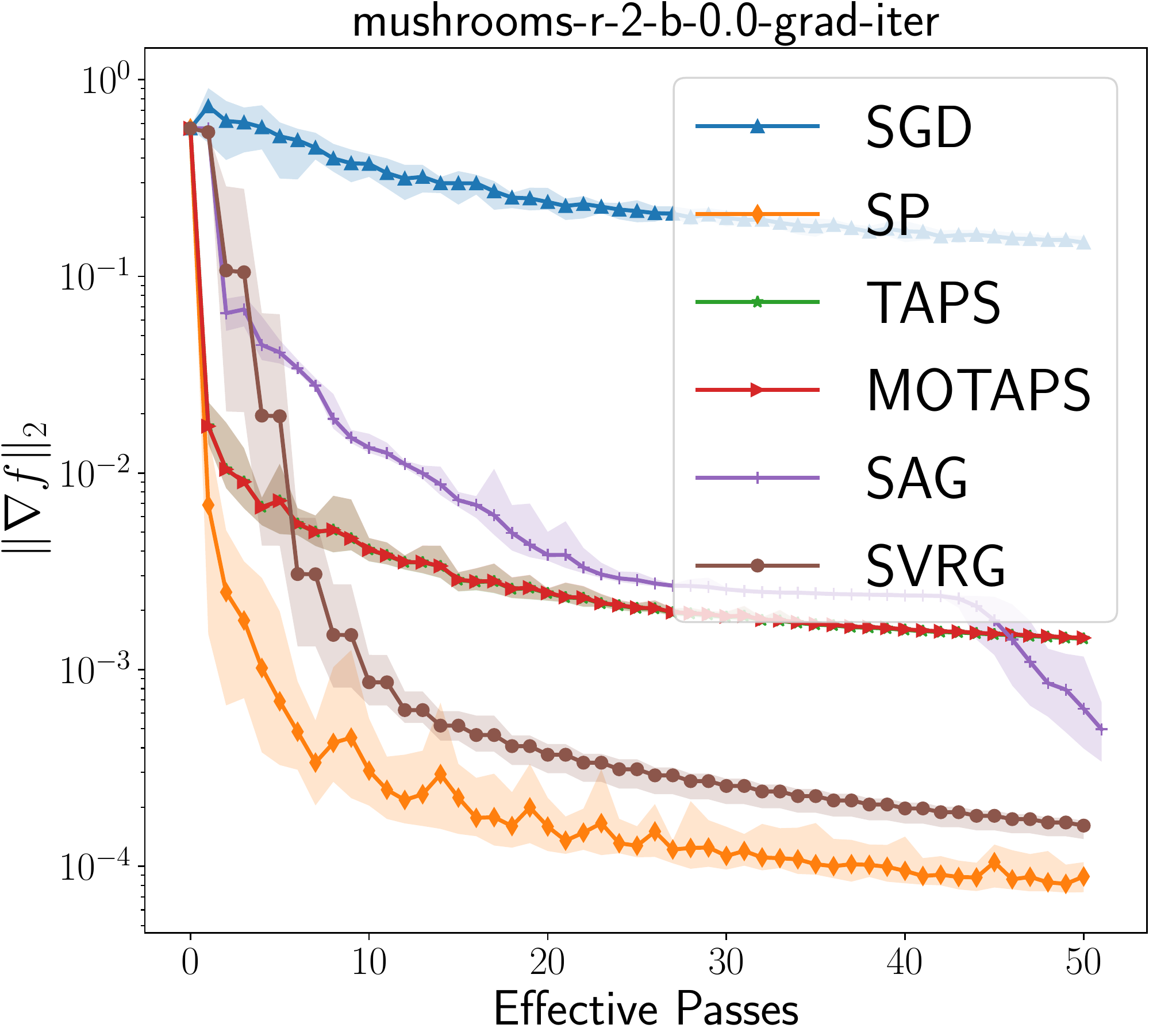}
\includegraphics[width =0.23\textwidth]{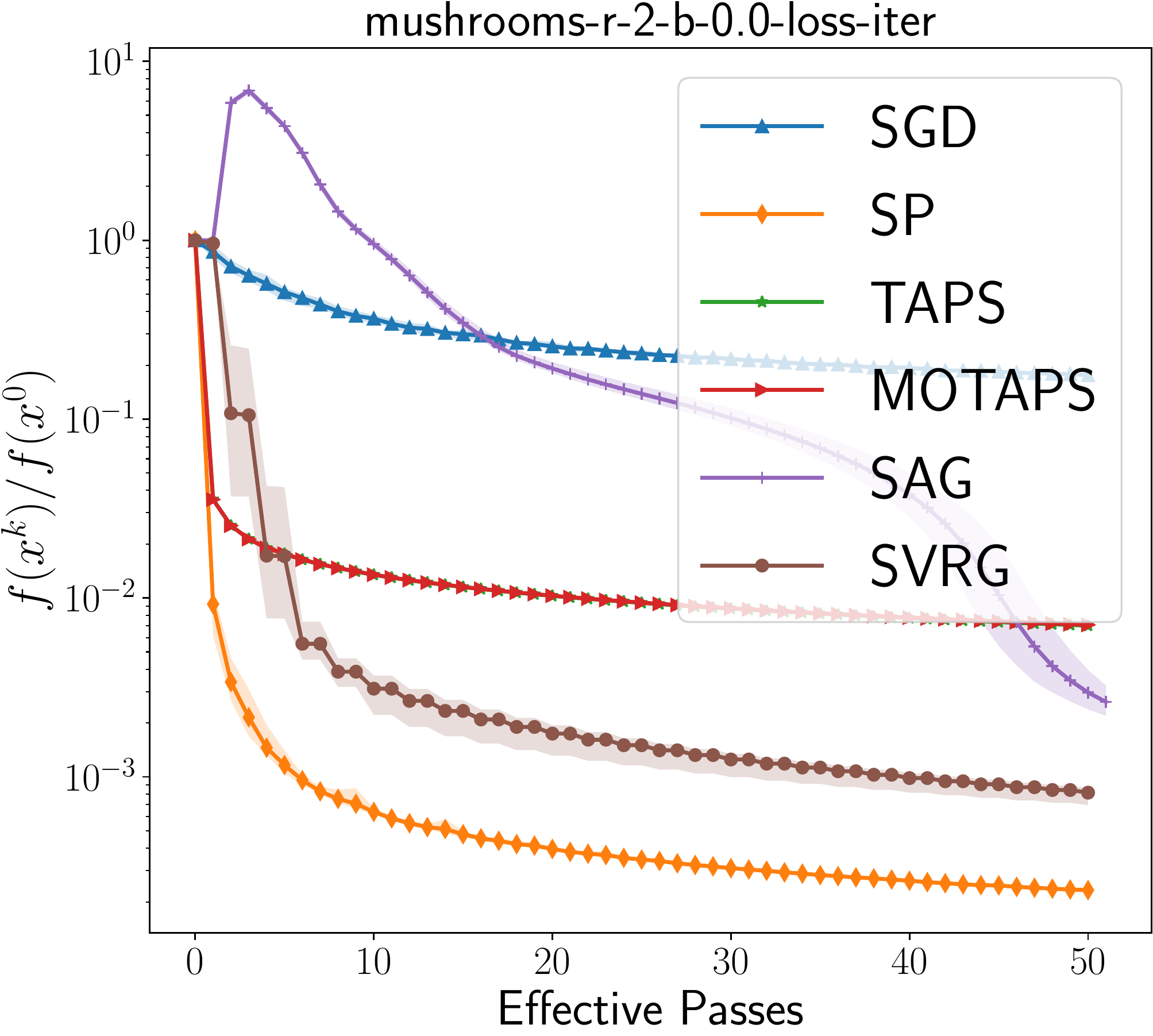}
\caption{Logistic Regression with data set  mushrooms $(n,d) = (8124, 112)   $. Left:  $\sigma =1/n$ and Right: $\sigma =1/n^2$ . See \cite{chang2011libsvm} 
}
\label{fig:mushrooms}
\end{figure}

\begin{figure}
\centering 
\includegraphics[width=0.23\textwidth]{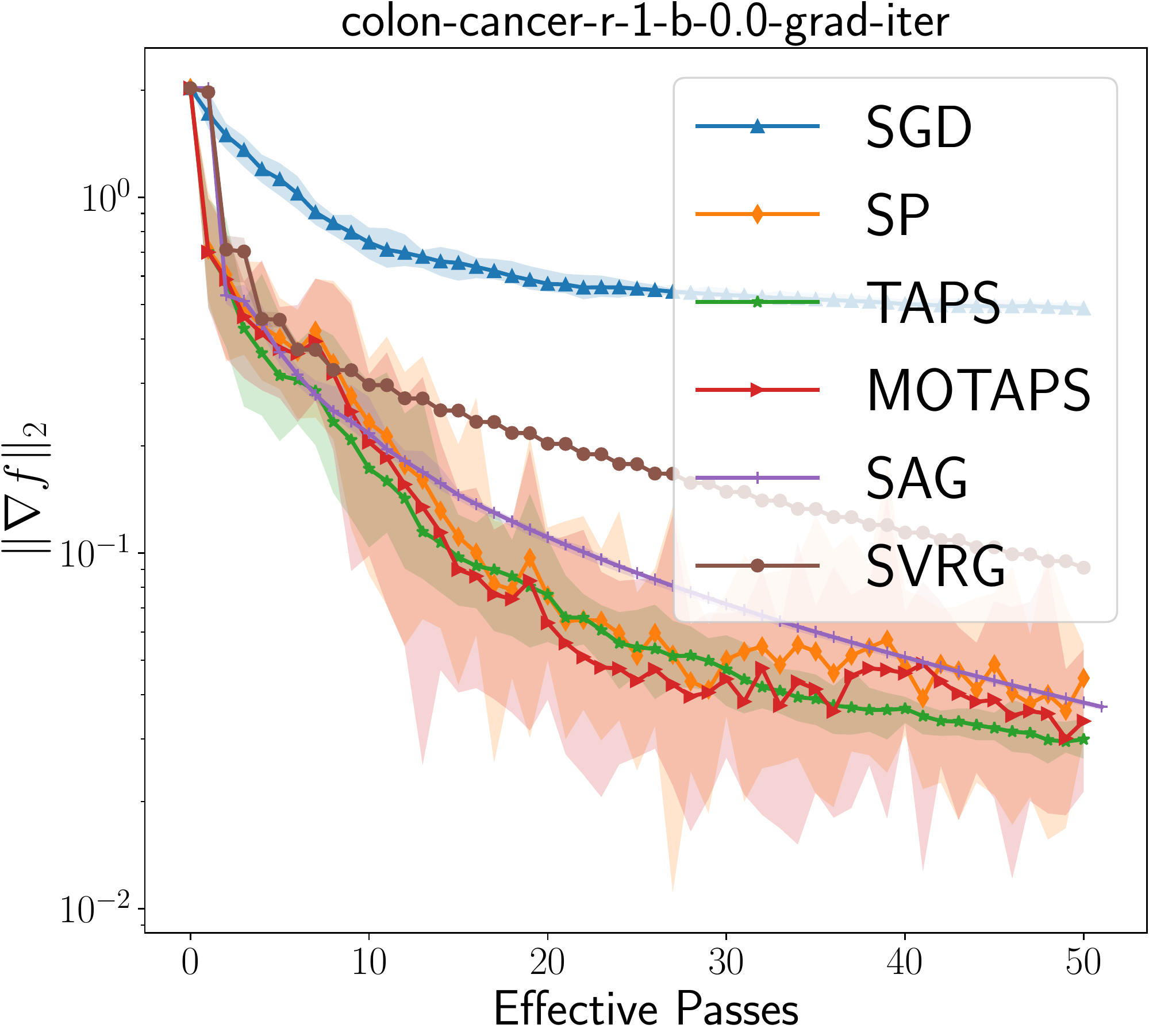}
\includegraphics[width =0.23\textwidth]{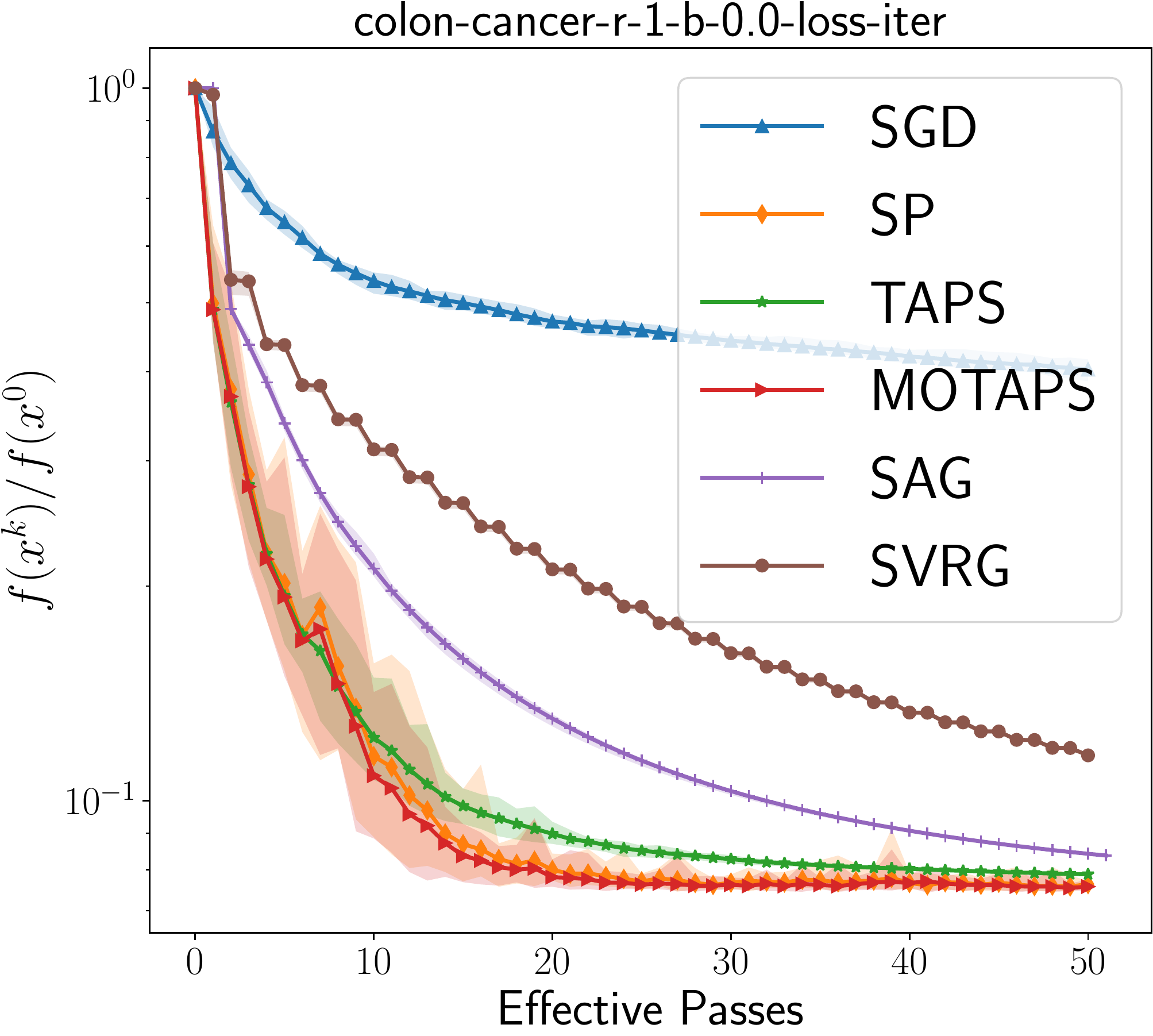}
\includegraphics[width=0.23\textwidth]{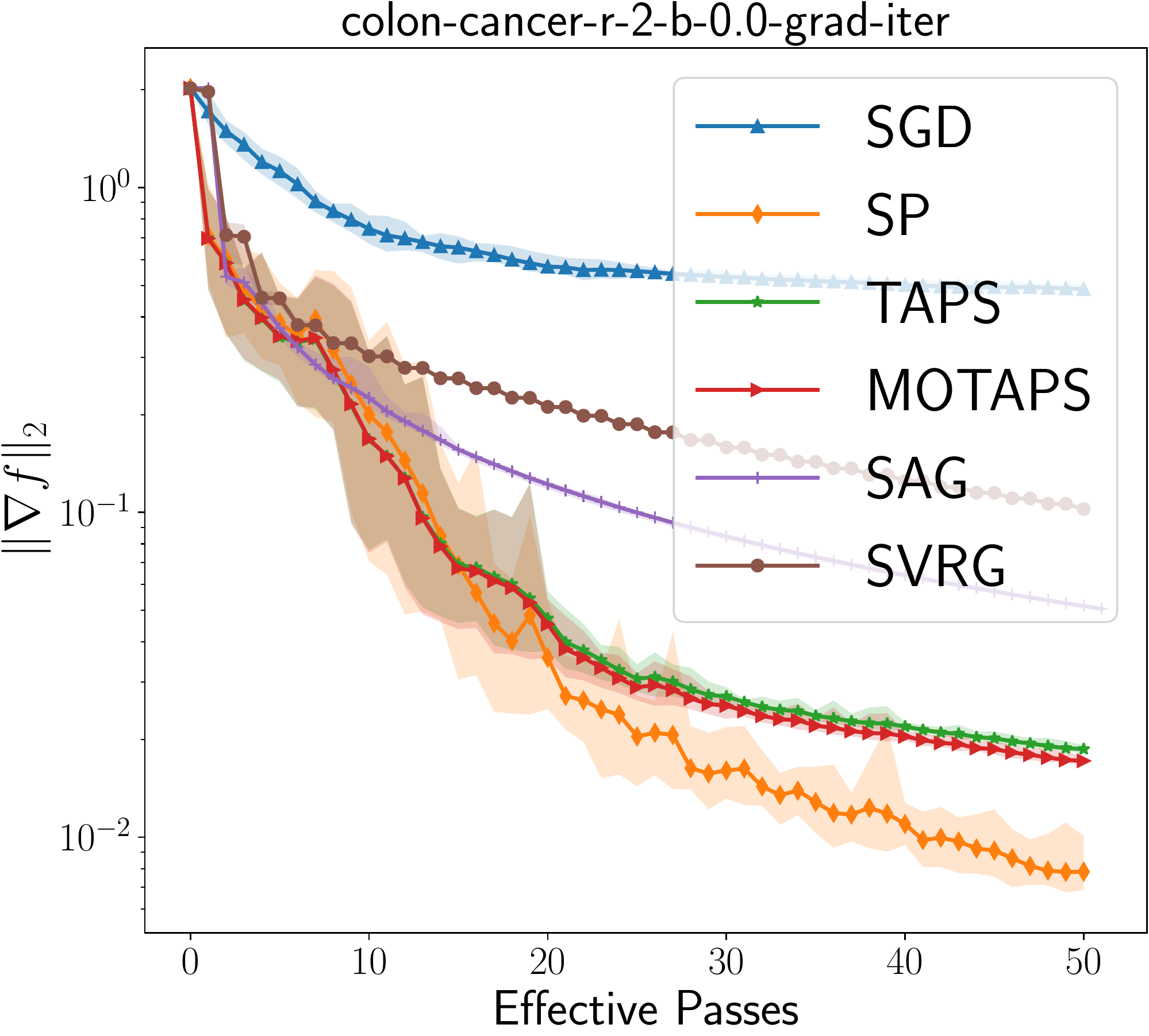}
\includegraphics[width =0.23\textwidth]{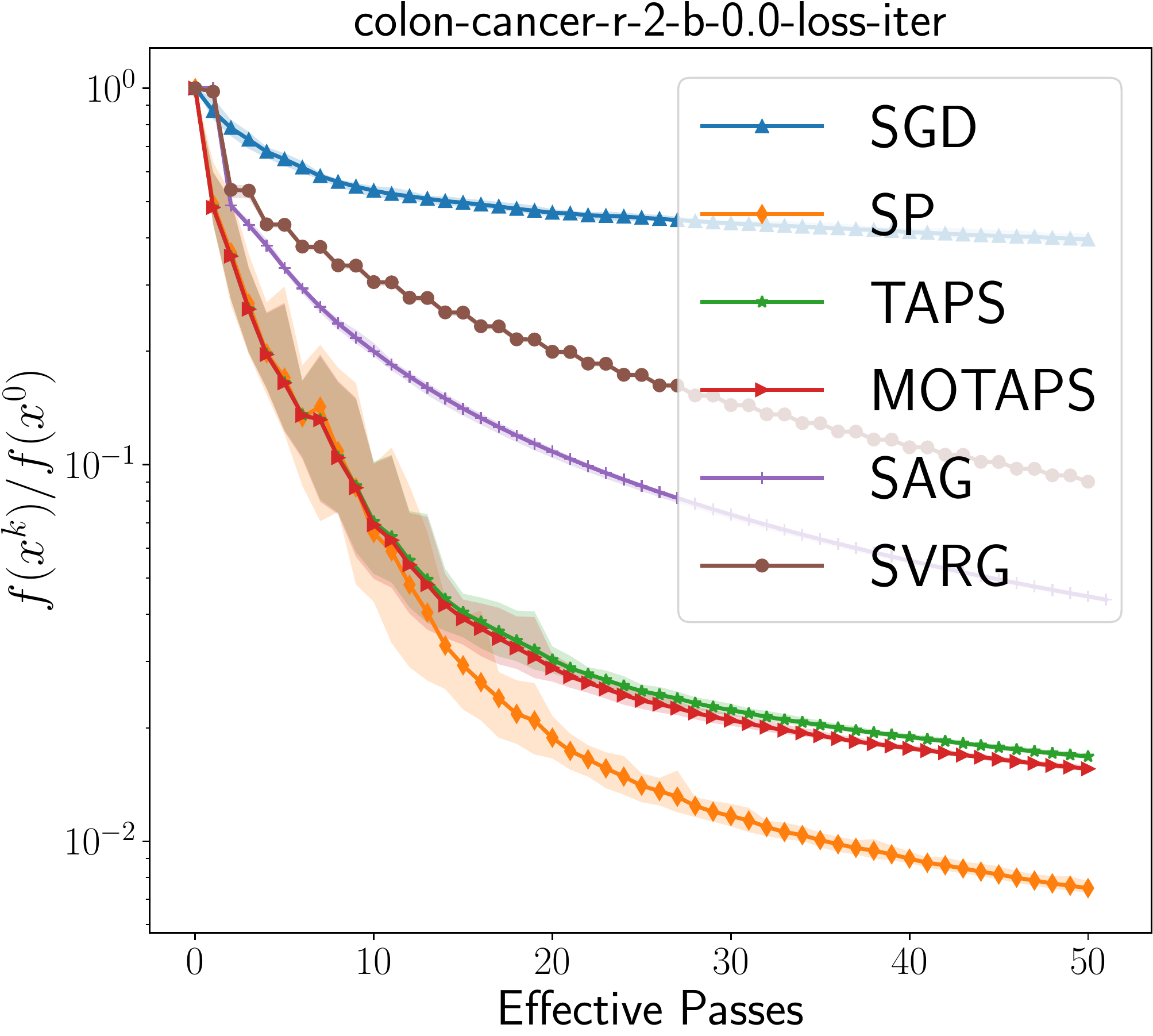}
\caption{Logistic Regression with data set  colon-cancer $(n,d) = (62, 2001)$ and regularization. Left:  $\sigma =1/n$ and Right: $\sigma =1/n^2$ . See \cite{chang2011libsvm} 
}
\label{fig:colon}
\end{figure}

\begin{figure}
\centering
\includegraphics[width=0.23\textwidth]{./convex/duke-r-1-b-0-grad-iter}
\includegraphics[width =0.23\textwidth]{./convex/duke-r-1-b-0-loss-iter}
\includegraphics[width=0.23\textwidth]{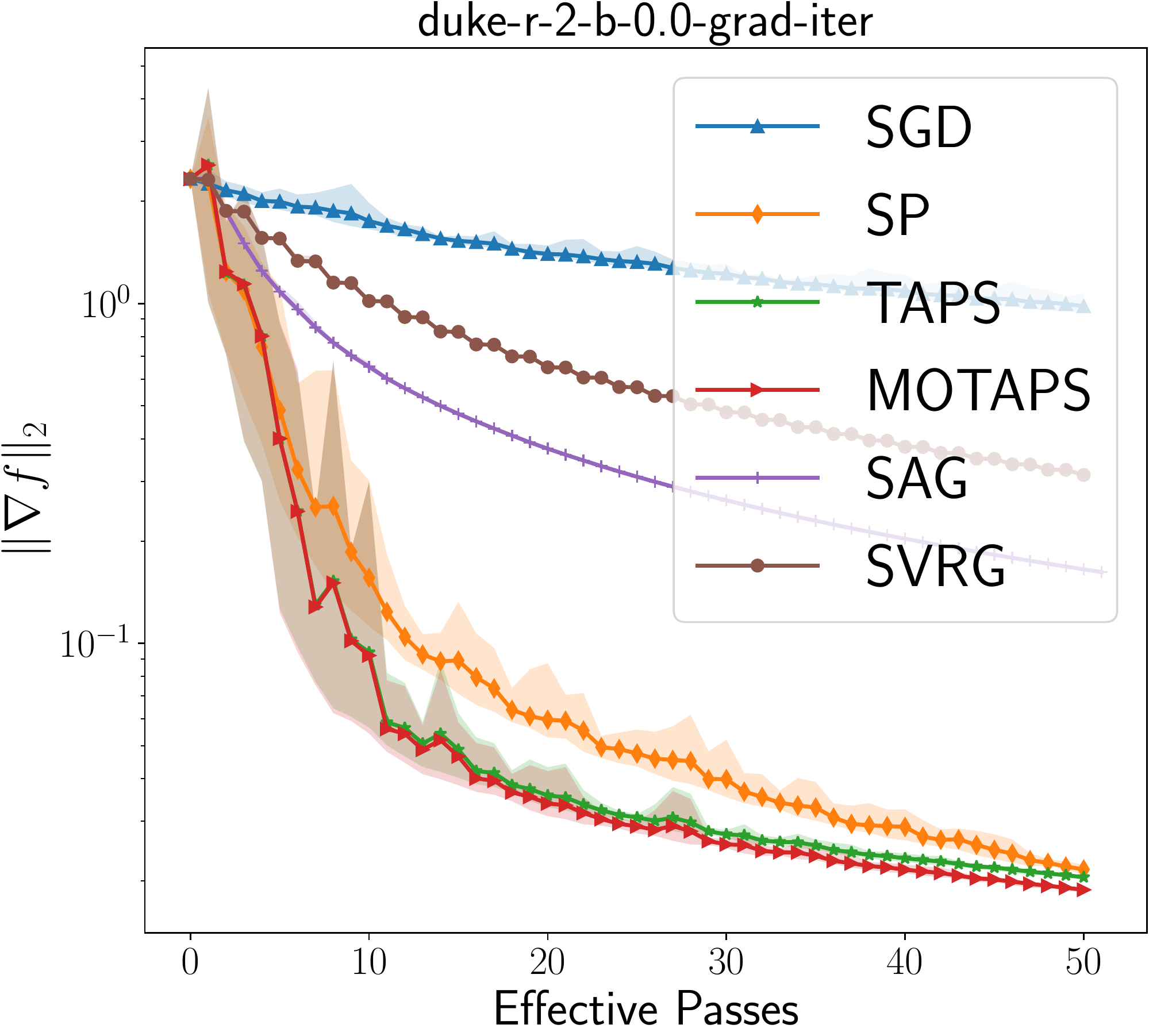}
\includegraphics[width =0.23\textwidth]{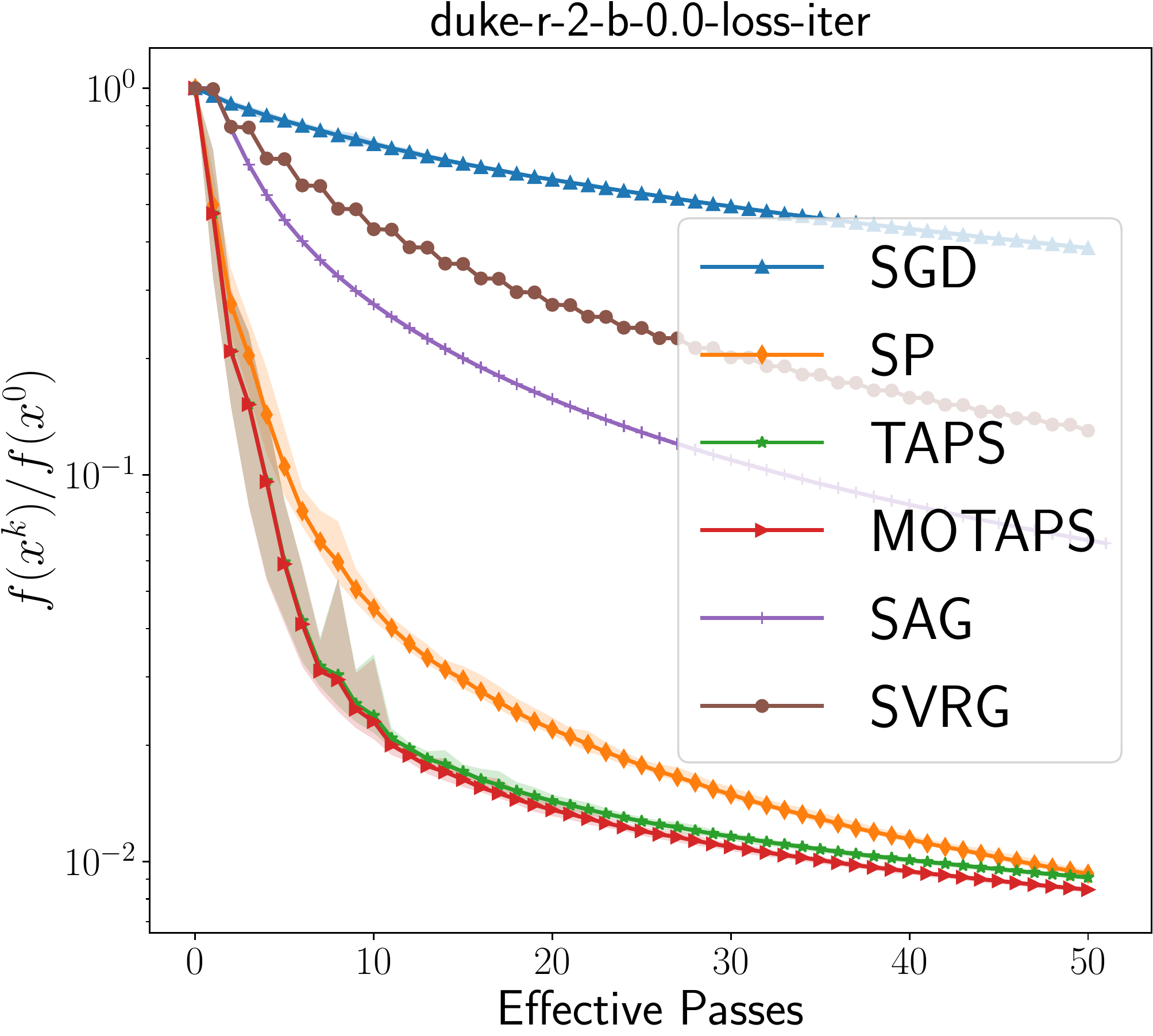}
\caption{Logistic Regression with data set  duke $(n,d) = (44, 7130)   $. Left:  $\sigma =1/n$ and Right: $\sigma =1/n^2$ . 
}
\label{fig:duke}
\end{figure}

\begin{figure}
\centering
\includegraphics[width=0.23\textwidth]{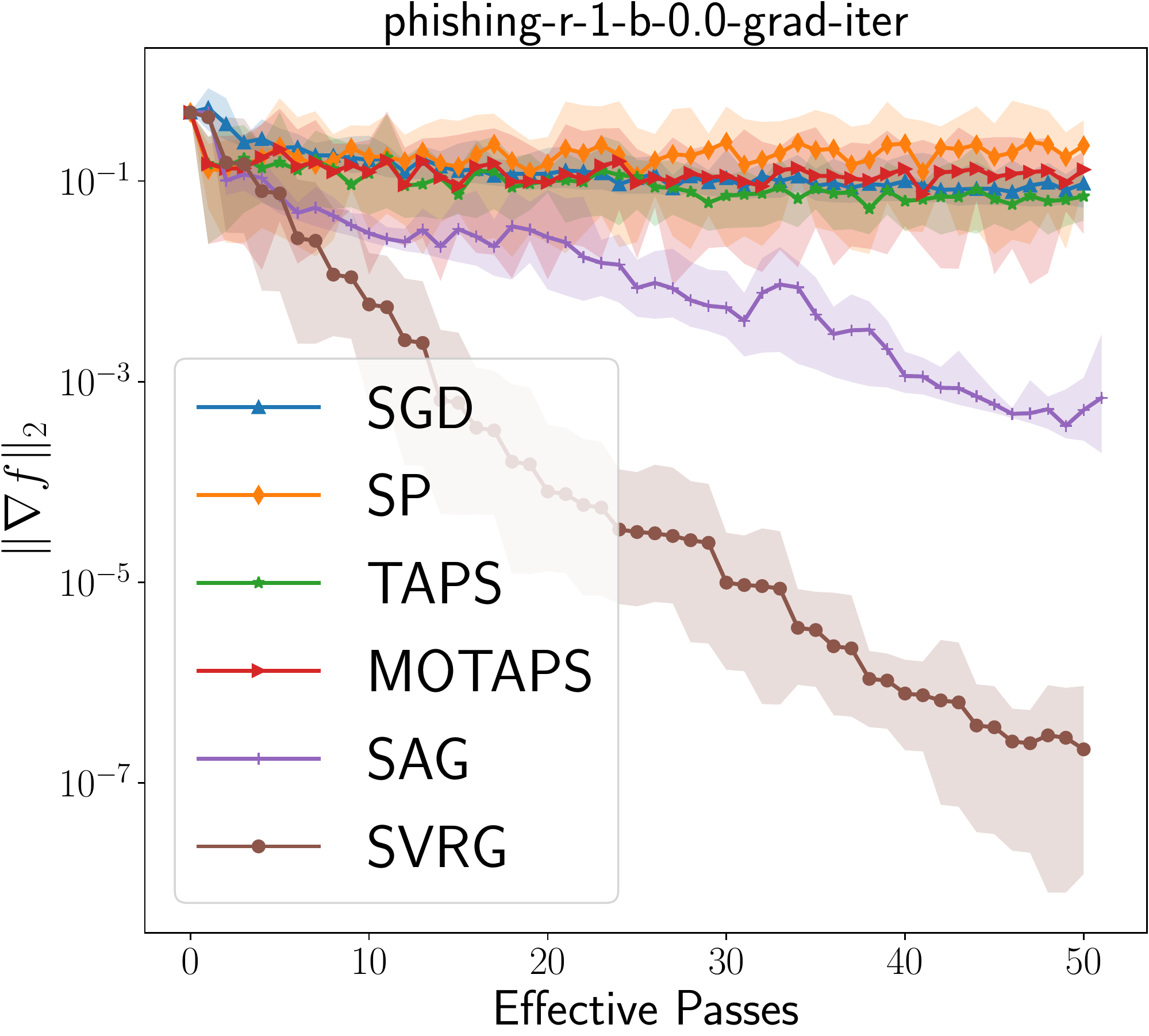}
\includegraphics[width =0.23\textwidth]{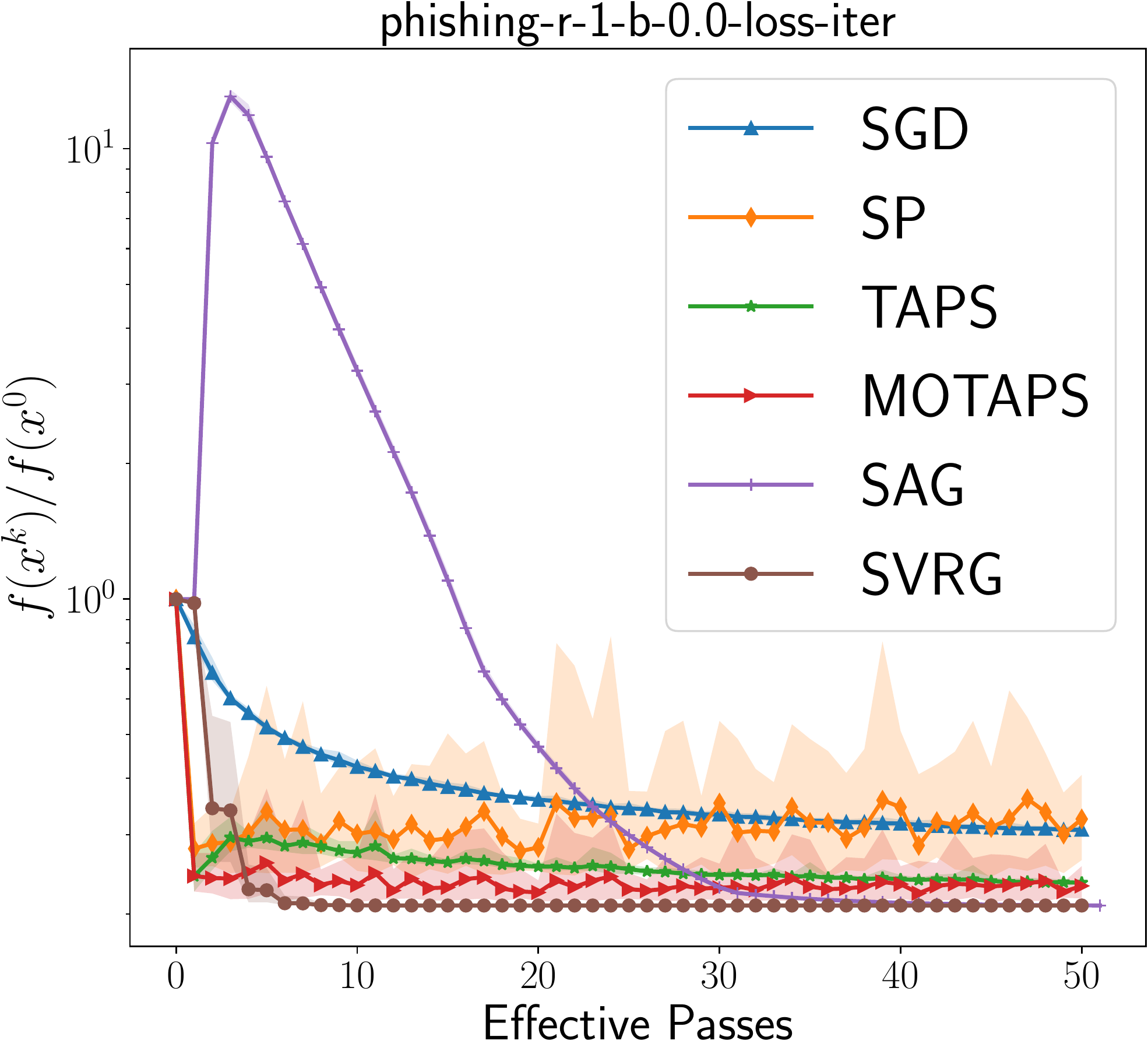}
\includegraphics[width=0.23\textwidth]{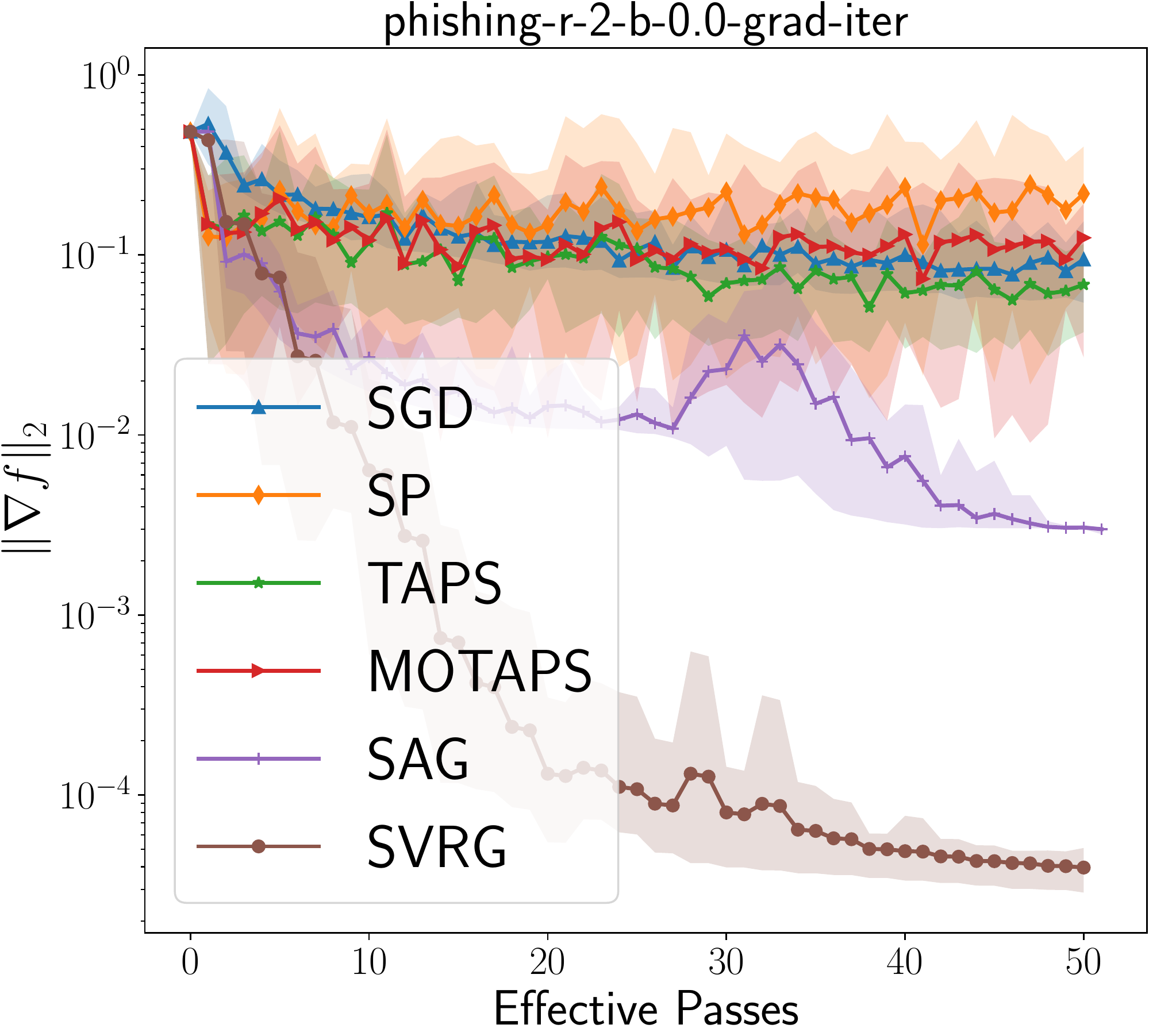}
\includegraphics[width =0.23\textwidth]{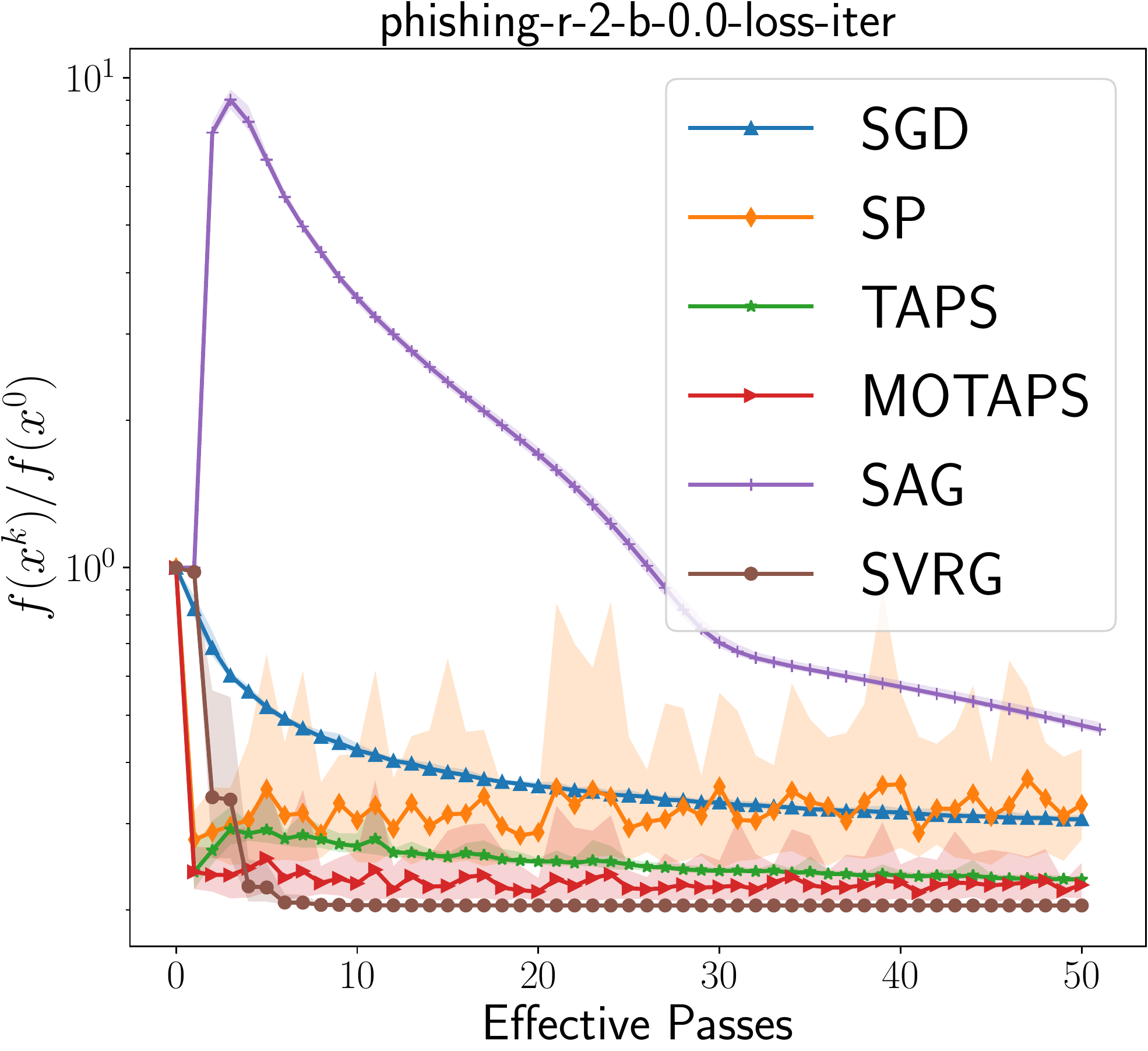}
\caption{Logistic Regression with data set  phishing    $(n,d) = (11055, 68) $. Left:  $\sigma =1/n$ and Right: $\sigma =1/n^2$ . 
}
\label{fig:phishing}
\end{figure}

\subsection{Momentum variants}
\label{asec:mom}
We also found that adding  momentum to  \SP and \MOTAPS could speed up the methods. To add momentum we used the  iterate averaging viewpoint of momentum~\cite{Sebbouh2020} where by we replace the updates in $w^t$ by a weighted average over past iterates. For \TAPS and \MOTAPS this is equivalent to introducing a sequences of $z^t$ variables and updating according to
\begin{align*}
z^{t+1}  &= z^t -\eta \frac{f_i(w^t) -\alpha^t_i}{\norm{\nabla f_i(w^t)}^2+1} \nabla f_i(w^t) \\
w^{t+1} & =\beta w^t + (1-\beta) z^{t+1} 
\end{align*} 
where $\eta = \gamma\left( 1+ \frac{\beta}{1-\beta}\right)$ is the adjusted stepsize~\footnote{See Proposition 1.6 in~\cite{Sebbouh2020} for the details of form of momentum and parameter settings. }. See Figures~\ref{fig:mom-mush}, ~\ref{fig:mom-cancer} and~\ref{fig:mom-duke} for the results of our experiments with momentum as compared to ADAM~\cite{ADAM}. We found that in regimes of moderate regularization ($\sigma = 1/n$) the \MOTAPS method was the fastest among all method, even faster than \TAPS despite not having access to $f^*,$ see the left side of  Figures~\ref{fig:mom-mush}, ~\ref{fig:mom-cancer} and~\ref{fig:mom-duke}.  Yep when using moderate regularization, adding on momentum gave no benefit to \SP, \TAPS, and \MOTAPS. Quite the opposite, for momentum $\beta =0.5$, we see that \texttt{MOTAPS}M-0.5, which is the \MOTAPS method with momentum and $\beta =0.5$, hurt the convergence rate of the method. 

In the regime of small regularization $\sigma =\frac{1}{n^2}$, we found that 
momentum sped up the convergence of our methods, see the right of Figures~\ref{fig:mom-mush}, ~\ref{fig:mom-cancer} and~\ref{fig:mom-duke}.  On the under-parameterized problem mushrooms, the gains from momentum were marginal, and the ADAM method was the fastest overall, see the right of Figure~\ref{fig:mom-mush}. On the over-parametrized problem colon-cancer, adding momentum to \SP gave a significant boost in convergence speed,  see the right of Figure~\ref{fig:mom-cancer}. Finally on the most over-parametrized problem duke, adding momentum offered a significant speed-up for \MOTAPS, but still the ADAM method was the fastest, see the right of Figure~\ref{fig:mom-duke}.


%
%

 \begin{figure}
\centering
\includegraphics[width=0.23\textwidth]{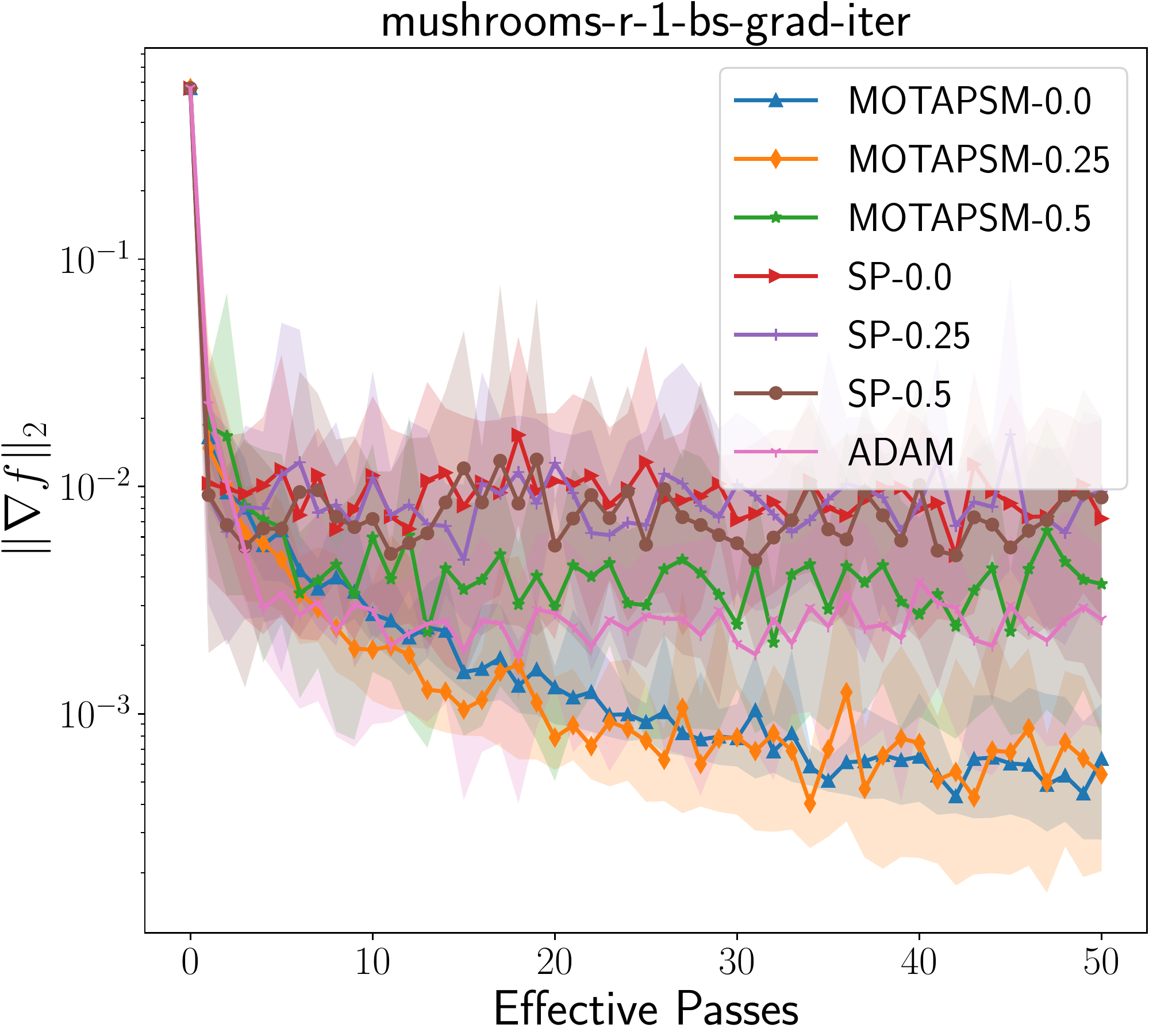}
\includegraphics[width =0.23\textwidth]{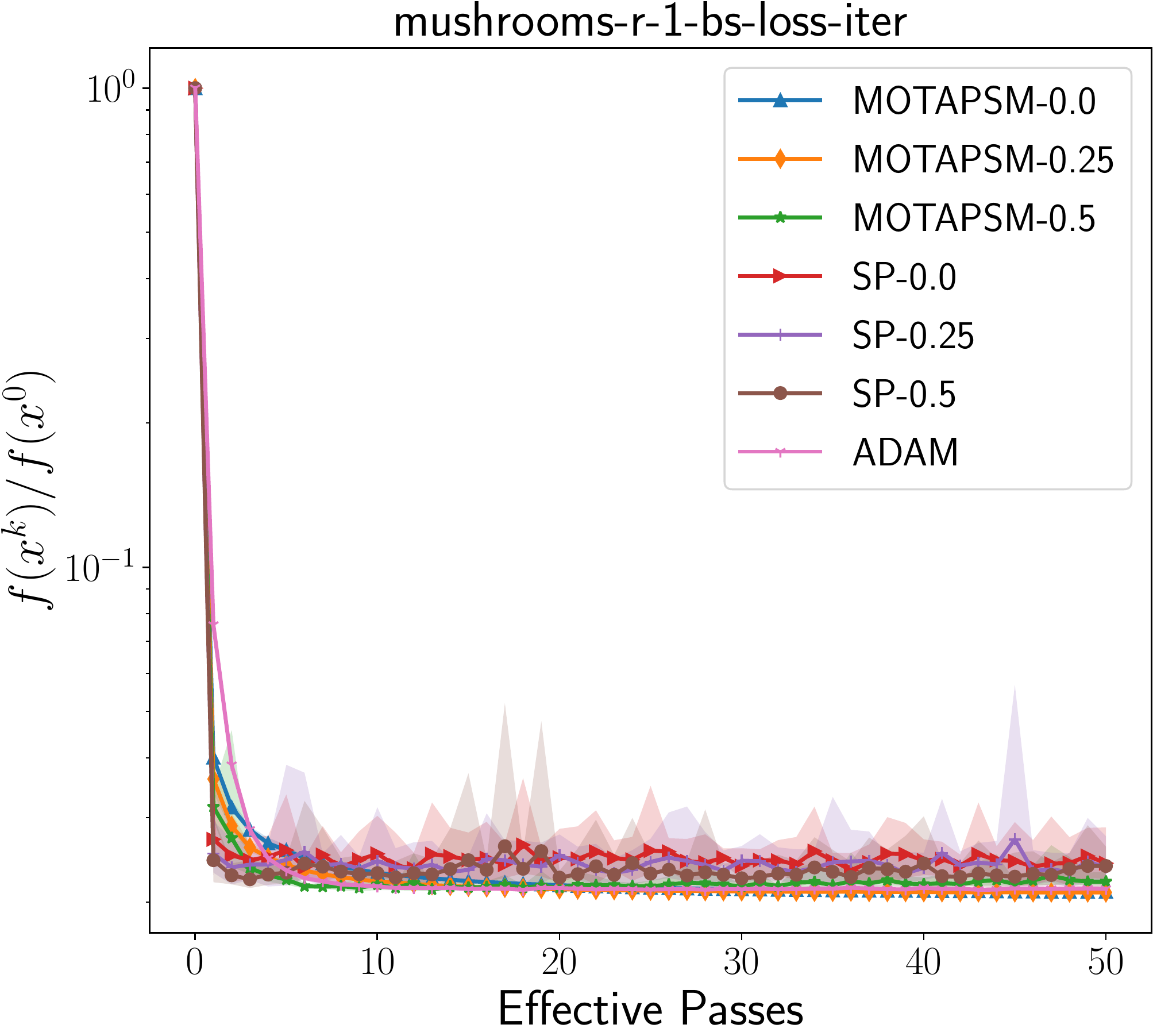}
\includegraphics[width=0.23\textwidth]{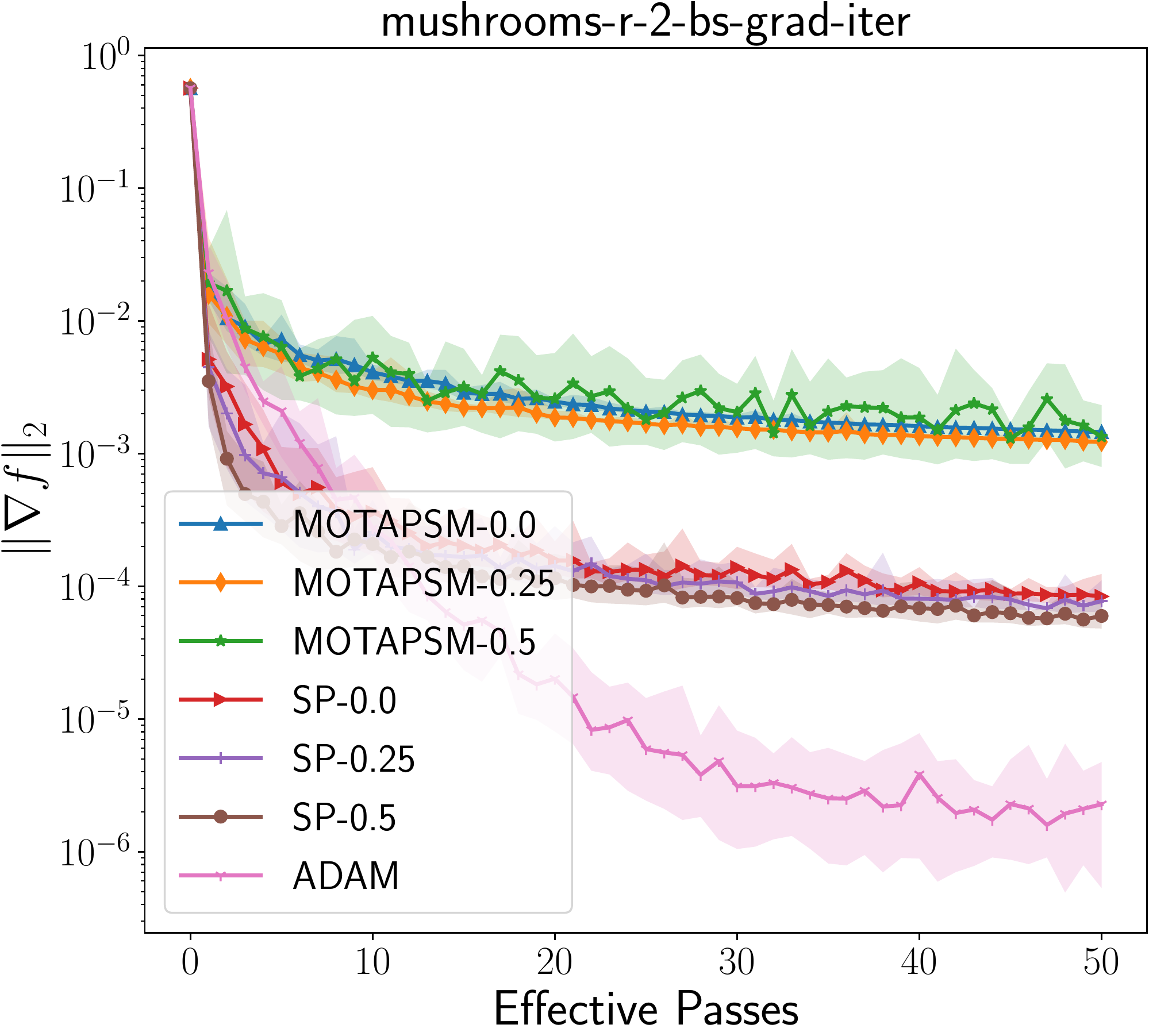}
\includegraphics[width =0.23\textwidth]{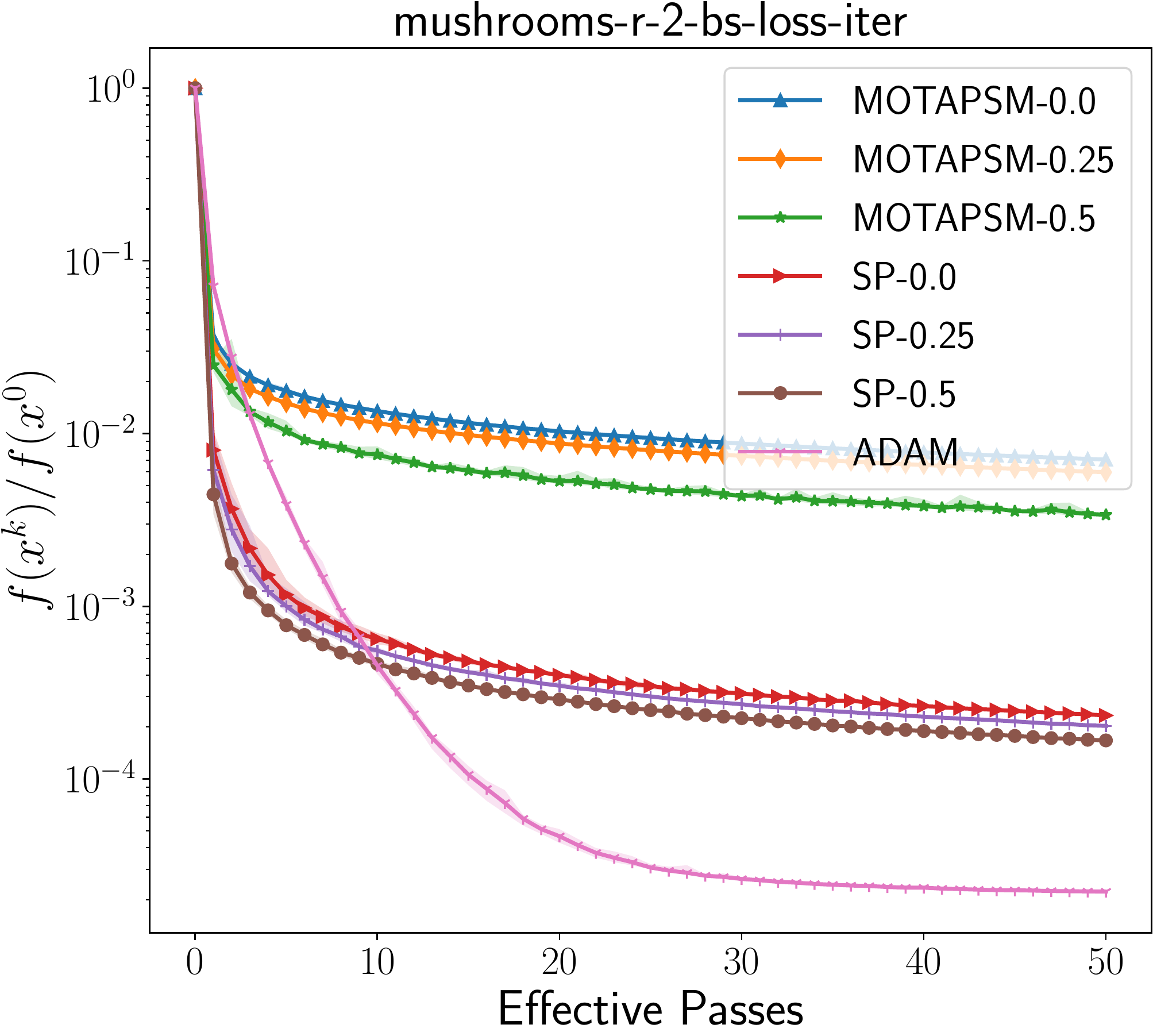}
\caption{Experiments on momentum with  mushrooms $(n,d) = (8124, 112)   $. Left:  $\sigma =1/n$ and Right: $\sigma =1/n^2$ .
}
\label{fig:mom-mush}
\end{figure}

\begin{figure}
\centering 
\includegraphics[width=0.23\textwidth]{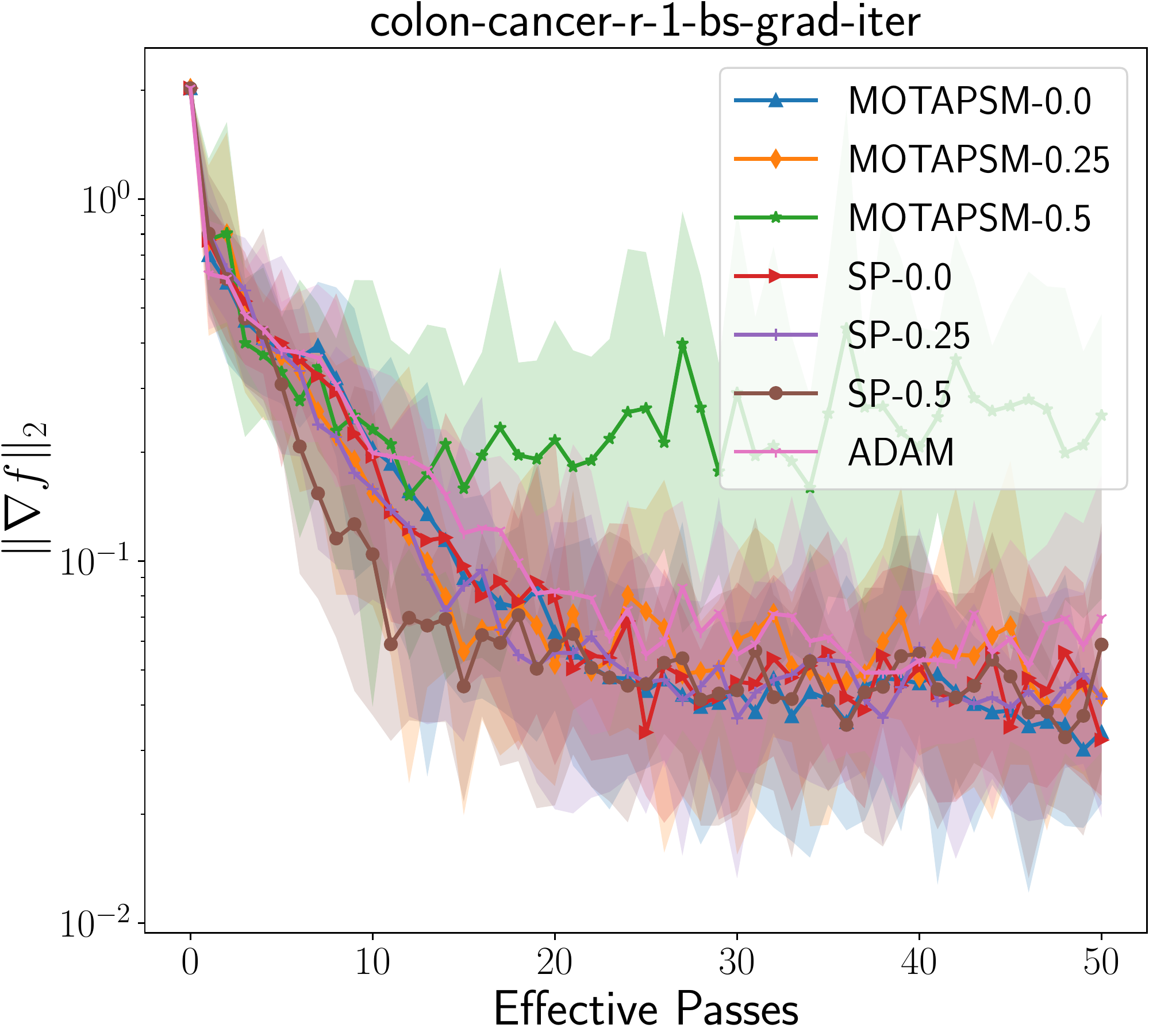}
\includegraphics[width =0.23\textwidth]{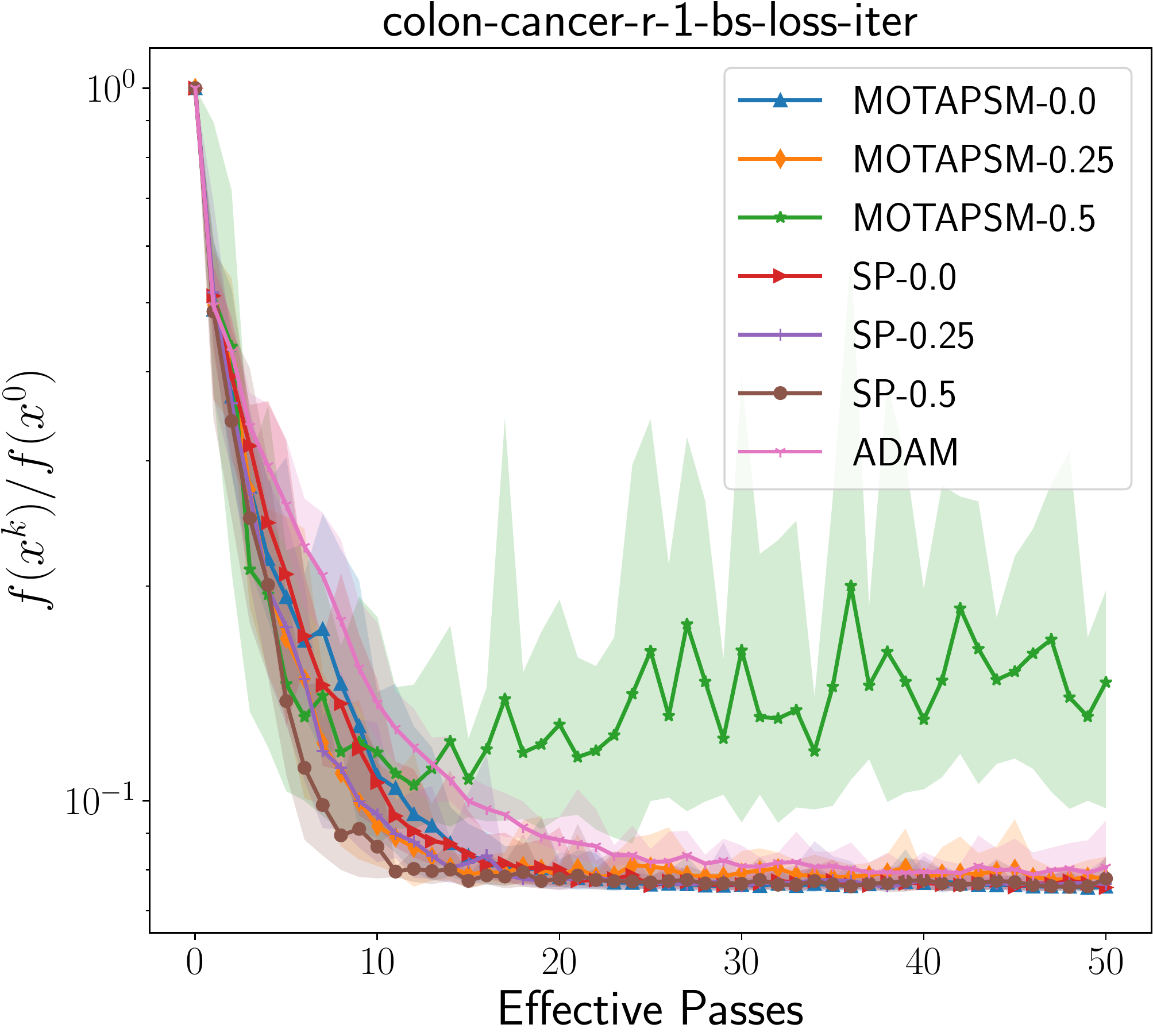}
\includegraphics[width=0.23\textwidth]{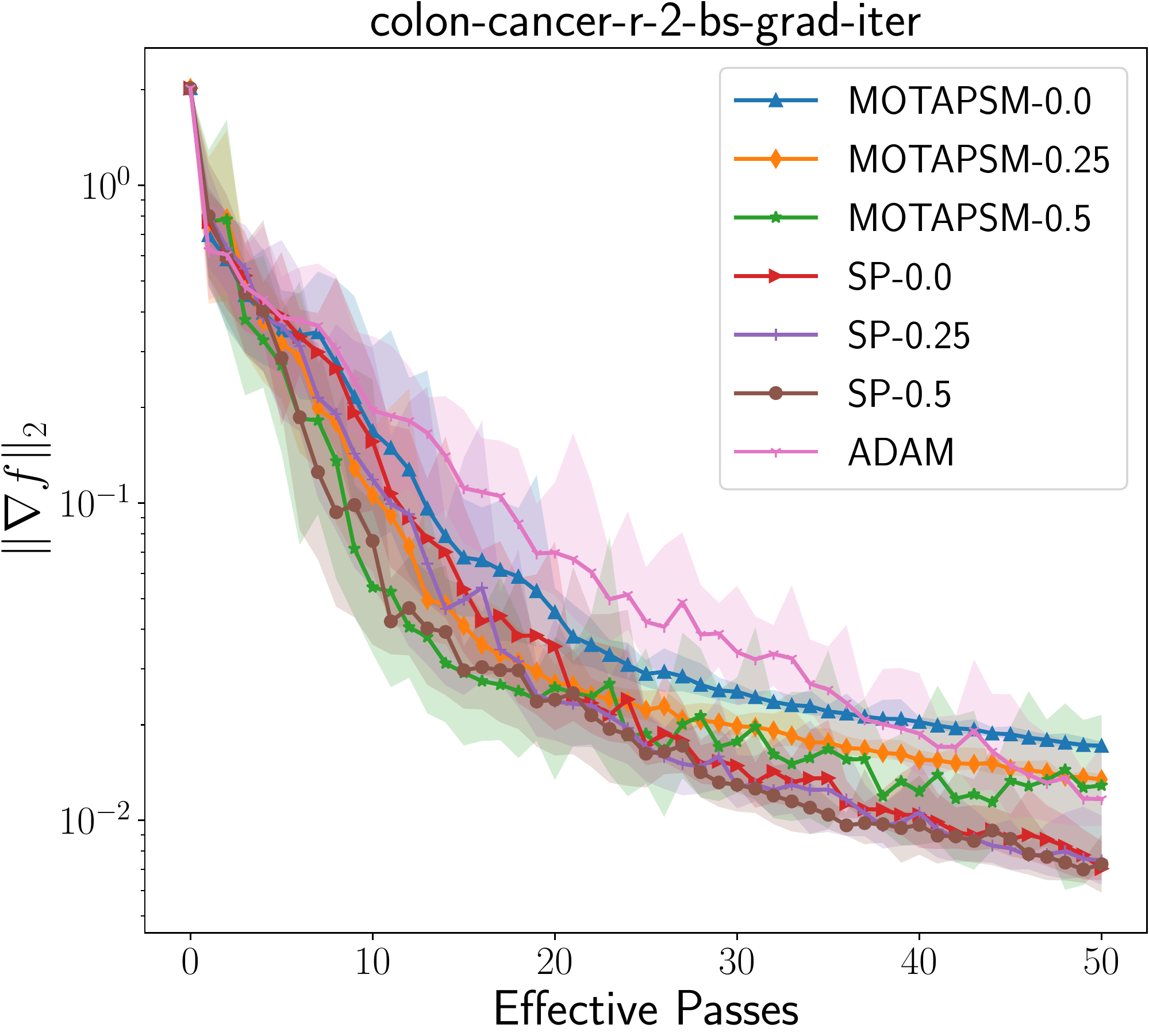}
\includegraphics[width =0.23\textwidth]{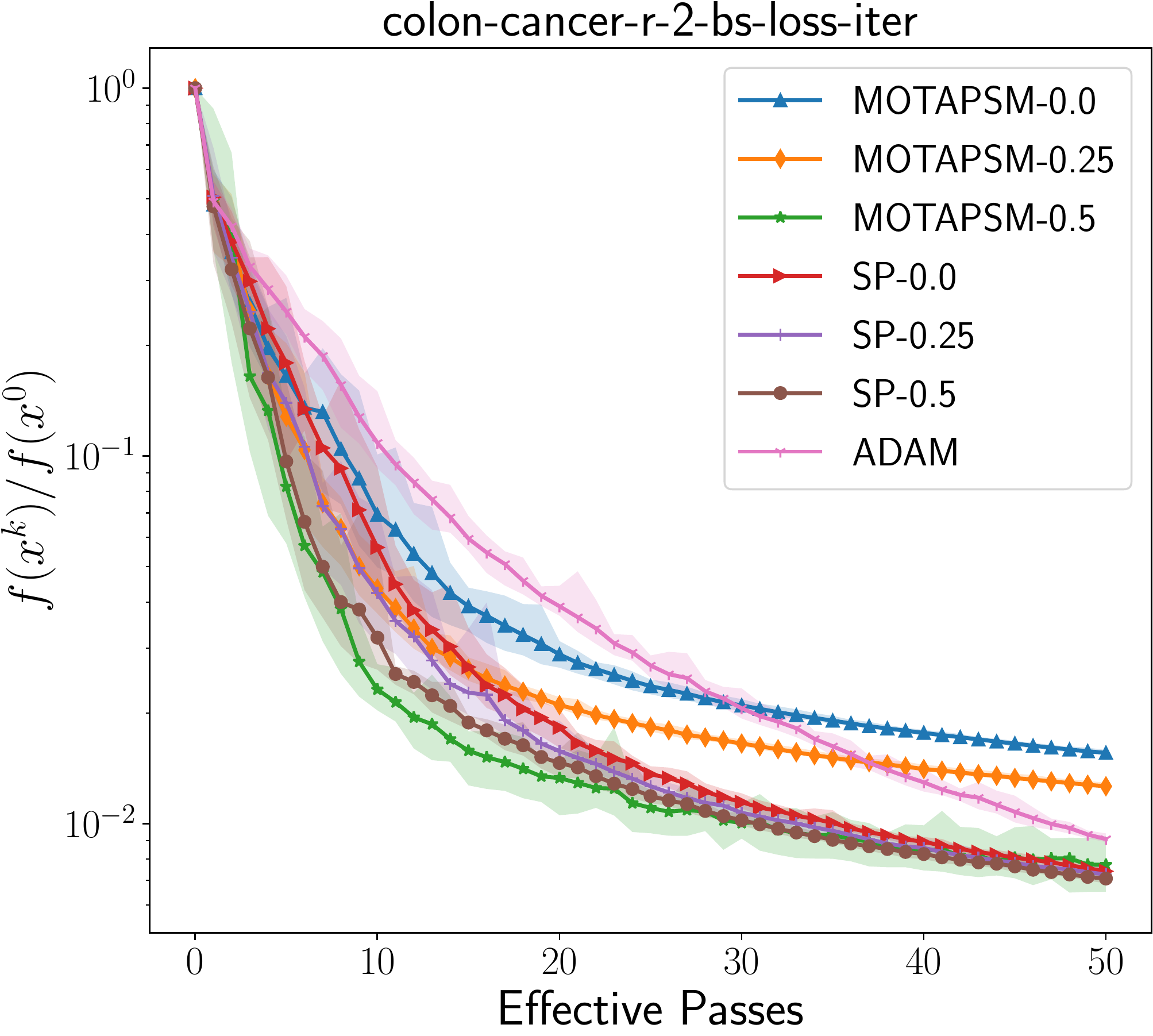}
\caption{Experiments on momentum with  colon-cancer $(n,d) = (62, 2001)$ and regularization. Left:  $\sigma =1/n$ and Right: $\sigma =1/n^2$ .
}
\label{fig:mom-cancer}
\end{figure}

\begin{figure}
\centering
\includegraphics[width=0.23\textwidth]{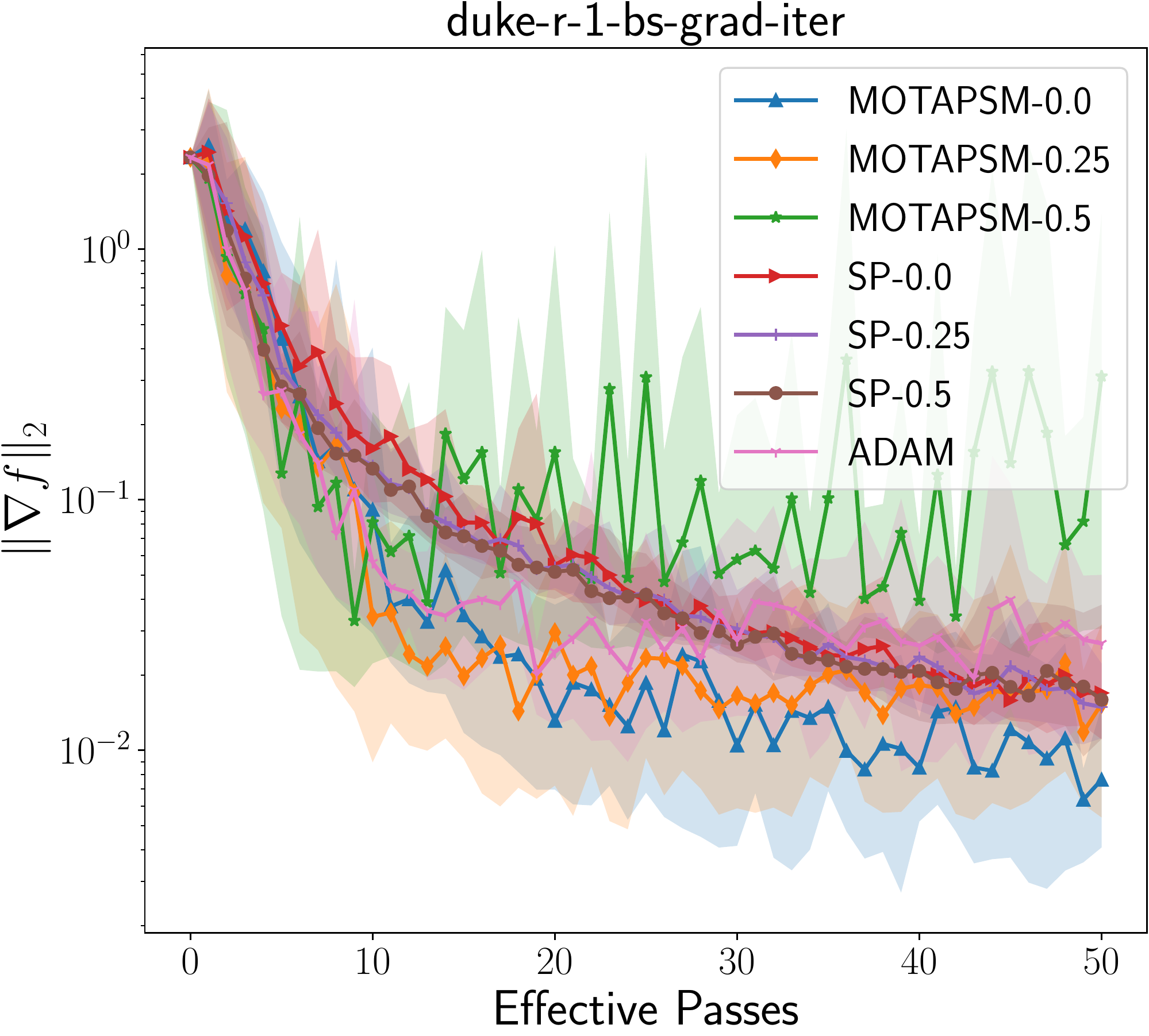}
\includegraphics[width =0.23\textwidth]{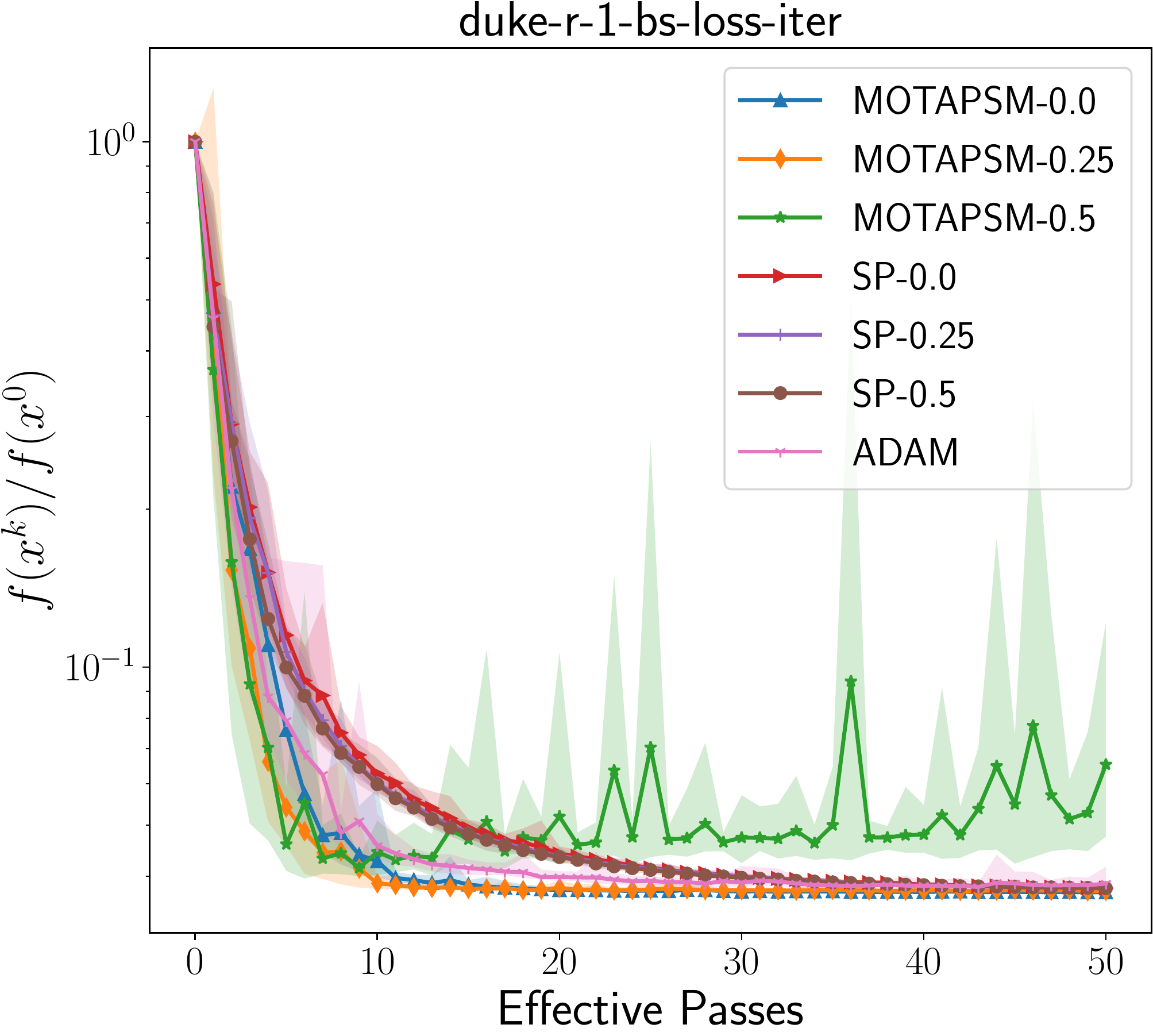}
\includegraphics[width=0.23\textwidth]{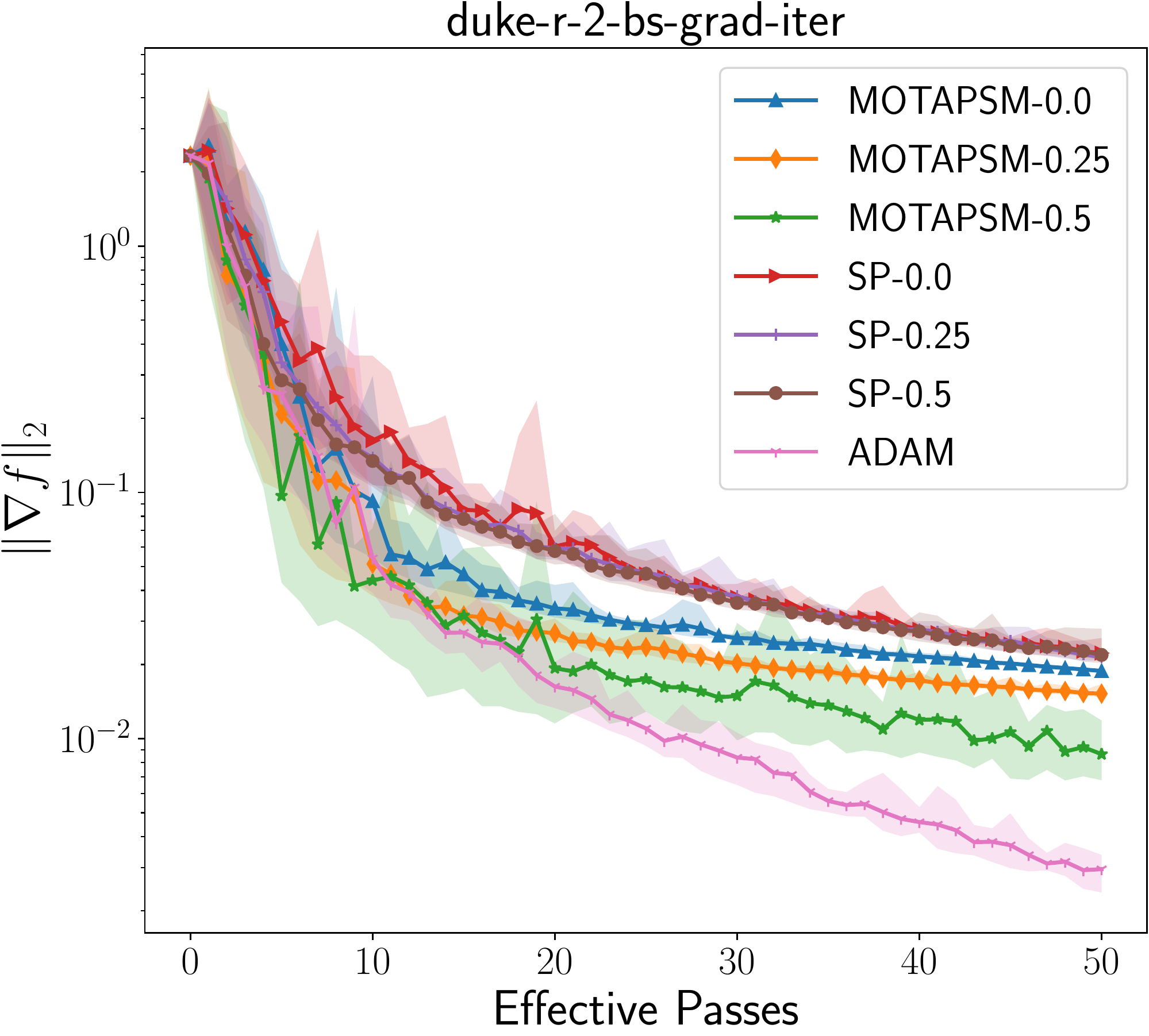}
\includegraphics[width =0.23\textwidth]{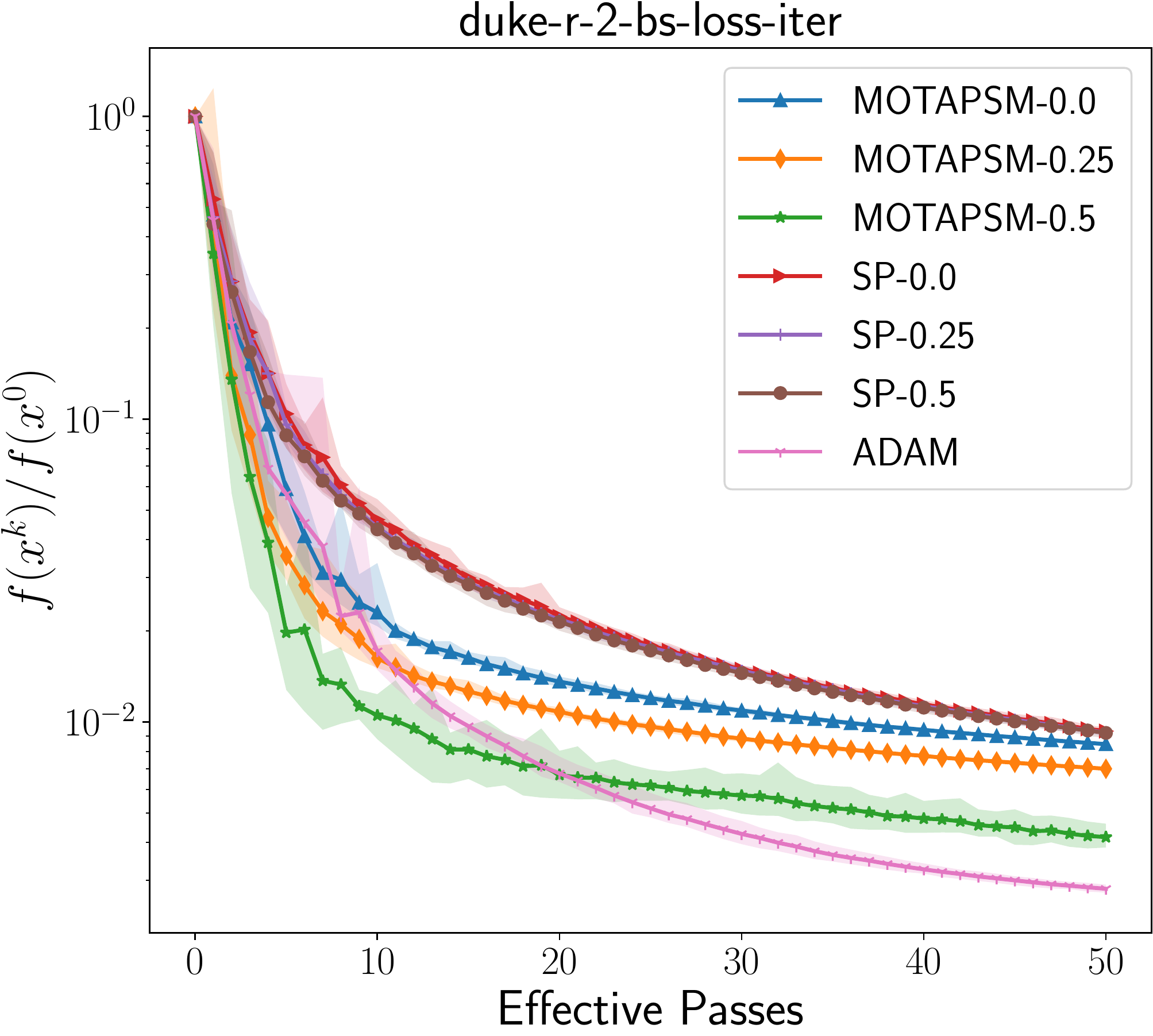}
\caption{Experiments on momentum with   duke $(n,d) = (44, 7130)   $. Left:  $\sigma =1/n$ and Right: $\sigma =1/n^2$ . 
}
\label{fig:mom-duke}
\end{figure}


%
%
%
\section{Deep learning experimental setup details}
 
 \label{asec:deep}

 In this section we detail the specific implementation choices for each environment. Across all environments, minibatching was accomplished by treating each minibatch as a single data-point. Since per-datapoint values are tracked across epochs, our training setup used minibatches which contain the same set of points each epoch.
 
 \subsection{CIFAR10}
 We trained for 300 epochs using batch size 256 on 1 GPU. Momentum 0.9 was used for all methods. The pre-activation ResNet used has 58,144,842 parameters.  Following standard practice we apply data augmentation of the training data; horizontal flipping, 4 pixel padding followed by random cropping to 32x32 square images. 
 
 \subsection{SVHN}
 
 We trained for 150 epochs on a single GPU, using a batch size of 128. Momentum 0.9 was used for each method. Data augmentations were the same as for our CIFAR10 experiments. The ResNet-1-16 network has a total of 78,042 parameters, and uses the classical, non-preactivation structure.
 
\subsection{IWSLT14}
We used a very simple preprocessing pipeline, consisting of the Spacy de\_core\_news\_sm/en\_core\_web\_sm tokenizers and filtering out of sentences longer than 100 tokens to fit without our GPU memory constraints. Training used batch-size 32, across 1 GPU for 25 epochs. Other hyper-parameters include momentum of 0.9, weight decay of 5e-6, and a linear learning rate warmup over the first 5 epochs

\begin{figure}
\centering
\includegraphics[width=4.6cm]{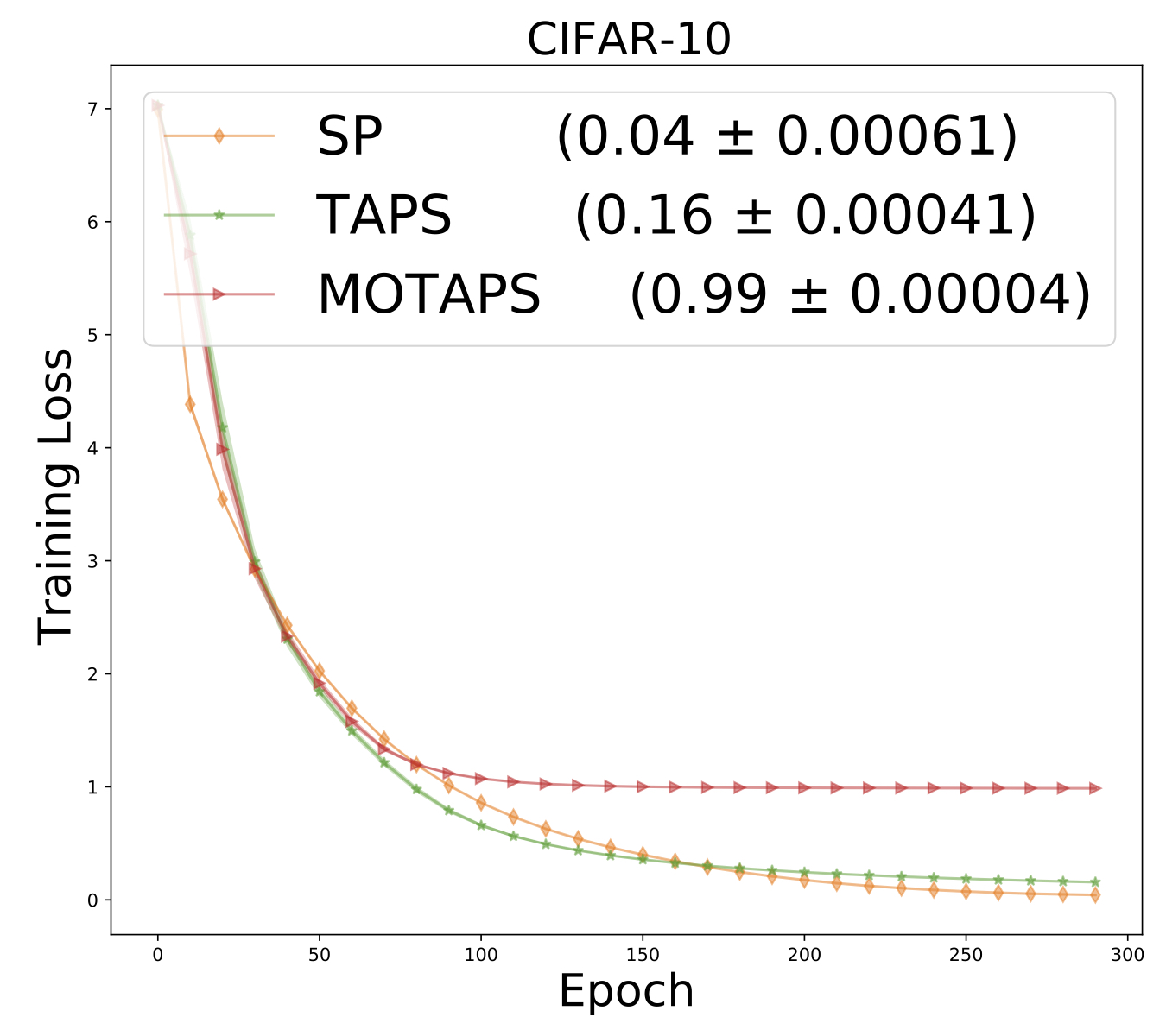}
\includegraphics[width=4.6cm]{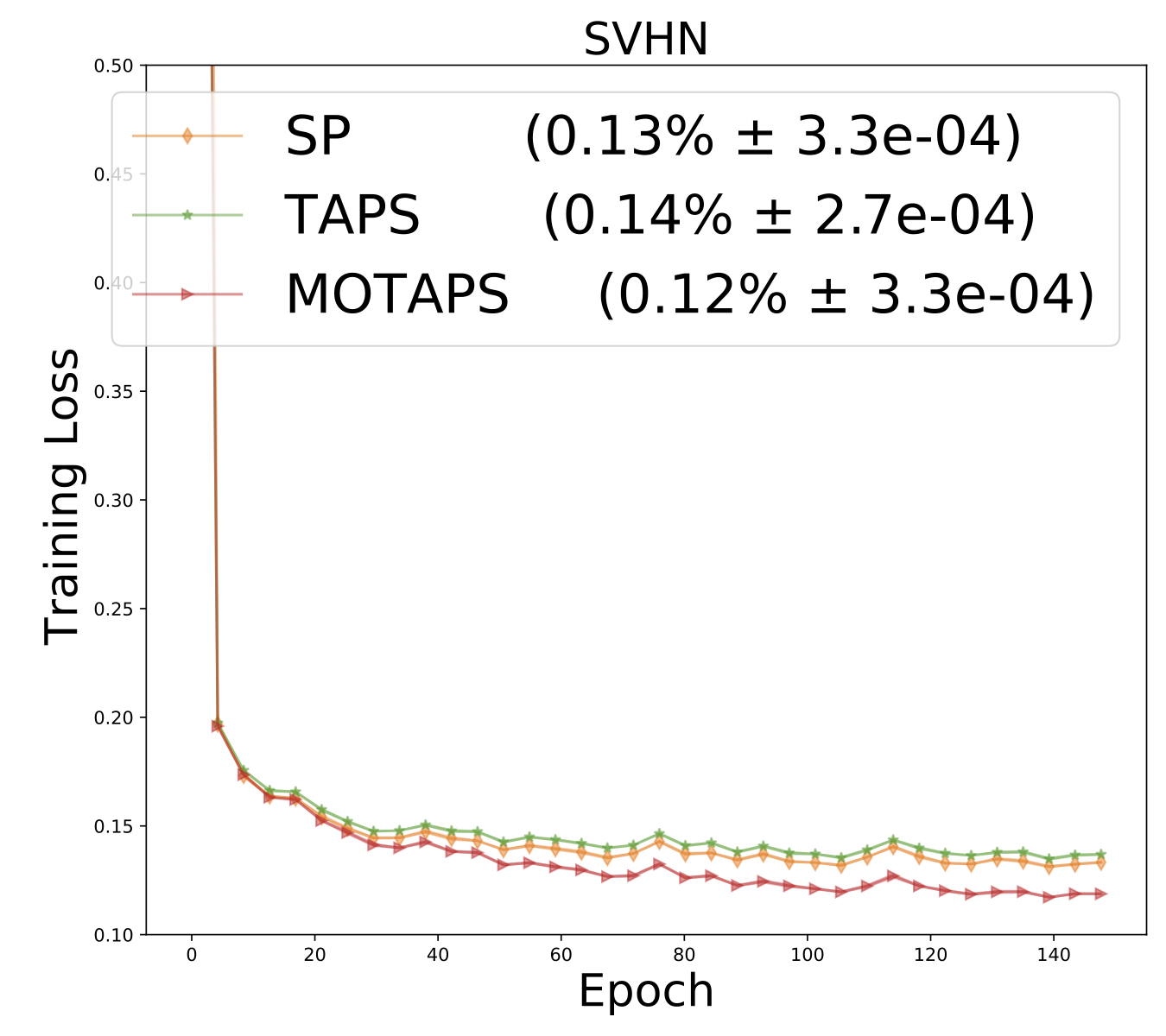}
\includegraphics[width=4.6cm]{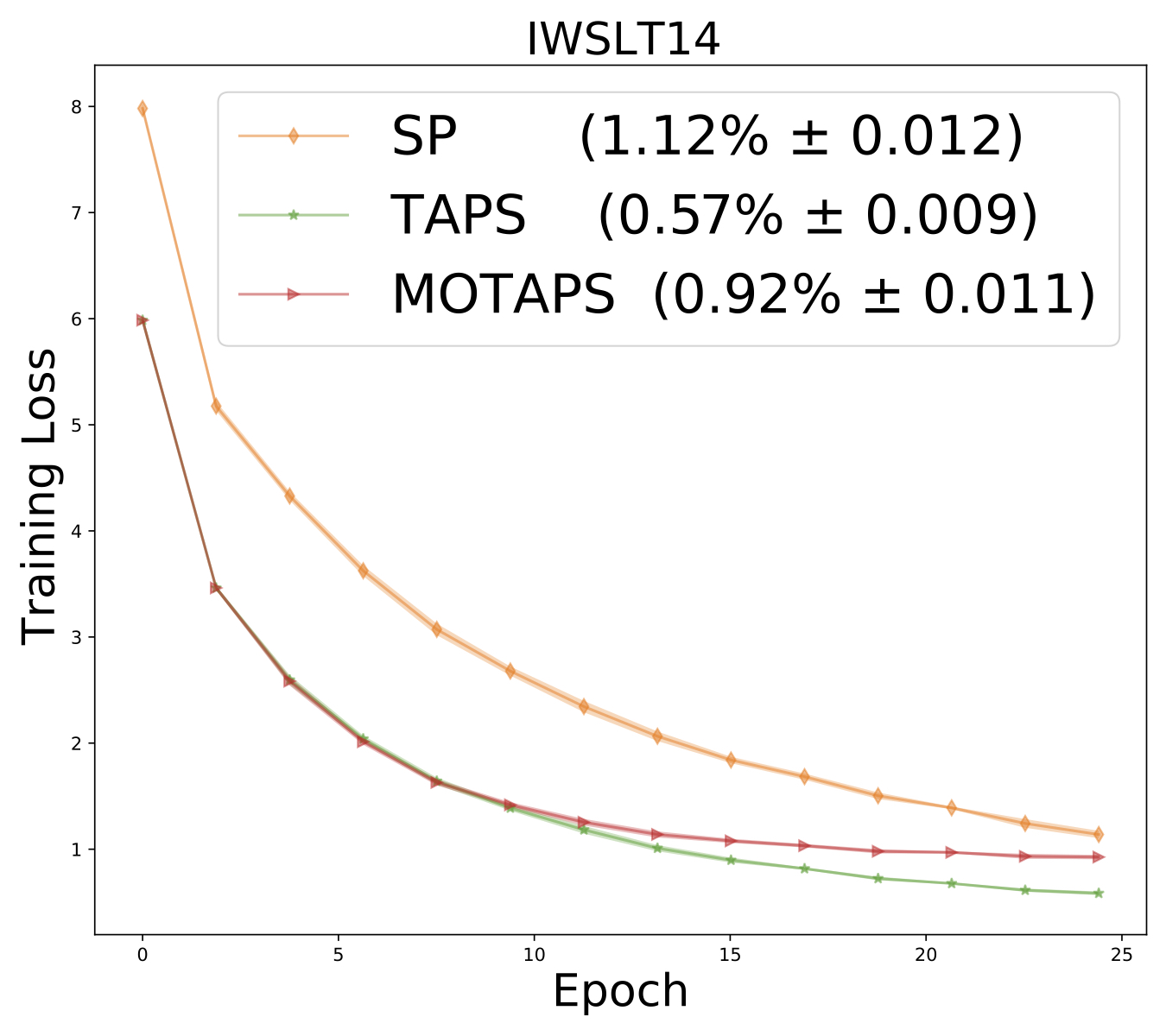}
\caption{Deep learning experiments training loss}
\label{fig:deep_learning_train}
\end{figure}

\end{document}